%% file: main_arxiv.tex
\documentclass[11pt]{article}
\include{header}
\usepackage{algpseudocode}
\usepackage{epstopdf}
\usepackage{authblk}
\usepackage[subrefformat=parens,labelformat=parens]{subfig}
\newtheorem*{rep@theorem}{\rep@title}
\newcommand{\newreptheorem}[2]{%
\newenvironment{rep#1}[1]{%
 \def\rep@title{#2 \ref{##1}}%
 \begin{rep@theorem}}%
 {\end{rep@theorem}}}
 \usepackage[margin=1in]{geometry}
\makeatother
\usepackage{graphicx}

\newcommand{\IN}{\text{in}}
\newcommand{\OUT}{\text{out}}

\newcommand{\vin}{v_{in}}
\newcommand{\vout}{v_{out}}
\newcommand{\relu}{\sigma_{\textsc{relu}}}

\newcommand{\RSGD}{Path-SGD }

\newcommand{\picwidth}{1.84in}

\newreptheorem{theorem}{Theorem}
\newreptheorem{lemma}{Lemma}

\newcommand{\removed}[1]{}

\usepackage{times}
\usepackage{hyperref}
\usepackage{breqn}
\usepackage{url}

\title{Path-SGD: Path-Normalized Optimization in \\Deep Neural Networks}
\author[1]{Behnam Neyshabur}
\author[2]{Ruslan Salakhutdinov}
\author[1]{Nathan Srebro}
\affil[1]{Toyota Technological Institute at Chicago}
\affil[2]{Department of Computer Science, University of Toronto}

\date{}
\begin{document}

\maketitle

\begin{abstract}
  We revisit the choice of SGD for training deep neural networks by
  reconsidering the appropriate geometry in which to optimize the
  weights.  We argue for a geometry invariant to rescaling of weights
  that does not affect the output of the network, and suggest
  Path-SGD, which is an approximate steepest descent method with
  respect to a path-wise regularizer related to max-norm
  regularization.  Path-SGD is easy and efficient to implement and
  leads to empirical gains over SGD and AdaGrad.
\end{abstract}

\section{Introduction}
\removed{
Deep neural networks that learn a hierarchy of features (representations) from data have
have been successfully applied in a variety of application domains, including 
visual object recognition~\cite{AlexNet,DBLP:journals/corr/SzegedyLJSRAEVR14}, 
speech recognition~\cite{HintonIBM},
machine translation~\cite{Ilya,DBLP:journals/corr/BahdanauCB14}, and many others. 
Contrary to the practical success of deep learning, there are still many open questions on understanding the optimization and generalization issues in deep learning models. }

Training deep networks is a challenging problem
\cite{icml2013,difficulty} and various heuristics and
optimization algorithms have been suggested in order to improve the
efficiency of the training~\cite{batch_norm,Kronecker,he2015delving}.
However, training deep architectures is still considerably slow and
the problem has remained open. Many of the current training methods
rely on good initialization and then performing Stochastic Gradient
Descent (SGD), sometimes together with an adaptive stepsize or
momentum term~\cite{icml2013,adagrad,Adam}.

Revisiting the choice of gradient descent, we recall that optimization
is inherently tied to a choice of geometry or measure of distance,
norm or divergence.  Gradient descent for example is tied to the
$\ell_2$ norm as it is the steepest descent with respect to $\ell_2$
norm in the parameter space, while coordinate descent corresponds to
steepest descent with respect to the $\ell_1$ norm and exp-gradient
(multiplicative weight) updates is tied to an entropic divergence.
Moreover, at least when the objective function is convex, convergence
behavior is tied to the corresponding norms or potentials. For example, with
gradient descent, or SGD, convergence speeds depend on the $\ell_2$
norm of the optimum.  The norm or divergence can be viewed as a
regularizer for the updates.  There is therefore also a strong link
between regularization for optimization and regularization for
learning: optimization may provide implicit regularization in terms of
its corresponding geometry, and for ideal optimization performance the
optimization geometry should be aligned with inductive bias driving
the learning \cite{srebro11}.

Is the $\ell_2$ geometry on the weights the appropriate geometry for
the space of deep networks?  Or can we suggest a geometry with more
desirable properties that would enable faster optimization and perhaps
also better implicit regularization?  As suggested above, this
question is also linked to the choice of an appropriate regularizer
for deep networks.  

Focusing on networks with RELU activations, we observe that scaling
down the incoming edges to a hidden unit and scaling up the outgoing
edges by the same factor yields an equivalent network computing the
same function.  Since predictions are invariant to such rescalings, it
is natural to seek a geometry, and corresponding optimization method,
that is similarly invariant.  

We consider here a geometry inspired by max-norm regularization
(regularizing the maximum norm of incoming weights into any unit)
which seems to provide a better inductive bias compared to the
$\ell_2$ norm (weight decay) \cite{goodfellow13,srivastava14}.  But to
achieve rescaling invariance, we use not the max-norm itself, but
rather the minimum max-norm over all rescalings of the weights.  We
discuss how this measure can be expressed as a ``path regularizer''
and can be computed efficiently.

We therefore suggest a novel optimization method, \RSGD, that is an
approximate steepest descent method with respect to path
regularization.  \RSGD is rescaling-invariant and we demonstrate that
\RSGD outperforms gradient descent and AdaGrad for classifications
tasks on several benchmark datasets.

\paragraph{Notations} A feedforward neural network that computes a
function $f:\R^D \rightarrow \R^C$ can be represented by a directed
acyclic graph (DAG) $G(V,E)$ with $D$ input nodes
$\vin[1],\dots,\vin[D]\in V$, $C$ output nodes $\vout[1],\dots,
\vout[C]\in V$, weights $w:E\rightarrow \R$ and an activation function
$\sigma:\R\rightarrow\R$ that is applied on the internal nodes (hidden
units). We denote the function computed by this network as
$f_{G,w,\sigma}$. In this paper we focus on RELU (REctified Linear
Unit) activation function $\relu(x) = \max\{0,x\}$. We refer to the
depth $d$ of the network which is the length of the longest directed
path in $G$.  For any $0\leq i \leq d$, we define $V^i_{\IN}$ to be
the set of vertices with longest path of length $i$ to an input unit
and $V^{i}_{\OUT}$ is defined similarly for paths to output units. In
layered networks $V^i_{\IN} = V^{d-i}_{\OUT}$ is the set of hidden
units in a hidden layer $i$.

\section{Rescaling and Unbalanceness}

\begin{figure}[t!]
\hspace{0.12in}
\subfloat[Training on MNIST]{
  \includegraphics[width=0.35\textwidth]{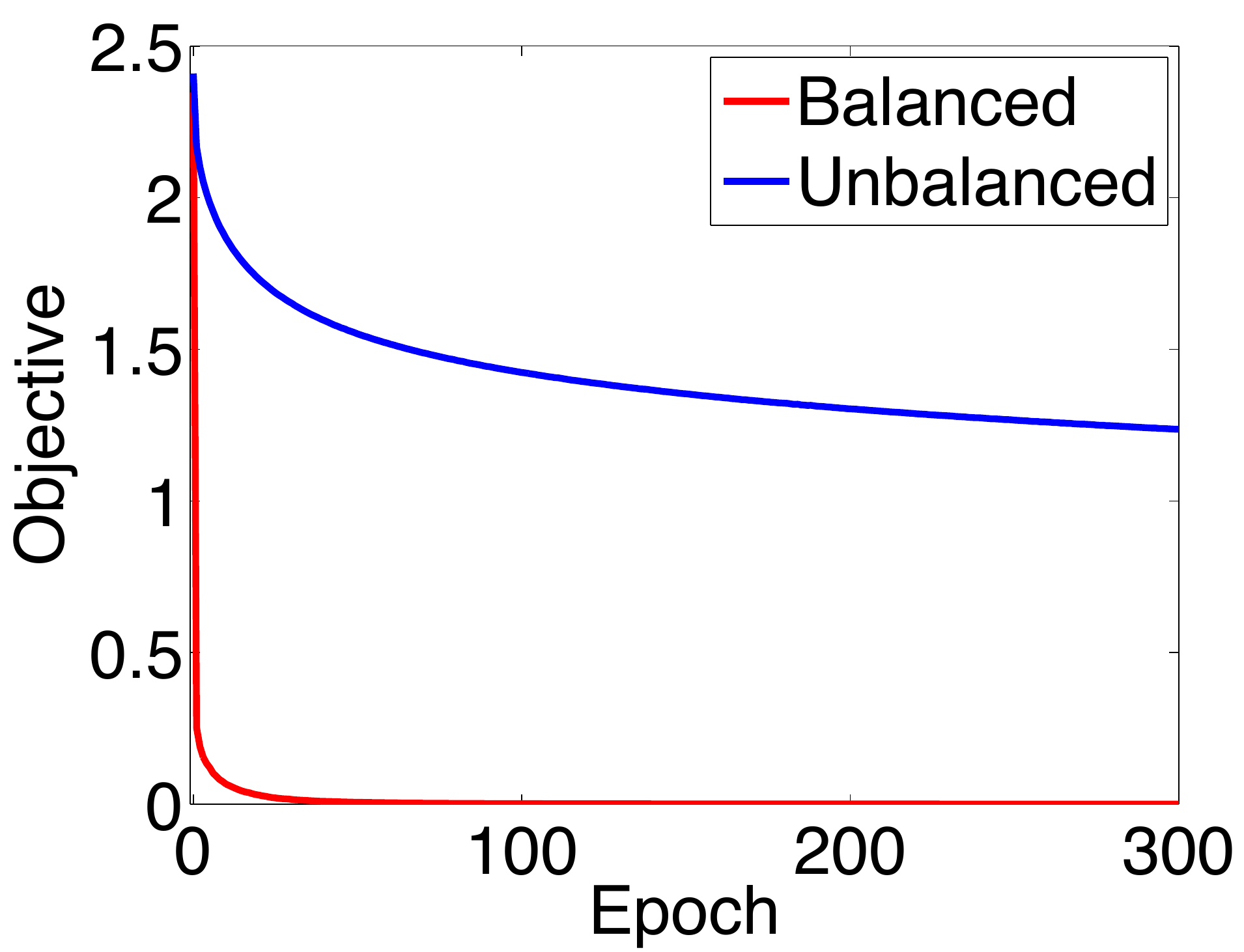}\label{fig:compare-a}
 }\hspace{0.5in}
 \subfloat[Weight Explosion in an unbalanced network]{
  \includegraphics[width=0.50\textwidth]{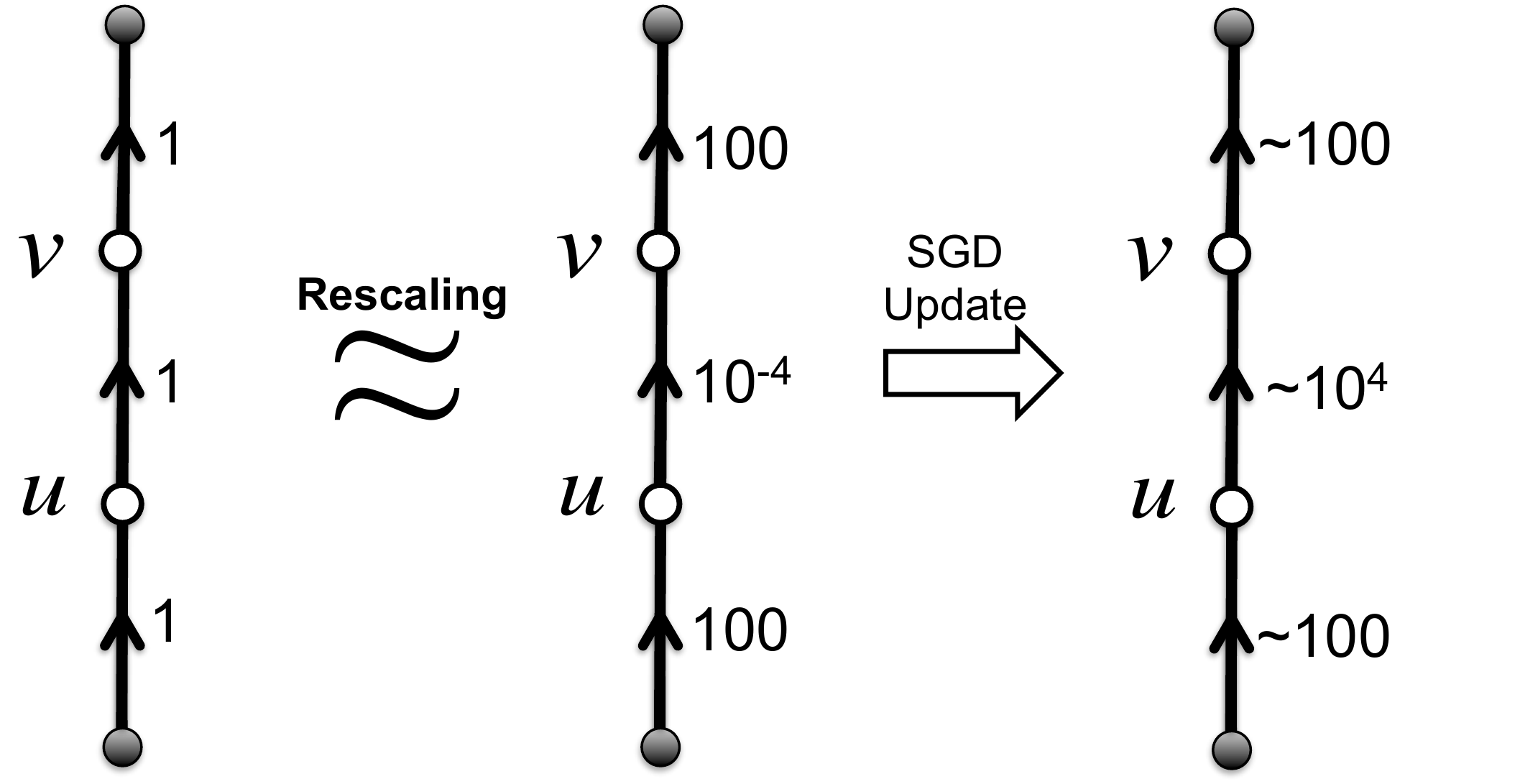}\label{fig:compare-b}
 }
 \newline
  \subfloat[Poor updates in an unbalanced network]{
  \includegraphics[width=1\textwidth]{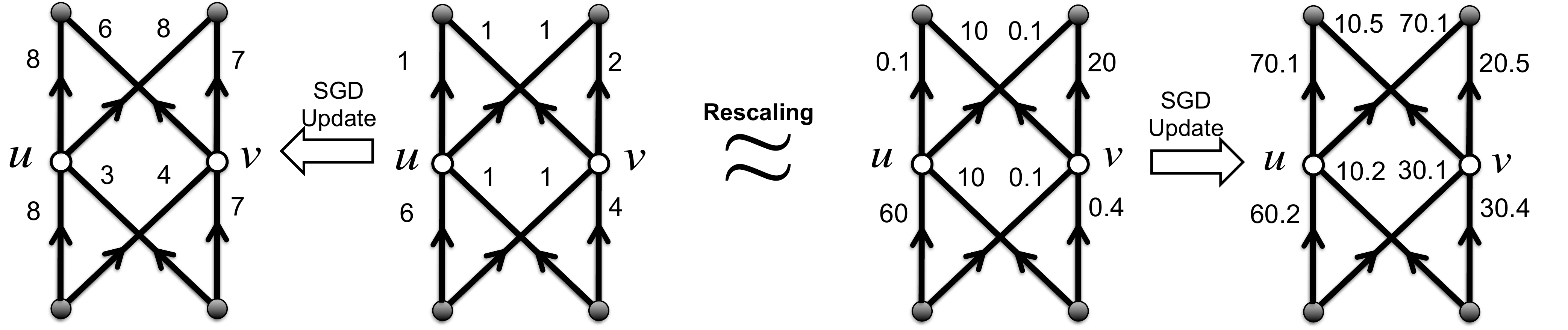}\label{fig:compare-c}
   }
 \caption{ \small (a): Evolution of the cross-entropy error function when training a 
feed-forward network on MNIST with two hidden layers, each containing 4000 hidden units.
The unbalanced initialization (blue curve) is generated by applying a sequence of rescaling functions on the balanced initializations (red curve). (b): Updates for a simple case where the input is $x=1$, 
thresholds are set to zero (constant), the stepsize is 1, and the gradient with respect to output is $\delta = -1$. (c): Updated network for the case where the input is $x=(1,1)$, thresholds are set to zero (constant), the stepsize is 1, and the gradient with respect to output is $\delta=(-1,-1)$. }
 \label{fig:unbalanced}
\end{figure}

One of the special properties of RELU activation function is non-negative homogeneity. That is, for any scalar $c\geq 0$ and any $x\in \R$, we have $\relu(c\cdot x)=c\cdot \relu(x)$. This  interesting property allows the network to be rescaled without changing the function computed by the network. We define the {\em rescaling function} $\rho_{c,v}(w)$, such that given the weights of the network $w:E\rightarrow \R$, a constant $c>0$, and a node $v$, the rescaling function multiplies the incoming edges and divides the outgoing edges of $v$ by $c$. That is, $\rho_{c,v}(w)$ maps $w$ to the weights $\tilde{w}$ for the rescaled network, where for any $(u_1\rightarrow u_2)\in E$:
\begin{equation}
\tilde{w}_{(u_1\rightarrow u_2)}=
\begin{cases}
c.w_{(u_1\rightarrow u_2)}& u_2=v,\\
\frac{1}{c}w_{(u_1\rightarrow u_2)}& u_1=v,\\
w_{(u_1\rightarrow u_2)}& \text{otherwise.}\\
\end{cases}
\end{equation}
It is easy to see that the rescaled network computes the same function, i.e. $f_{G,w,\relu} =f_{G,\rho_{c,v}(w),\relu}$. We say that the two networks with weights $w$ and $\tilde{w}$ are {\em rescaling equivalent} denoted by $w\sim \tilde{w}$ if and only if one of them can be transformed to another by applying a sequence of rescaling functions $\rho_{c,v}$.

Given a training set $\calS = \{(x_1,y_n),\dots, (x_n,y_n)\}$, our goal is to 
minimize the following objective function:
\begin{equation}
L(w) = \frac{1}{n}\sum_{i=1}^n \ell(f_w(x_i),y_i).
\end{equation}
Let $w^{(t)}$ be the weights at step $t$ of the optimization. We consider update step of the following form $w^{(t+1)} = w^{(t)} + \Delta w^{(t+1)}$. For example, for gradient descent, we have  $\Delta w^{(t+1)} = -\eta\nabla L(w^{(t)})$, where $\eta$ is the step-size. In the stochastic setting, such as SGD or mini-batch gradient descent, we calculate the gradient on a small subset of the training set. 

Since {\em rescaling equivalent} networks compute the same function,
it is desirable to have an update rule that is not affected by
rescaling. We call an optimization method {\em rescaling invariant} if
the updates of rescaling equivalent networks are rescaling equivalent.
That is, if we start at either one of the two rescaling equivalent
weight vectors $\tilde{w}^{(0)}\sim w^{(0)}$, after applying $t$
update steps separately on $\tilde{w}^{(0)}$ and $w^{(0)}$, they will
remain rescaling equivalent and we have $\tilde{w}^{(t)} \sim
w^{(t)}$.

Unfortunately, gradient descent is {\em not} rescaling invariant. The main problem with the gradient updates is that scaling down the weights of an edge will also scale up the gradient which, as we see later, is exactly the opposite of what is expected from a rescaling invariant update. 

Furthermore, gradient descent performs very poorly on ``unbalanced''
networks.  We say that a network is {\em balanced} if the norm of
incoming weights to different units are roughly the same or within a
small range. For example, Figure~\subref*{fig:compare-a} shows a huge
gap in the performance of SGD initialized with a randomly generated
balanced network $w^{(0)}$, when training on MNIST, compared to a
network initialized with unbalanced weights $\tilde{w}^{(0)}$. Here
$\tilde{w}^{(0)}$ is generated by applying a sequence of random
rescaling functions on $w^{(0)}$ (and therefore $w^{(0)}\sim
\tilde{w}^{(0)}$).

In an unbalanced network, gradient descent updates could blow up the smaller weights, while keeping the larger weights almost unchanged. This is illustrated in Figure ~\subref*{fig:compare-b}. If this were the only issue, one could scale down all the weights after each update. However, in an unbalanced network,  the relative changes in the weights are also very different compared to a balanced network. For example, Figure \subref*{fig:compare-c} shows how two rescaling equivalent networks could end up computing a very different function after only a single update.

\removed{The same problem exists in other optimization methods used in deep learning, including AdaGrad updates~\cite{adagrad}, where $\Delta w^{(t+1)}_e = -\eta \left(\partial L / \partial w_e^{(t)}\right)/\sqrt{\sum_{k=1}^t (\partial L / \partial w_e^{(k)})^2}$. In order to better understand and investigate this issue, we next discuss different scale measures for training neural networks.}
\section{Magnitude/Scale measures for deep networks}
Following \cite{neyshabur15}, we consider the grouping of weights going into each node of the
network. This forms the following generic group-norm type regularizer, parametrized by $1\leq p,q \leq\infty$:
\begin{equation}
  \label{eq:mu}
  \mu_{p,q}(w) = \left(\sum_{v \in V}\left(\sum_{(u\rightarrow v) \in E} \left\lvert w_{(u\rightarrow v)}\right\rvert ^p\right)^{q/p}\right)^{1/q}.
\end{equation}
Two simple cases of above group-norm are $p=q=1$ and $p=q=2$ that
correspond to overall $\ell_1$ regularization and weight decay
respectively. Another form of regularization that is shown to be very
effective in RELU networks is the max-norm regularization, which is the
maximum over all units of norm of incoming edge to the
unit\footnote{This definition of max-norm is a bit different than the
  one used in the context of matrix factorization~\cite{srebro05}. The
  later is similar to the minimum upper bound over $\ell_2$ norm of
  both outgoing edges from the input units and incoming edges to the
  output units in a two layer feed-forward
  network.}~\cite{goodfellow13,srivastava14}. The max-norm correspond
to ``per-unit" regularization when we set $q=\infty$ in
equation~\eqref{eq:mu} and can be written in the following form:
\begin{equation}
  \label{eq:mu}
  \mu_{p,\infty}(w) =\sup_{v \in V}\left(\sum_{(u\rightarrow v) \in E} \left\lvert w_{(u\rightarrow v)}\right\rvert ^p\right)^{1/p}
\end{equation}

Weight decay is probably the most commonly used regularizer. On the
other hand, per-unit regularization might not seem ideal as it is
very extreme in the sense that the value of regularizer corresponds to
the highest value among all nodes.  However, the situation is very
different for networks with RELU activations (and other activation
functions with non-negative homogeneity property).  In these cases,
per-unit $\ell_2$ regularization has shown to be very
effective~\cite{srivastava14}. The main reason could be because RELU
networks can be rebalanced in such a way that all hidden units have
the same norm. Hence, per-unit regularization will not be a crude
measure anymore.

Since $\mu_{p,\infty}$ is not rescaling invariant and the values of the
scale measure are different for rescaling equivalent networks, it is
desirable to look for the minimum value of a regularizer among all
rescaling equivalent networks. Surprisingly, for a feed-forward
network, the minimum $\ell_p$ per-unit regularizer among all rescaling
equivalent networks can be efficiently computed by a single forward
step. To see this, we consider the vector $\pi(w)$, the {\em path
  vector}, where the number of coordinates of $\pi(w)$ is equal to the
total number of paths from the input to output units and each
coordinate of $\pi(w)$ is the equal to the product of weights along a
path from an input nodes to an output node. The $\ell_p$-path
regularizer is then defined as the $\ell_p$ norm of
$\pi(w)$~\cite{neyshabur15}:
\begin{equation}\label{eq:defphi}
  \phi_p(w) = \norm{\pi(w)}_p = \left(\sum_{\vin[i] \overset{e_1}\rightarrow v_1\overset{e_2}\rightarrow v_2\dots\overset{e_d}{\rightarrow}\vout[j]} \left|\prod_{k=1}^d w_{e_k}\right|^p\right)^{1/p}
\end{equation}
The following Lemma establishes that the $\ell_p$-path regularizer
corresponds to the minimum over all equivalent networks of the
per-unit $\ell_p$ norm:
\begin{lem}[\cite{neyshabur15}]\label{lem:path-unit}
$\displaystyle \phi_p(w) = \min_{\tilde{w} \sim w} \bigg(\mu_{p,\infty}(\tilde{w})\bigg)^d$
\end{lem}
The definition \eqref{eq:defphi} of the $\ell_p$-path regularizer
involves an exponential number of terms.  But it can be computed
efficiently by dynamic programming in a single forward step using the
following equivalent form as nested sums:
\begin{equation*}
\phi_p(w) = \left(\sum_{(v_{d-1}\rightarrow v_{\OUT}[j])\in E}\left| w_{(v_{d-1}\rightarrow v_{\OUT}[j])}\right|^p\sum_{(v_{d-2}\rightarrow v_{d-1})\in E}\dots \sum_{(v_{\IN}[i]\rightarrow v_{1})\in E} \left| w_{(v_{\IN}[i]\rightarrow v_{1})}\right|^p \right)^{1/p}
\end{equation*}
A straightforward consequence of Lemma \ref{lem:path-unit} is that the $\ell_p$ path-regularizer $\phi_p$ is invariant to rescaling, i.e. for any $\tilde{w} \sim w$, $\phi_p(\tilde{w})=\phi_p(w)$.

\section{\RSGD: An Approximate Path-Regularized Steepest Descent}
Motivated by empirical performance of max-norm regularization and the
fact that path-regularizer is invariant to rescaling, we are
interested in deriving the steepest descent direction with respect to
the path regularizer $\phi_p(w)$:
\begin{align}\label{eq:sd}
w^{(t+1)} &= \arg\min_w \;\;\eta \inner{ \nabla L(w^{(t)})}{w} + \frac{1}{2}\norm{\pi(w)-\pi(w^{(t)})}_p^2\\ \notag
&= \arg\min_w \;\;\eta \inner{ \nabla L(w^{(t)}) }{w} + \left(\sum_{\vin[i] \overset{e_1}\rightarrow v_1\overset{e_2}\rightarrow v_2\dots\overset{e_d}{\rightarrow}\vout[j]}\left(\prod_{k=1}^d w_{e_k} - \prod_{k=1}^d w^{(t)}_{e_k})\right)^p\right)^{2/p}\\ \notag
& = \arg\min_w J^{(t)}(w)
\end{align}
The steepest descent step \eqref{eq:sd} is hard to calculate exactly.  Instead, we will update each coordinate $w_e$ independently (and synchronously) based on~\eqref{eq:sd}. That is:
\begin{equation}
w^{(t+1)}_e =\arg\min_{w_e} \;J^{(t)}(w) \qquad \text{s.t.}\;\;\forall_{e'\neq e} \;\;w_{e'}=w^{(t)}_{e'}
\end{equation}
Taking the partial derivative with respect to $w_e$ and setting it to zero we obtain:
\begin{equation*}
0 =\eta \frac{\partial L}{\partial w_e}(w^{(t)}) - \left(w_e-w^{(t)}_e\right) \left(\sum_{v_{\text{in}}[i] \dots \stackrel{e}{\rightarrow} \dots v_{\text{out}}[j]} \prod_{e_k\neq e} \abs{w^{(t)}_e}^p\right)^{2/p}
\end{equation*}
where $v_{\text{in}}[i] \dots \stackrel{e}{\rightarrow} \dots v_{\text{out}}[j]$ denotes the paths from any input unit $i$ to any output unit $j$ that includes $e$. Solving for $w_e$  gives us the following update rule:
\begin{equation}\label{eq:update}
\hat{w}^{(t+1)}_e = w^{(t)}_e- \frac{\eta}{\gamma_p(w^{(t)},e)}\frac{\partial L}{\partial w}(w^{(t)})
\end{equation}
where $\gamma_p(w,e)$ is given as
\begin{equation}
\gamma_p(w,e) =\left(\sum_{v_{\text{in}}[i] \dots \stackrel{e}{\rightarrow} \dots v_{\text{out}}[j]} \prod_{e_k\neq e} \abs{w_{e_k}}^p\right)^{2/p}
\end{equation}
We call the optimization using the update rule \eqref{eq:update} path-normalized gradient descent. When used in stochastic settings, we refer to it as \RSGD.

\removed{ In the following lemma, we prove that if the relative change in the weights of the network is small enough (which can be achieved by choosing a small enough stepsize), then the update rule~\eqref{eq:update} is an approximate steepest direction with respect to $\ell_p$-path regularizer.
\begin{lem}
Let $\delta_{\max}$ be the maximum relative change in a weight in the network by update rule~\eqref{eq:update}, i.e.
$\delta_{\max} = \max_{e\in E} \abs{ w^{(t+1)}_e - w^{(t)}_e }/\abs{w^{(t)}_e}$. Then, if $\delta_{\max} \leq \frac{1}{d}$, we have that:
$$
\norm{\pi(w^{(t+1)}) - \pi(w^{(t)})}_p \leq 2d\delta_{\max}\phi_p(w^{(t)})
$$
and so
$$
\phi_p(w^{(t+1)}) \leq 2d\delta_{\max}\phi_p(w^{(t)}).
$$
\end{lem}
\begin{proof}
Consider any path $\vin[i] \overset{e_1}\rightarrow v_1\overset{e_2}\rightarrow v_2\dots\overset{e_d}{\rightarrow}\vout[j]$ in the network. It is clear that:
\begin{equation*}
\abs{ \prod_{k=1}^d w^{(t+1)}_{e_k} - \prod_{k=1}^d w_{e_k}} \leq \left[(1+\delta_{\max})^d -1\right]\abs{\prod_{k=1}^d w_{e_k}} 
\end{equation*}
Applying the above inequality on all paths and taking the sum gives us:
\begin{align*}
\norm{\pi(w^{(t+1)}) - \pi(w^{(t)})}_p  &\leq \left[(1+\delta_{\max})^d-1\right]\phi_p(w^{(t)})\\
&\leq (e-1)d\delta_{\max}
\end{align*}
\end{proof}
}

Now that we know \RSGD is an approximate steepest descent with respect to the path-regularizer, we can ask whether or not this makes \RSGD a {\em rescaling invariant} optimization method. The next theorem proves that \RSGD is indeed rescaling invariant.

\begin{thm}
\RSGD is rescaling invariant.
\end{thm}
\begin{proof}
It is sufficient to prove that using the update rule~\eqref{eq:update}, for any $c>0$ and any $v\in E$, if $\tilde{w}^{(t)} = \rho_{c,v}(w^{(t)})$, then $\tilde{w}^{(t+1)} = \rho_{c,v}(w^{(t+1)})$. For any edge $e$ in the network, if $e$ is neither incoming nor outgoing edge of the node $v$, then $\tilde{w}(e)=w(e)$, and since the gradient is also the same for edge $e$ we have $\tilde{w}^{(t+1)}_e=w^{(t+1)}_e$. However, if $e$ is an incoming edge to $v$, we have that $\tilde{w}^{(t)}(e)=cw^{(t)}(e)$. Moreover, since the outgoing edges of $v$ are divided by $c$, we get $\gamma_p(\tilde{w}^{(t)},e) = \frac{\gamma_p(w^{(t)},e)}{c^2}$ and $\frac{\partial L}{\partial w_e}(\tilde{w}^{(t)})= \frac{\partial L}{c\partial w_e}(w^{(t)})$. Therefore,
\begin{align*}
{\tilde{w}}^{(t+1)}_e &= cw^{(t)}_e - \frac{c^2\eta}{\gamma_p(w^{(t)},e)} \frac{\partial L}{c\partial w_e}(w^{(t)})\\
&= c\left(w^{(t)} - \frac{\eta}{\gamma_p(w^{(t)},e)} \frac{\partial L}{\partial w_e}(w^{(t)})\right) = cw^{(t+1)}_e.
\end{align*}
A similar argument proves the invariance of \RSGD update rule for outgoing edges of $v$. Therefore, \RSGD is rescaling invariant.
\end{proof}

\paragraph{Efficient Implementation:}
The Path-SGD update rule~\eqref{eq:update}, in the way it is written,
needs to consider all the paths, which is exponential in the depth of
the network. However, it can be calculated in a time that is no more
than a forward-backward step on a single data point. That is, in a
mini-batch setting with batch size $B$, if the backpropagation on the
mini-batch can be done in time $BT$, the running time of the Path-SGD
on the mini-batch will be roughly $(B+1)T$ -- a very moderate runtime
increase with typical mini-batch sizes of hundreds or thousands of
points.  Algorithm ~\ref{alg:update} shows an efficient implementation
of the Path-SGD update rule.

We next compare \RSGD to other optimization methods in both balanced and unbalanced settings.
 
\begin{algorithm}[t]
  \caption{\RSGD update rule}\label{alg:update}
  \begin{algorithmic}[1]
  \State $\forall_{v\in V^0_{\IN}}\; \gamma_{\IN}(v)=1$\Comment{Initialization}
  \State $\forall_{v\in V^0_\OUT}\; \gamma_{\OUT}(v)=1$
  \For{$i=1\;\textbf{to}\;d\;$}
  \State $\forall_{v\in V^i_{\IN}} \gamma_{\IN}(v) = \sum_{(u\rightarrow v)\in E} \gamma_{\IN}(u)\abs{w_{(u,v)}}^p$
  \State $\forall_{v\in V^i_{\OUT}} \gamma_{\OUT}(v) = \sum_{(v\rightarrow u)\in E} \abs{w_{(v,u)}}^p\gamma_{\OUT}(u)$
  \EndFor
  \State $\forall_{(u\rightarrow v)\in E}\;\; \gamma(w^{(t)},(u,v)) = \gamma_{\IN}(u)^{2/p}\gamma_{\OUT}(v)^{2/p}$
  \State $\forall_{e\in E} w^{(t+1)}_e = w^{(t)}_e - \frac{\eta}{\gamma(w^{(t)},e)} \frac{\partial L}{\partial w_e}(w^{(t)})$\Comment{Update Rule}
   \end{algorithmic}
\end{algorithm}

\section{Experiments}
In this section, we compare $\ell_2$-\RSGD to two commonly used optimization methods in deep learning, SGD and AdaGrad. We conduct our experiments on 
four common benchmark datasets: the standard MNIST dataset of handwritten digits~\cite{lecun1998gradient}; 
CIFAR-10 and CIFAR-100 datasets of tiny images of natural scenes~~\cite{krizhevsky2009learning}; 
and Street View House Numbers (SVHN) dataset containing 
color images of house numbers collected by Google Street View~\cite{netzer2011reading}. 
Details of the datasets are shown in Table~\ref{table}.

\begin{table}[t]
\caption{General information on datasets used in the experiments.}
\label{table}
\begin{center}
\begin{tabular}{c c c c c}
{\bf Data Set}  &{\bf Dimensionality}&{\bf Classes}&{\bf Training Set}&{\bf Test Set}
\\ \hline
CIFAR-10&3072 ($32 \times 32$ color)&10&50000&10000\\
CIFAR-100&3072 ($32 \times 32$ color)&100&50000&10000\\
MNIST&784 ($28 \times 28$ grayscale)&10&60000&10000\\
SVHN&3072 ($32 \times 32$ color)&10&73257&26032\\
\hline
\end{tabular}
\end{center}
\end{table}

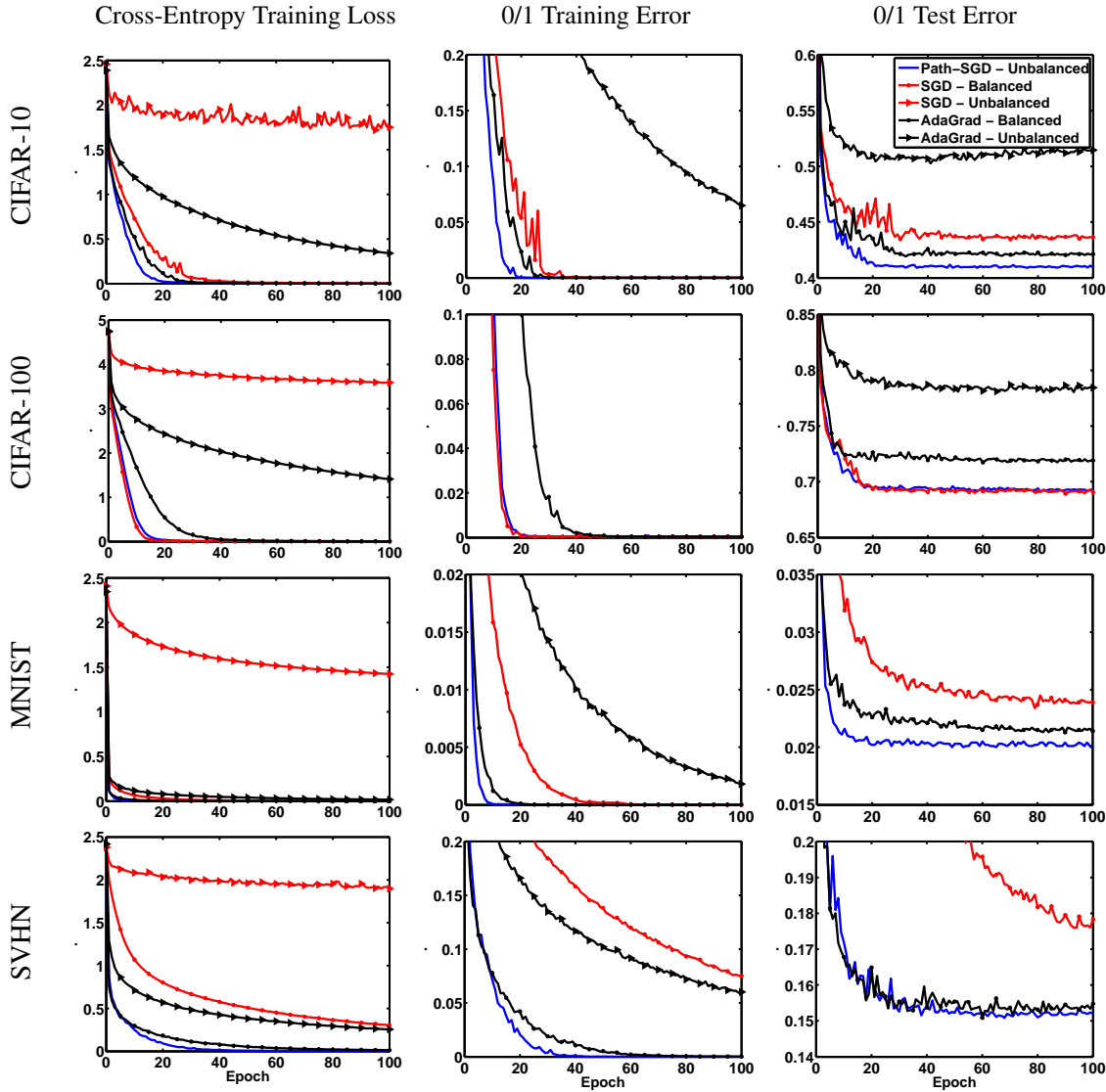
\begin{figure}[t!]
\hspace{0.1in}
 \subfloat{
  \begin{tabular}{r}
   \includegraphics[width=\picwidth]{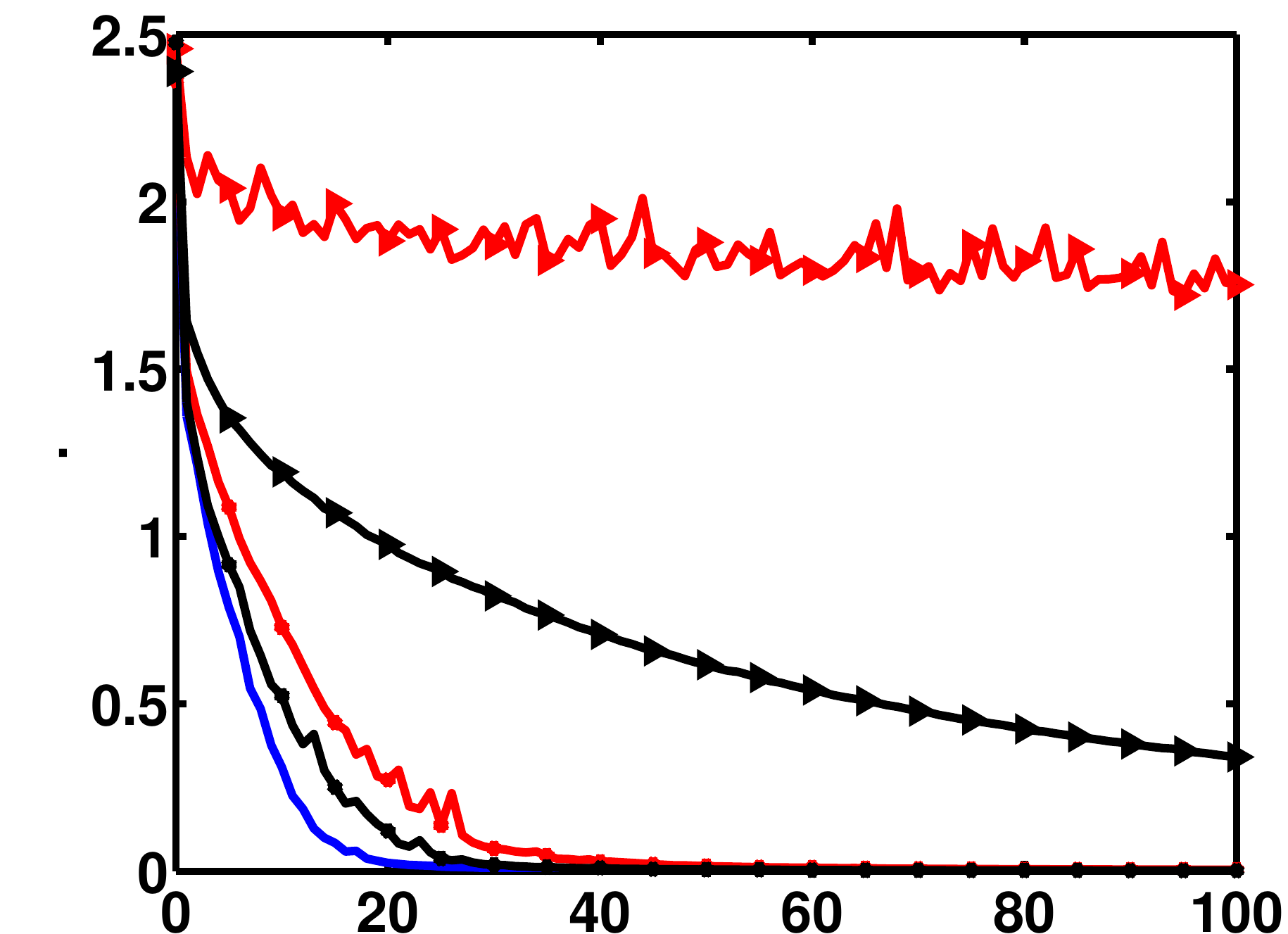} \\
   \includegraphics[width=1.76in]{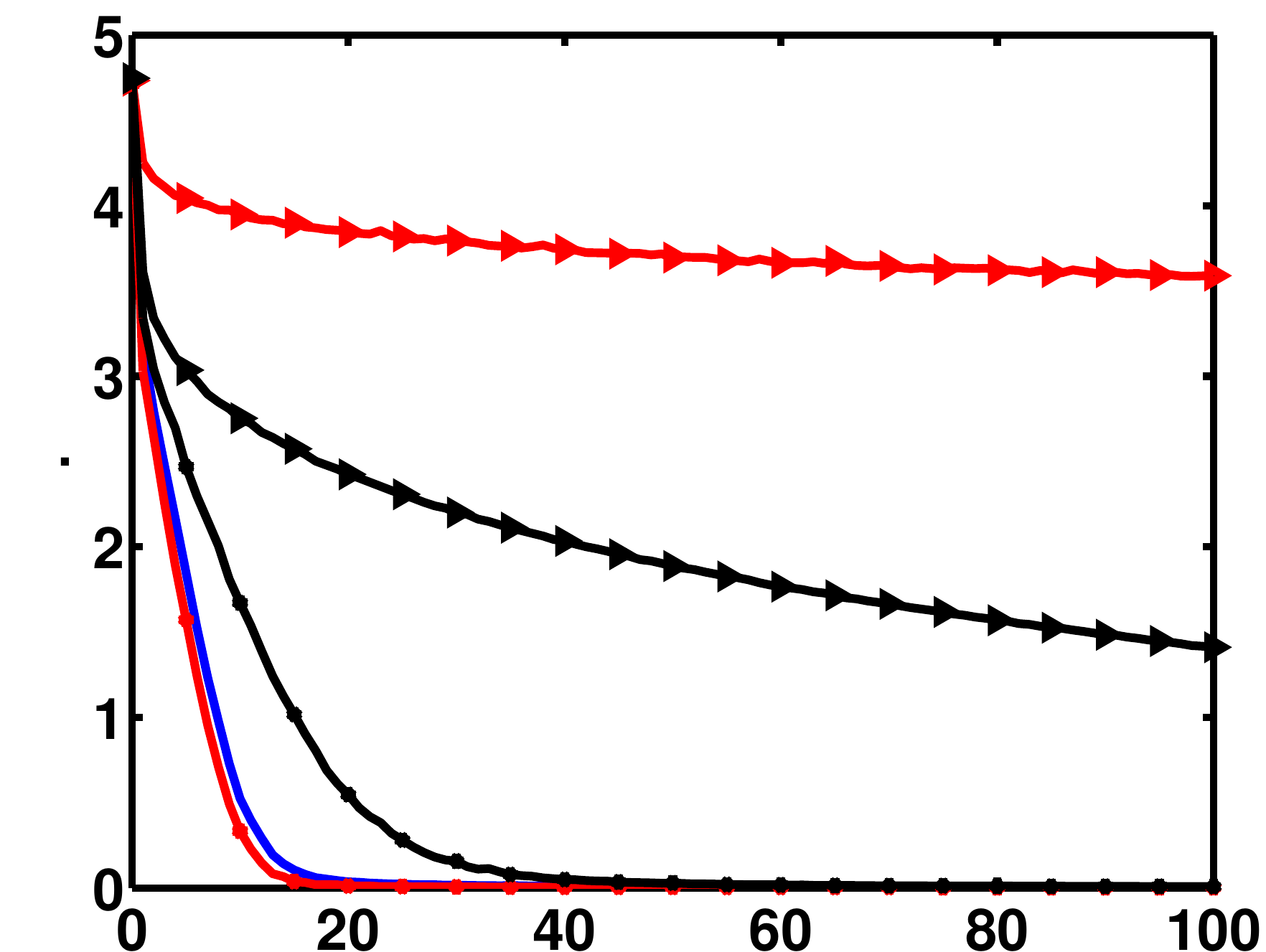} \\
   \includegraphics[width=\picwidth]{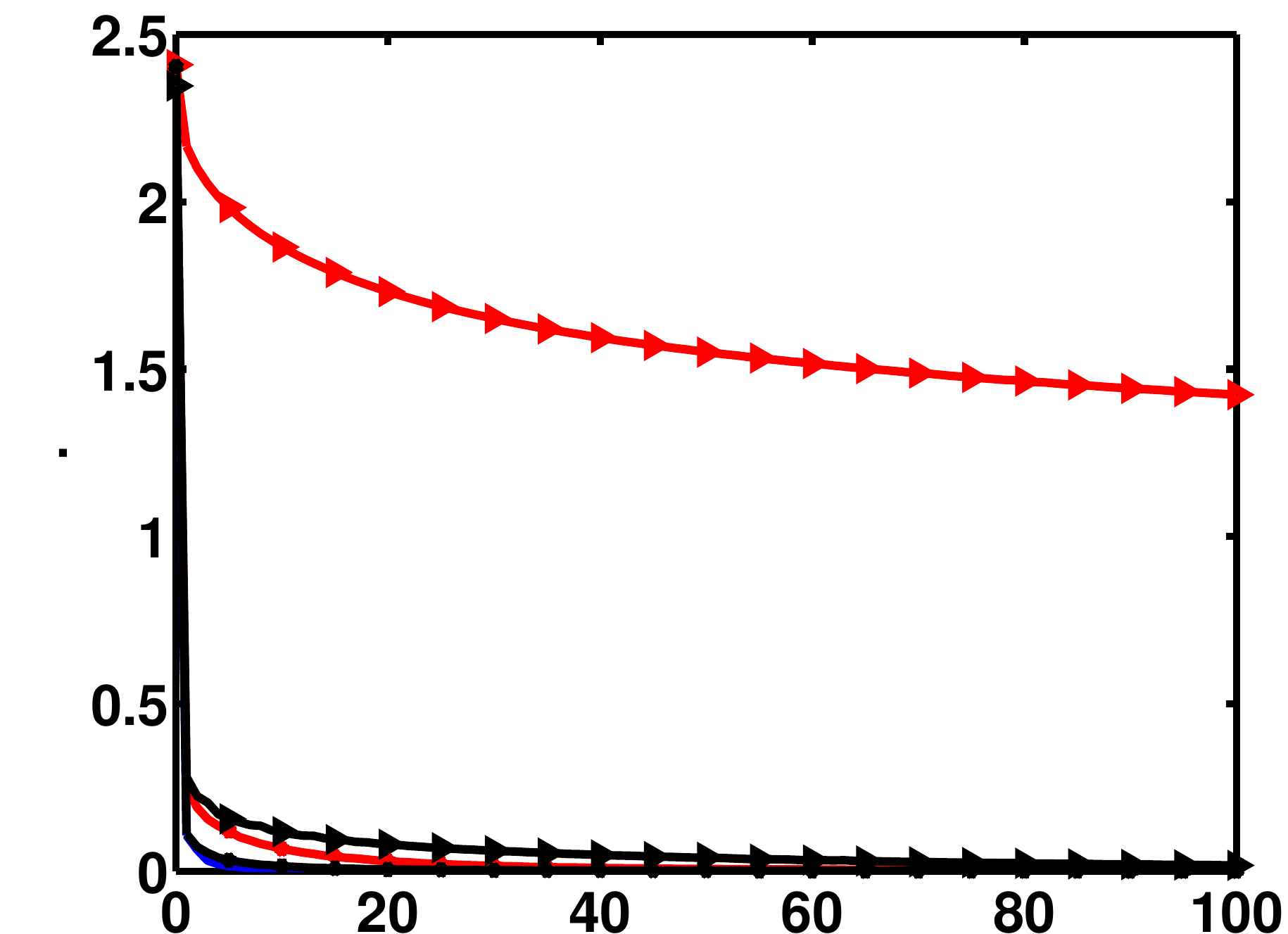} \\
   \includegraphics[width=\picwidth]{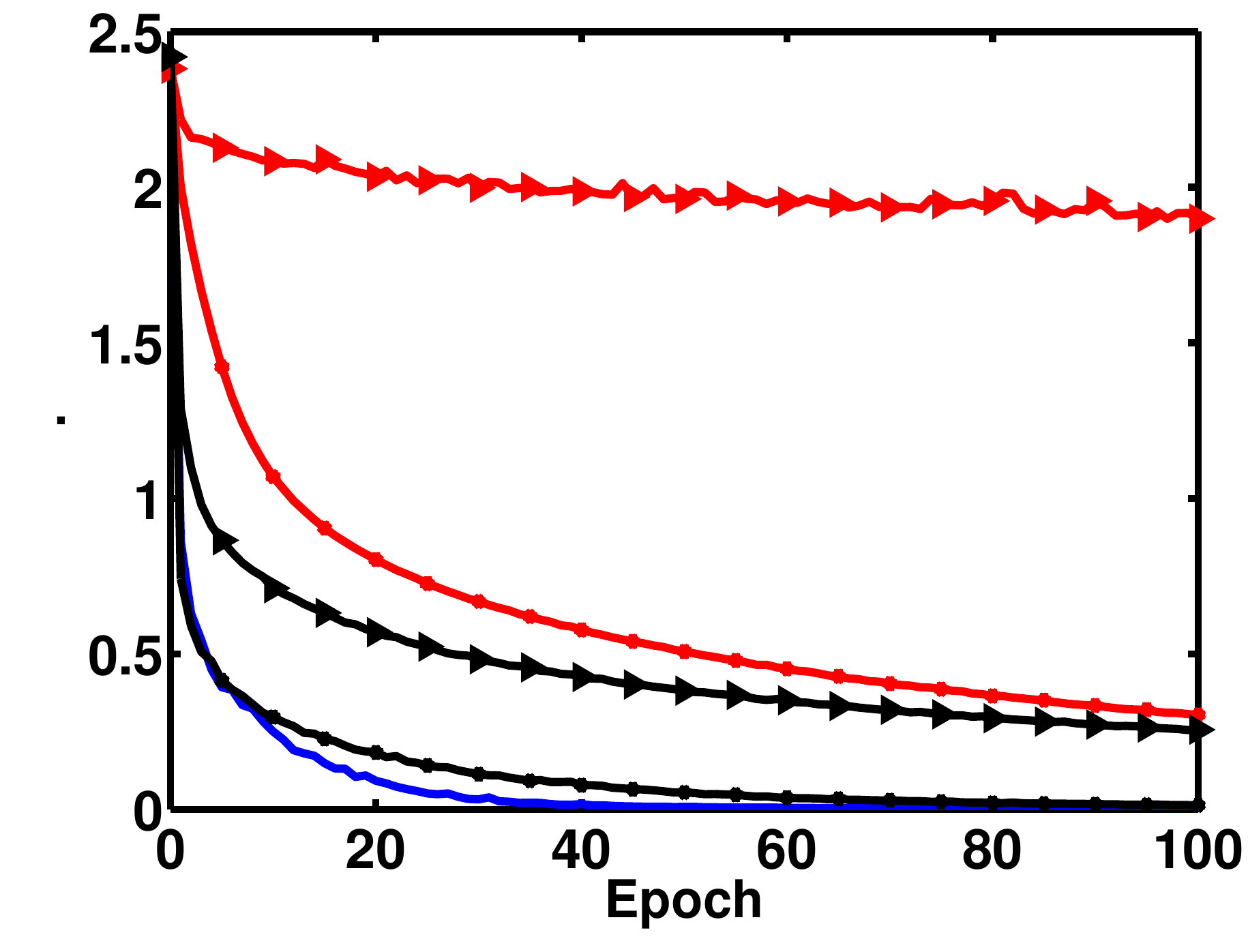}
  \end{tabular}
 }
 \hspace{-0.3in}
 \subfloat{
  \begin{tabular}{r}
   \includegraphics[width=\picwidth]{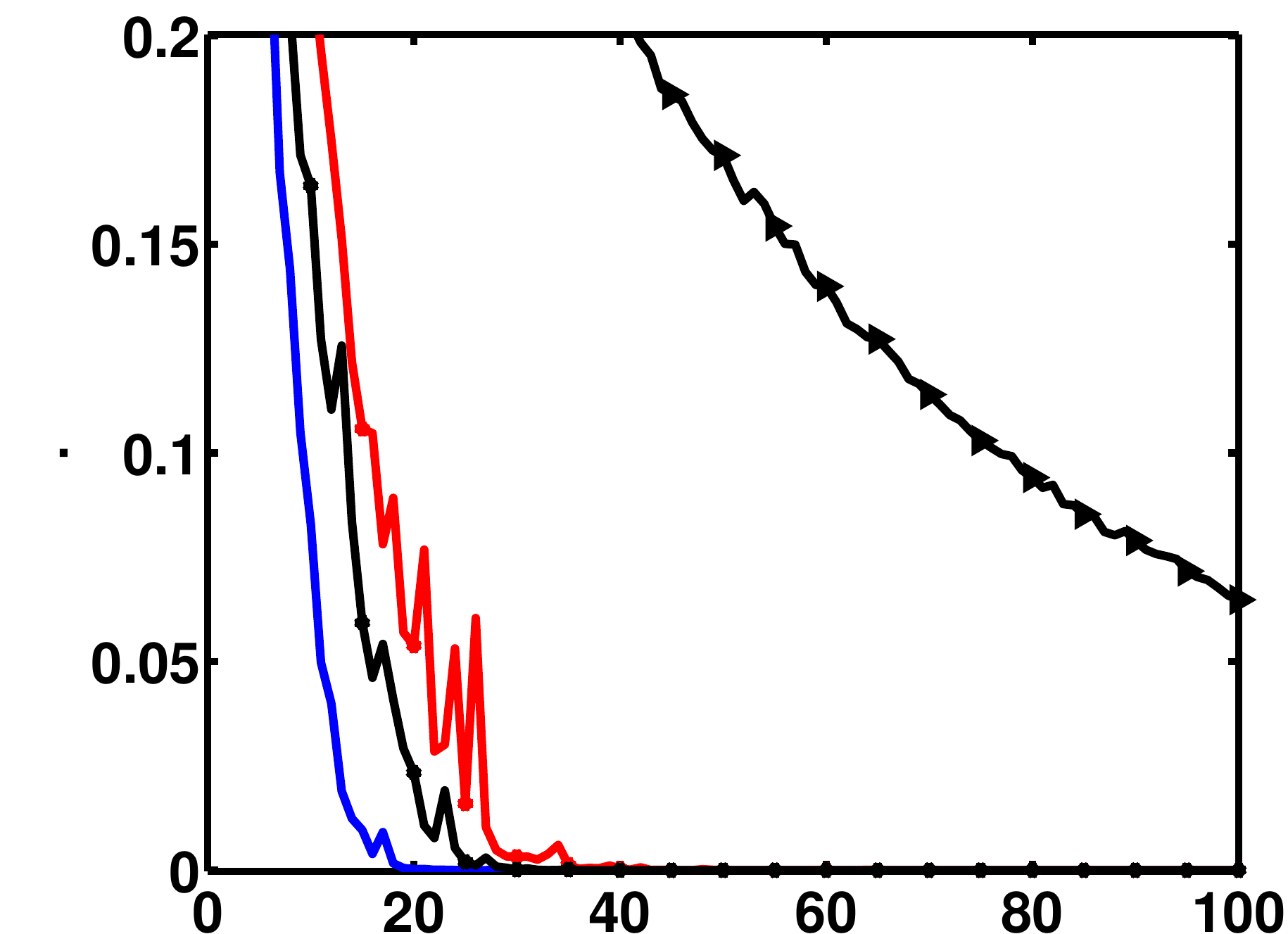} \\
   \includegraphics[width=\picwidth]{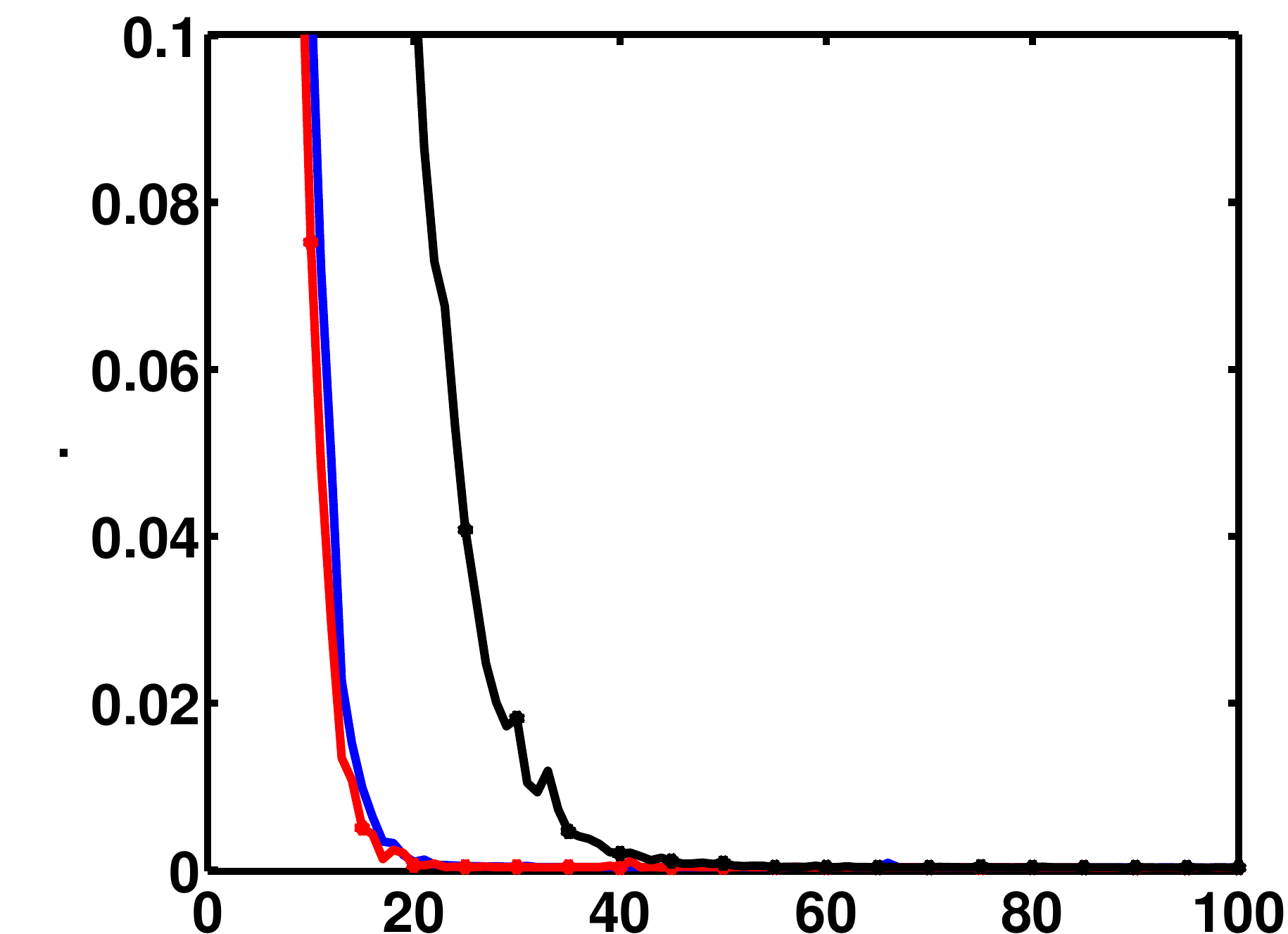} \\
   \includegraphics[width=1.9in]{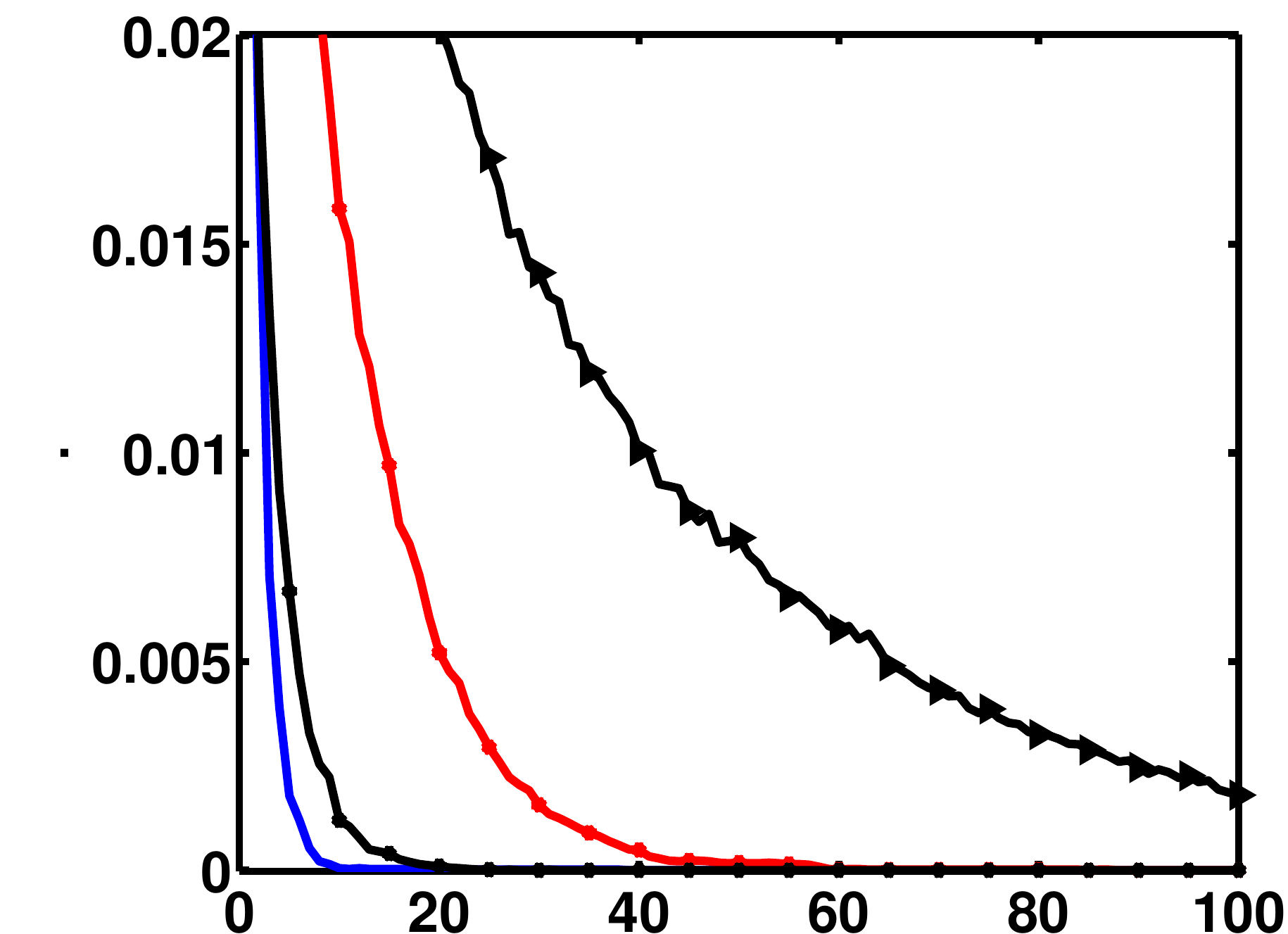} \\
   \includegraphics[width=1.85in]{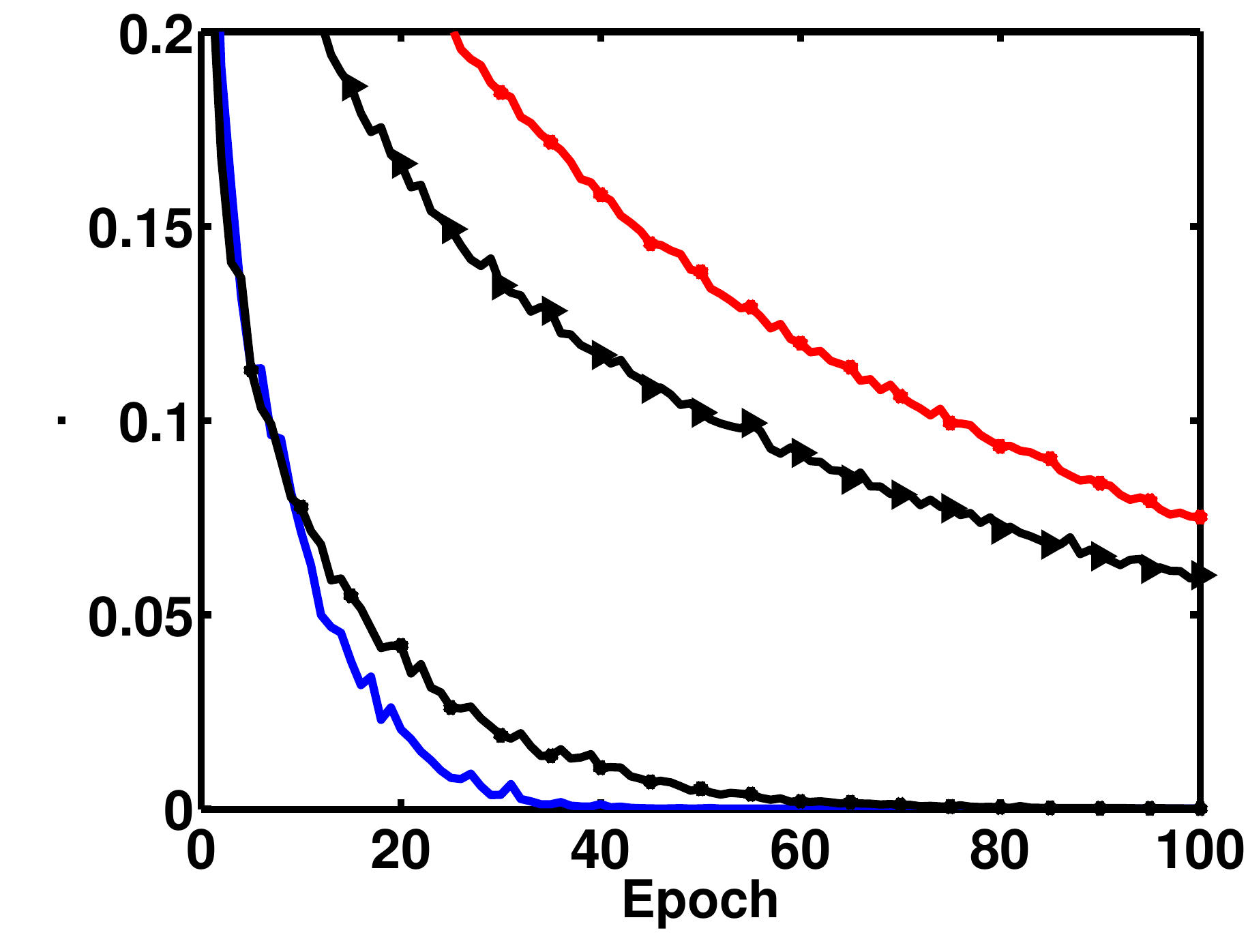}
  \end{tabular}
 }
 \hspace{-0.3in}
 \subfloat{
  \begin{tabular}{r}
   \includegraphics[width=\picwidth]{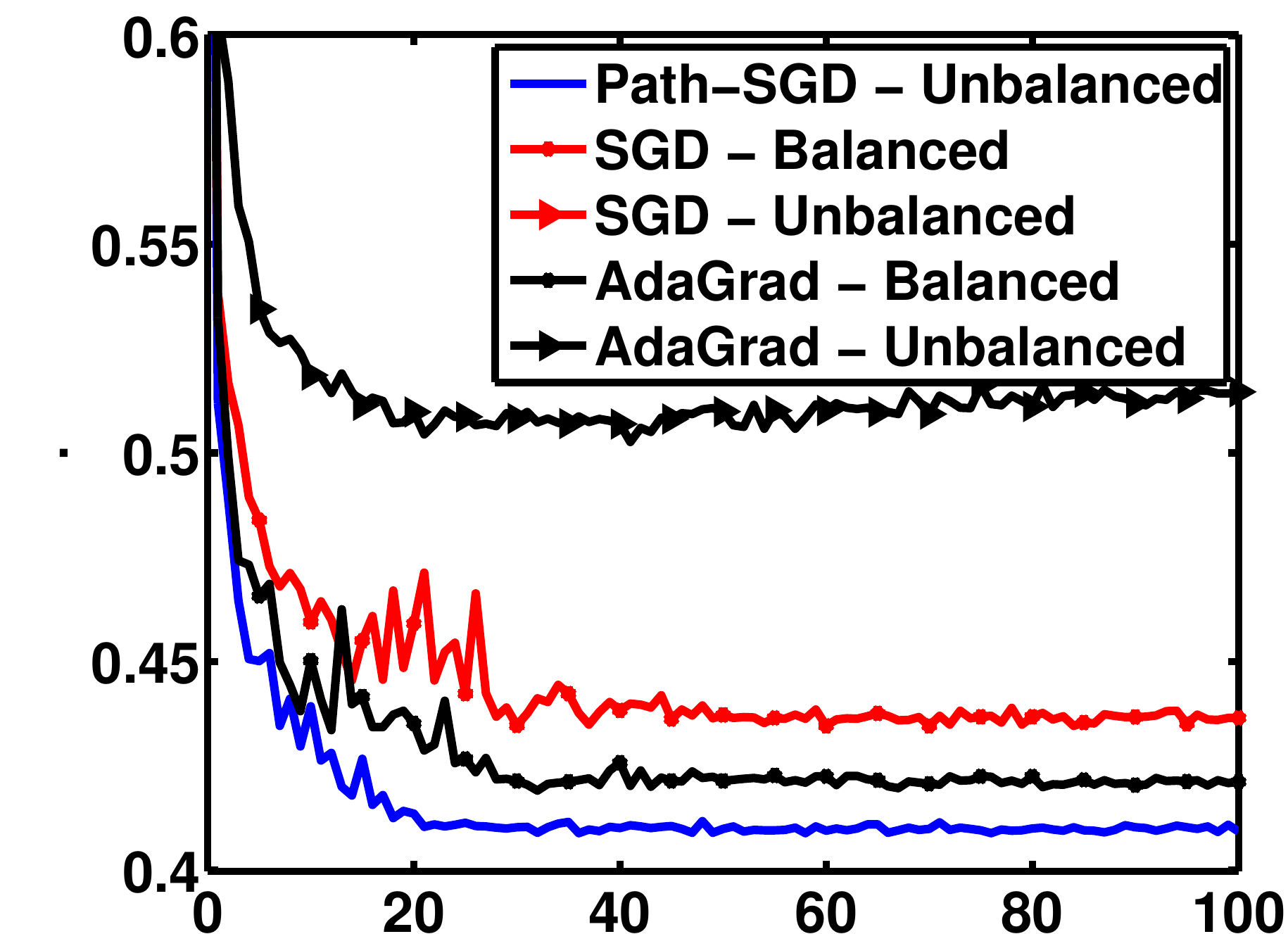} \\
   \includegraphics[width=\picwidth]{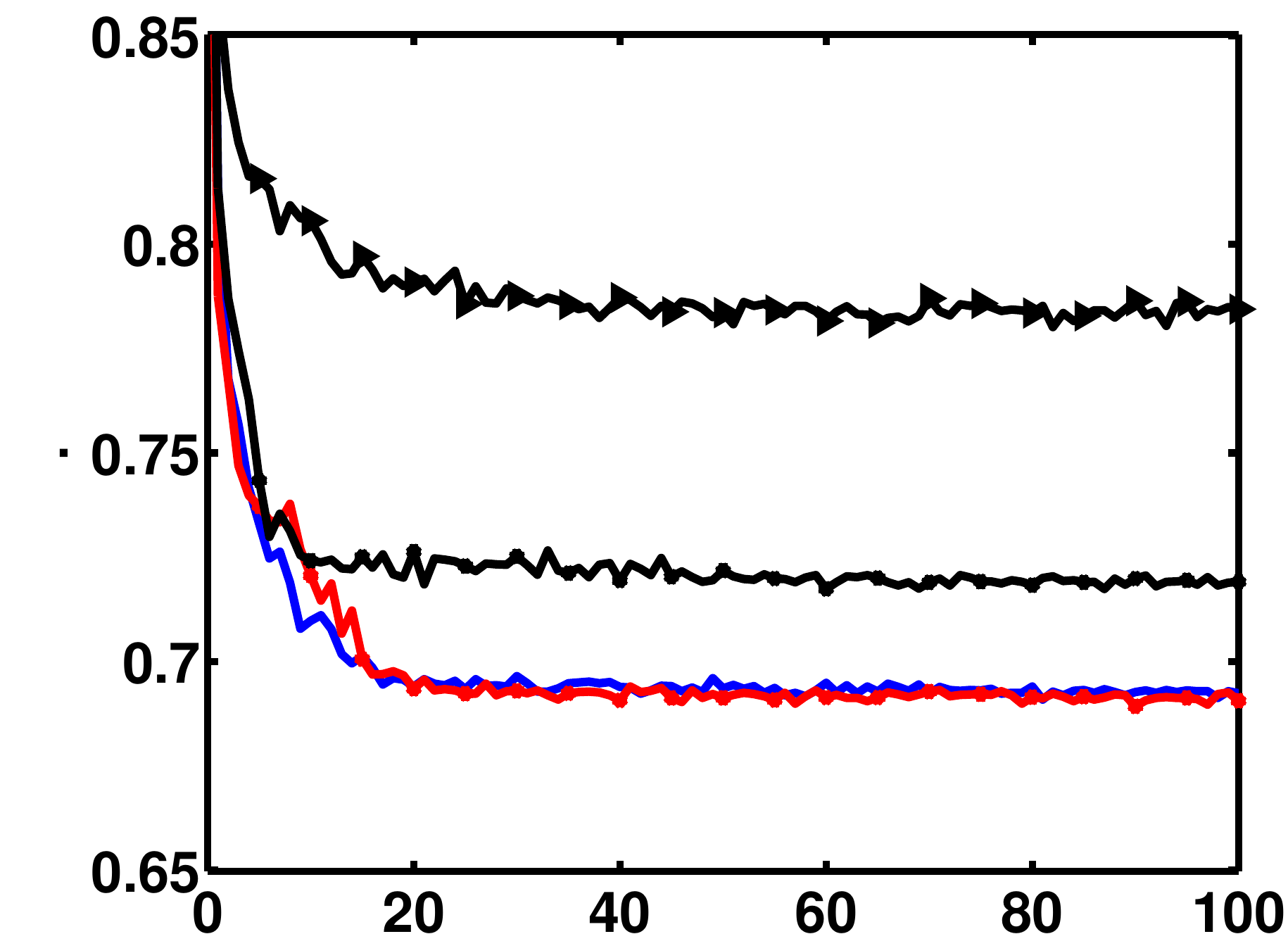} \\
   \includegraphics[width=1.9in]{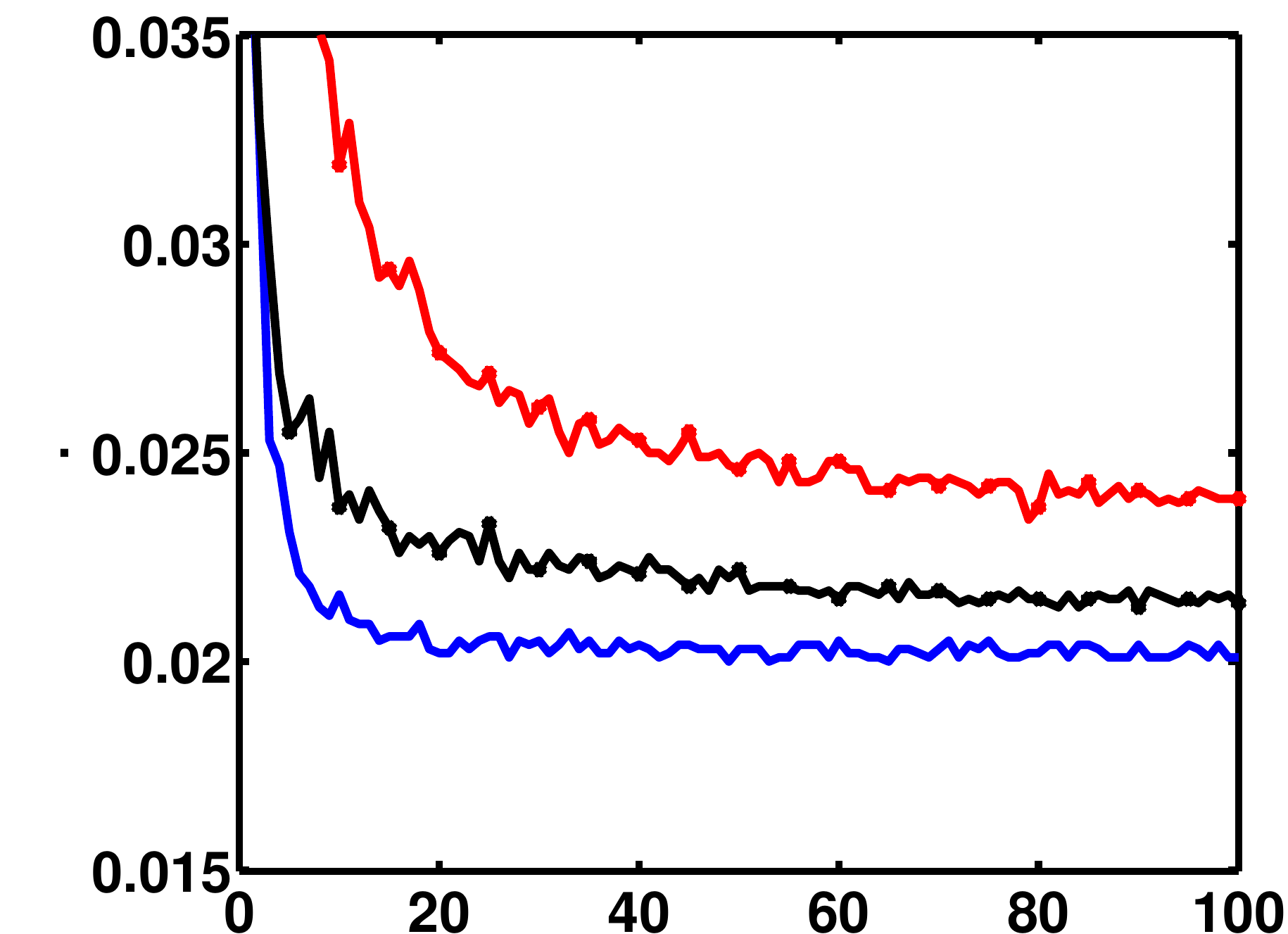} \\
   \includegraphics[width=1.85in]{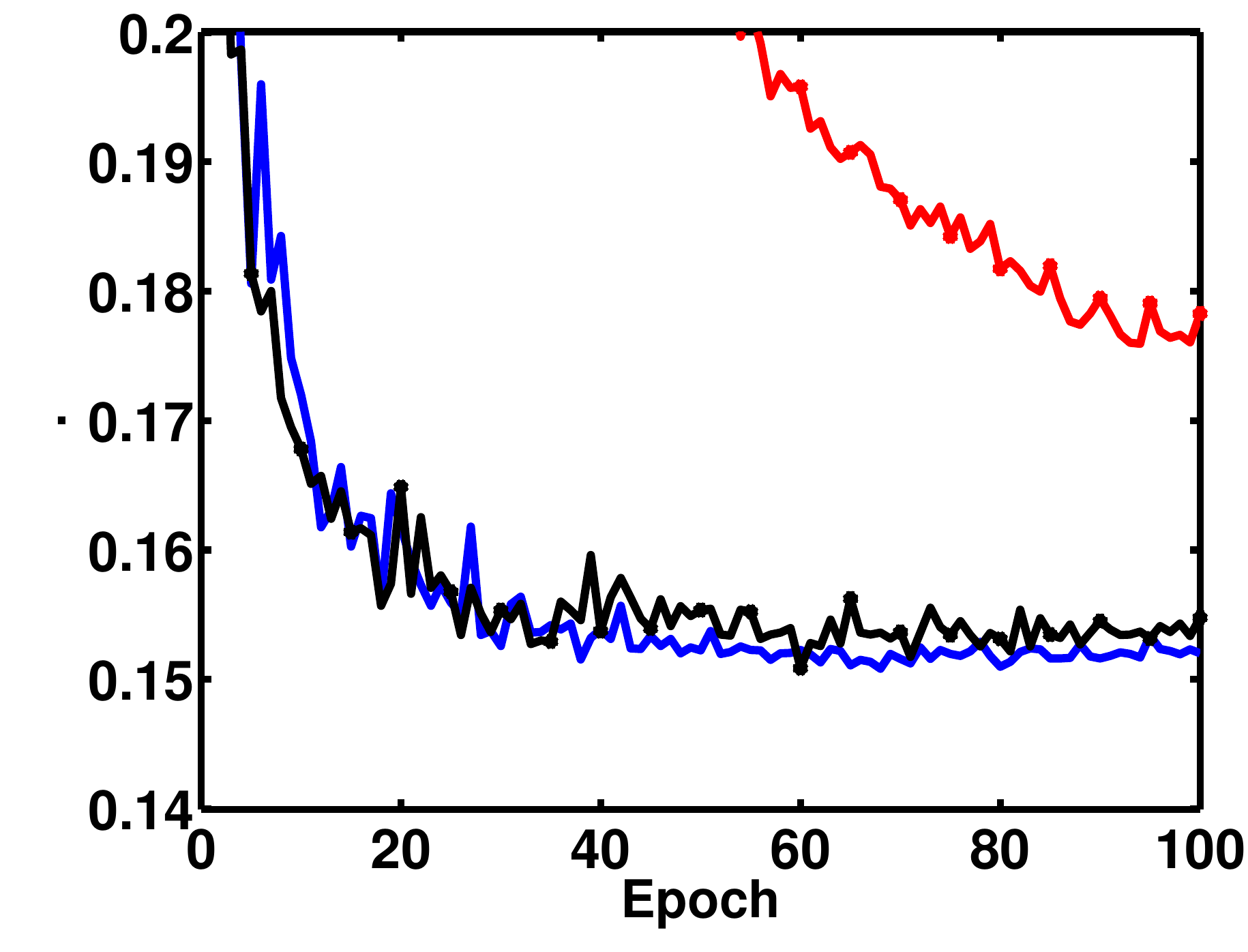}
  \end{tabular}
 }

 \begin{picture}(0,0)(0,0)
\rotatebox{90}{\put(342, 0){CIFAR-10}\put(240, 0){CIFAR-100}\put(147, 0){MNIST}\put(50, 0){SVHN}}
\end{picture}
  \begin{picture}(0,0)(0,0)
{\put(30, 420){\small Cross-Entropy Training Loss}\put(187, 420){\small 0/1 Training Error}\put(328, 420){ \small 0/1 Test Error}}
\end{picture}
 \caption{\small Learning curves using different optimization methods 
 for 4 datasets without dropout. Left panel displays the cross-entropy objective function; 
middle and right panels show the corresponding values of the training and test errors, where the values are reported on
different epochs during the course of optimization. Best viewed in color.}
 \label{fig:nodrop}
\vspace{-0.1in}
\end{figure}

In all of our experiments, we trained feed-forward networks with two hidden layers, each containing 4000 hidden units. We used mini-batches of size 100 and the step-size of $10^{-\alpha}$, where $\alpha$ is an integer between 0 and 10. To choose $\alpha$, for each dataset, we considered the validation errors over the validation set (10000 randomly chosen points that are kept out during the initial training) and picked the one that reaches the minimum error faster. We then trained the network over the entire training set. All the networks were trained both with and without dropout. When training with dropout, at each update step, we retained each unit with probability 0.5.

We tried both balanced and unbalanced initializations. In balanced initialization, incoming weights to each unit $v$ are initialized to i.i.d samples from a Gaussian distribution with standard deviation $1/\sqrt{\text{fan-in}(v)}$. In the unbalanced setting, we first initialized the weights to be the same as the balanced weights. 
We then picked 2000 hidden units randomly with replacement. For each unit, we multiplied its incoming edge and divided its outgoing edge by $10c$, where $c$ was chosen randomly from log-normal distribution.

\begin{figure}[t]
\hspace{0.1in}
 \subfloat{
  \begin{tabular}{r}
   \includegraphics[width=\picwidth]{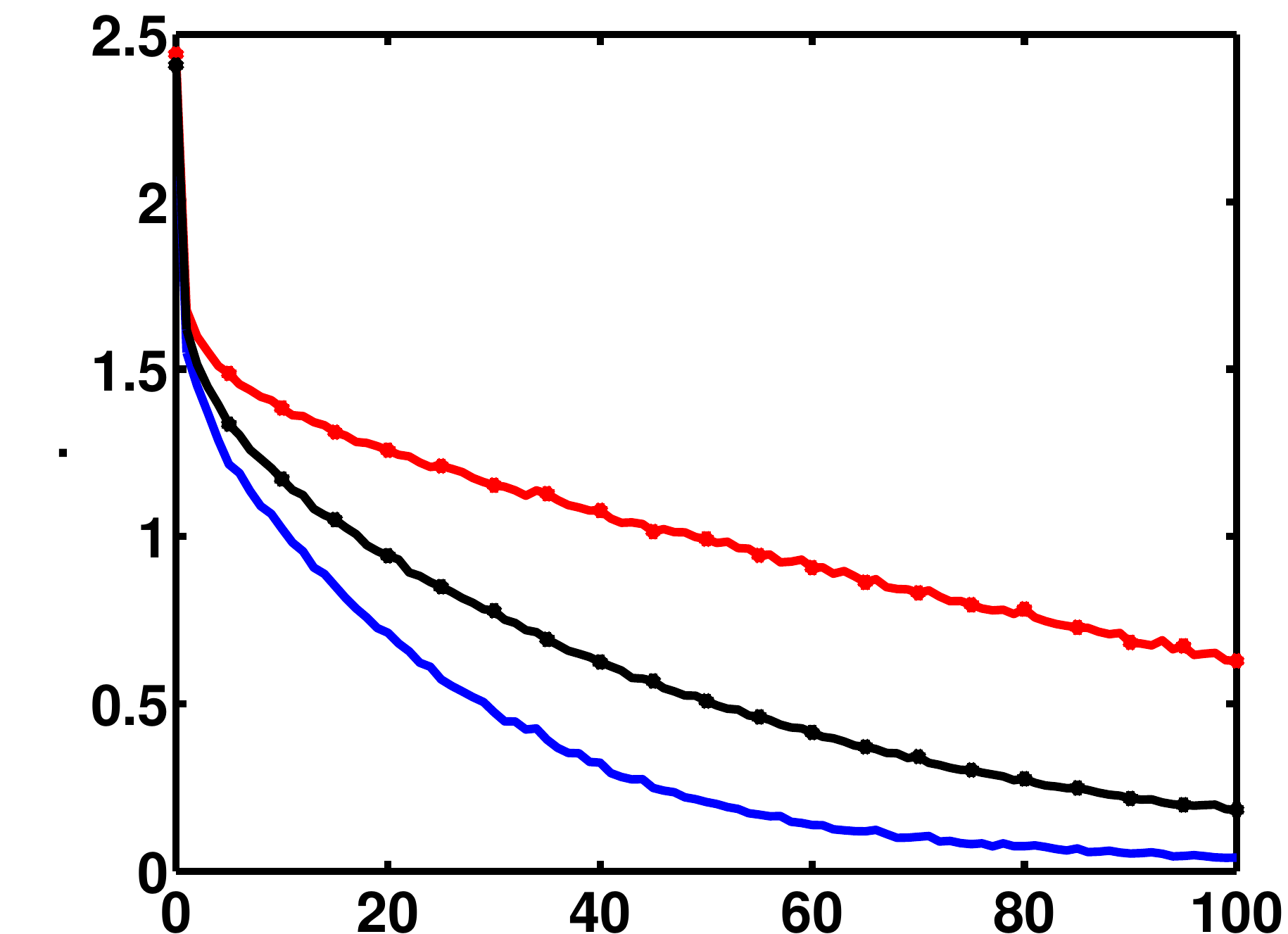} \\
   \includegraphics[width=1.79in]{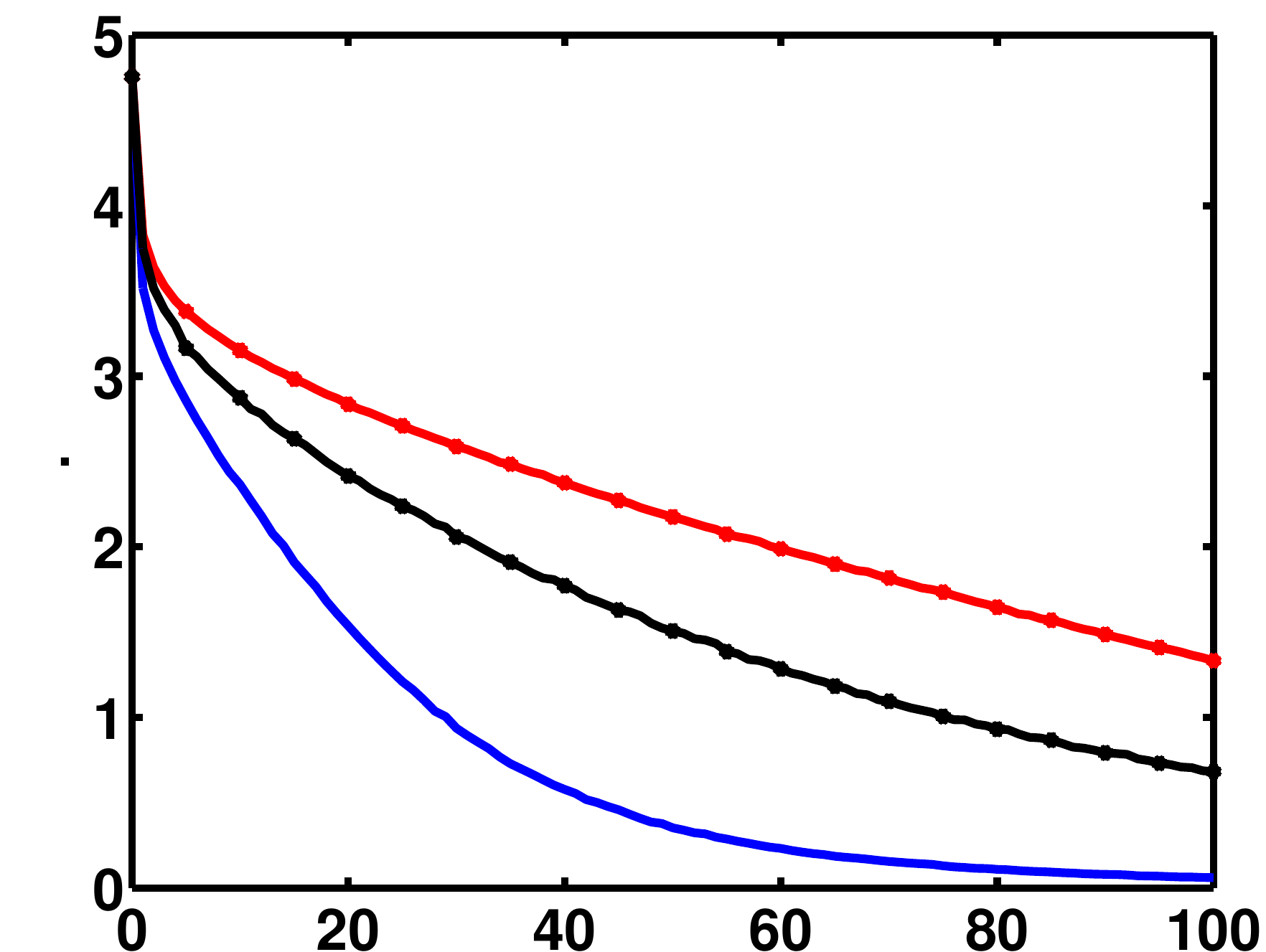} \\
   \includegraphics[width=1.85in]{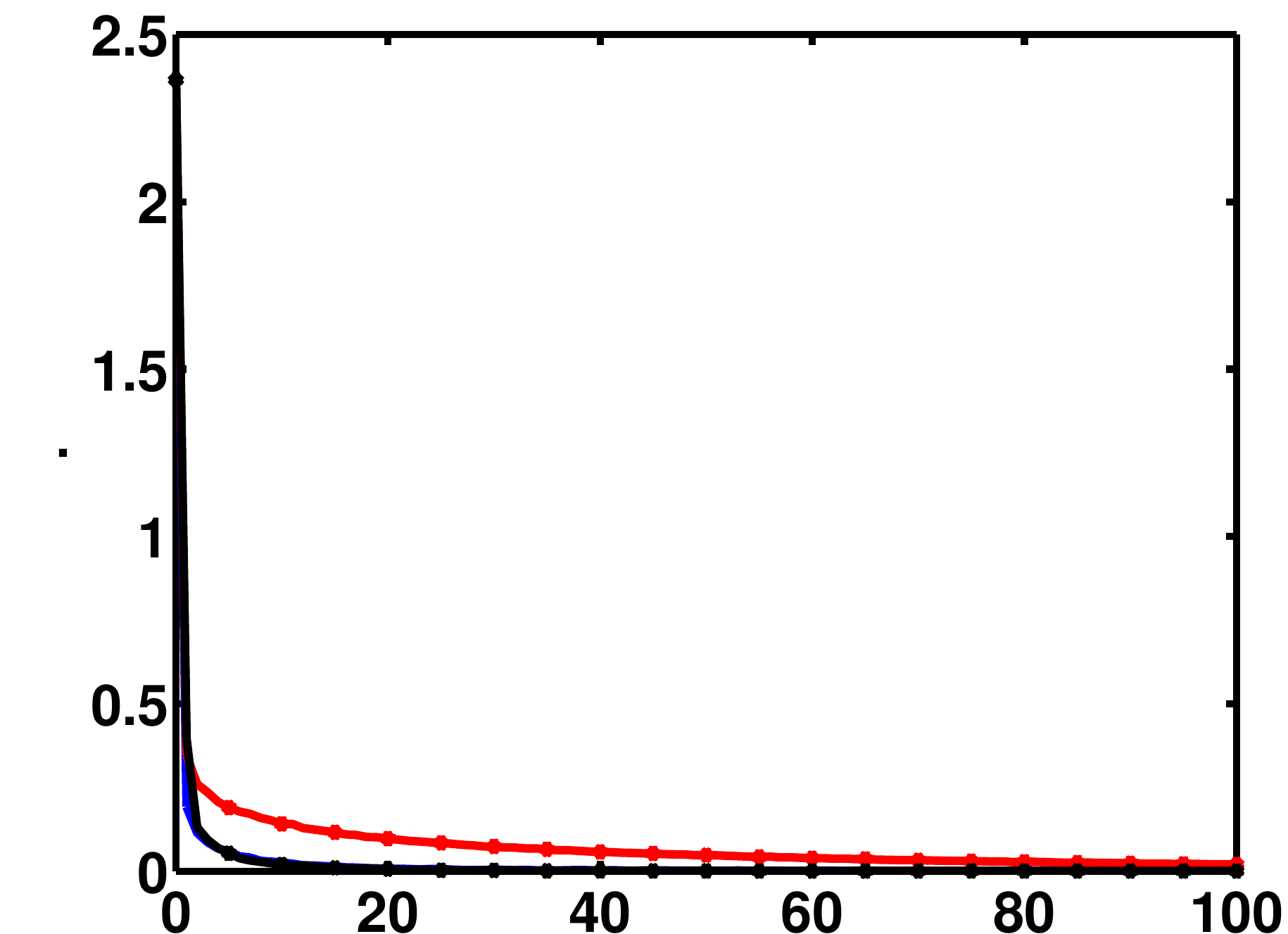} \\
   \includegraphics[width=1.85in]{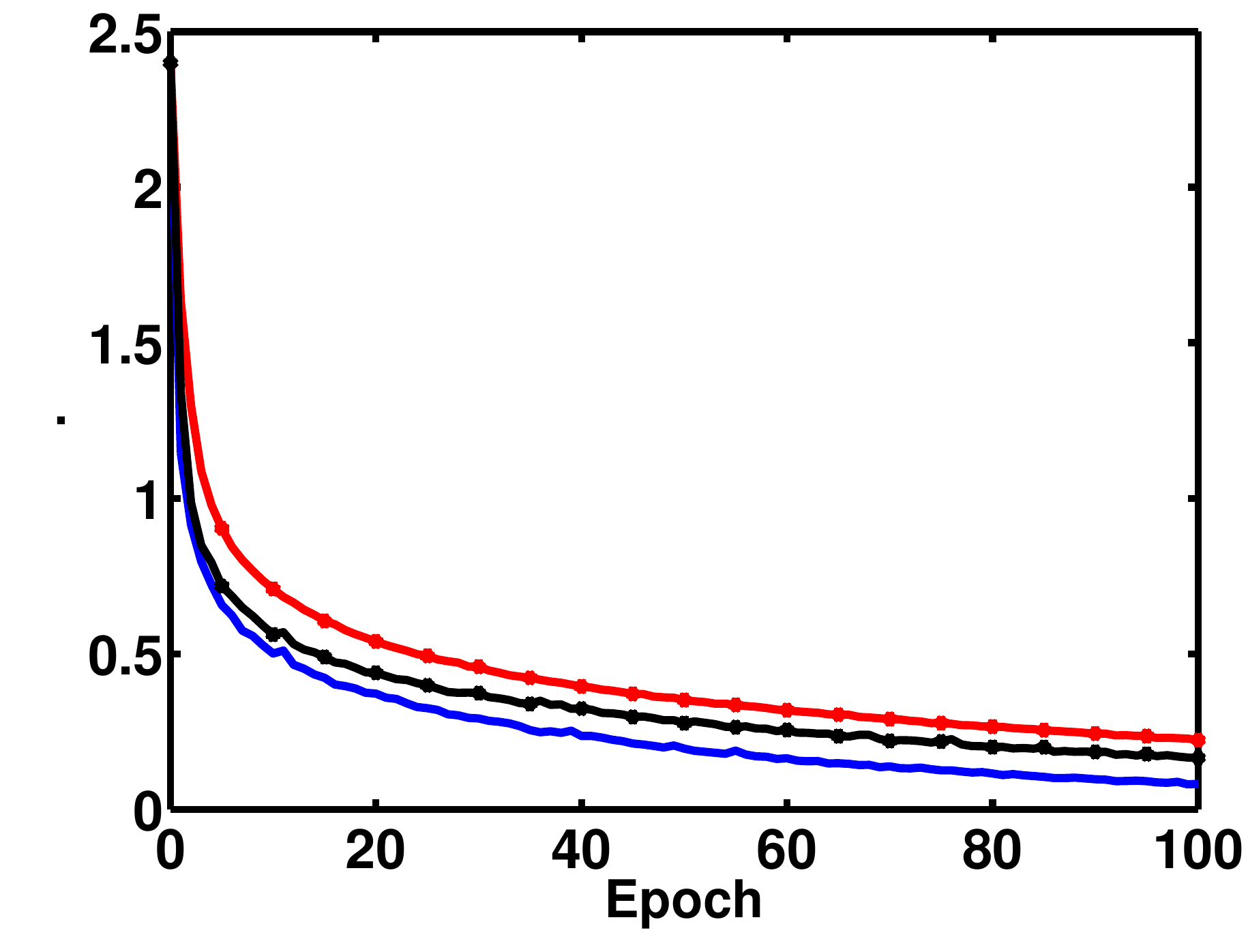}
  \end{tabular}
 }\hspace{-0.3in}
 \subfloat{
  \begin{tabular}{r}
   \includegraphics[width=\picwidth]{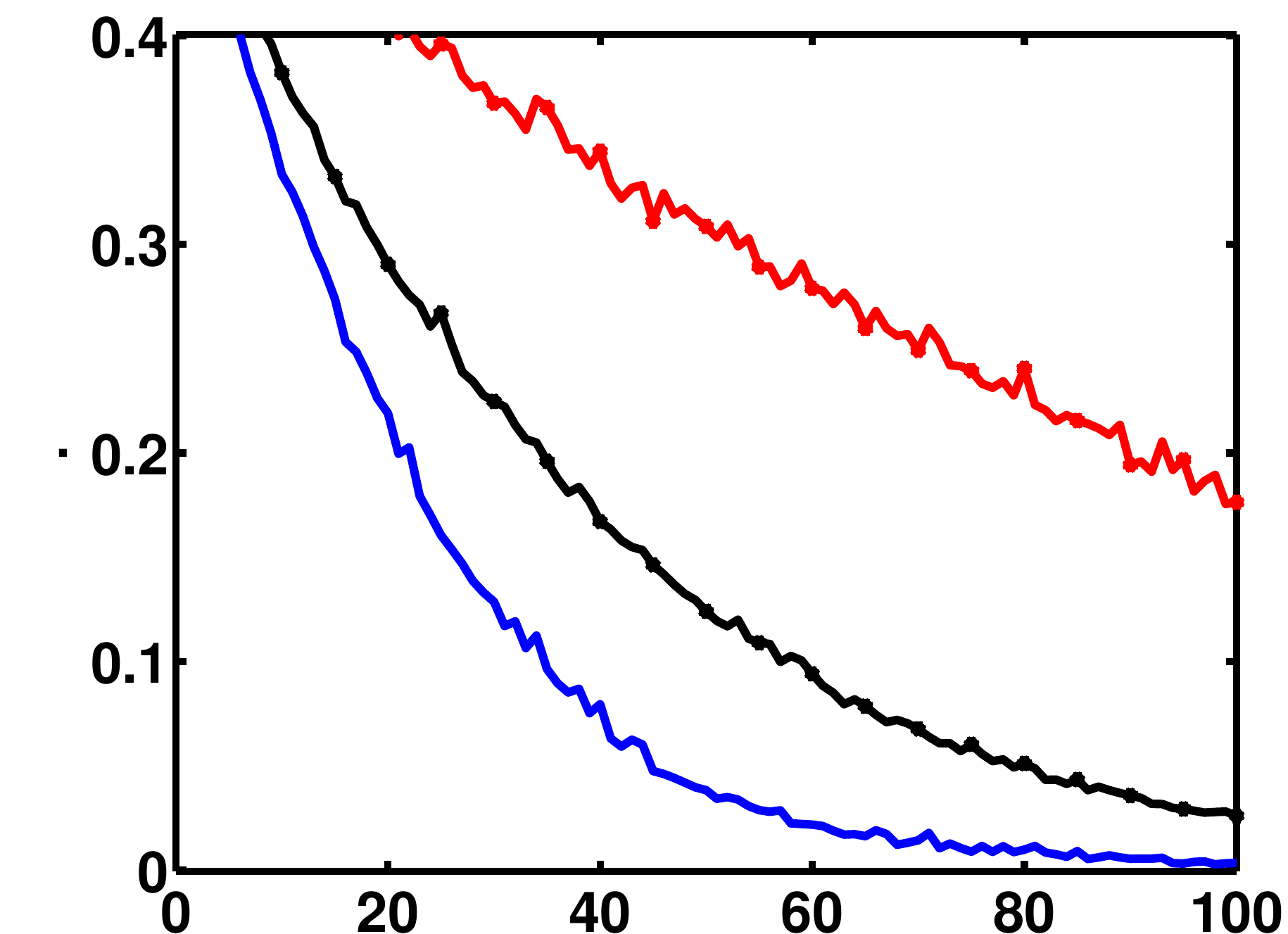} \\
   \includegraphics[width=\picwidth]{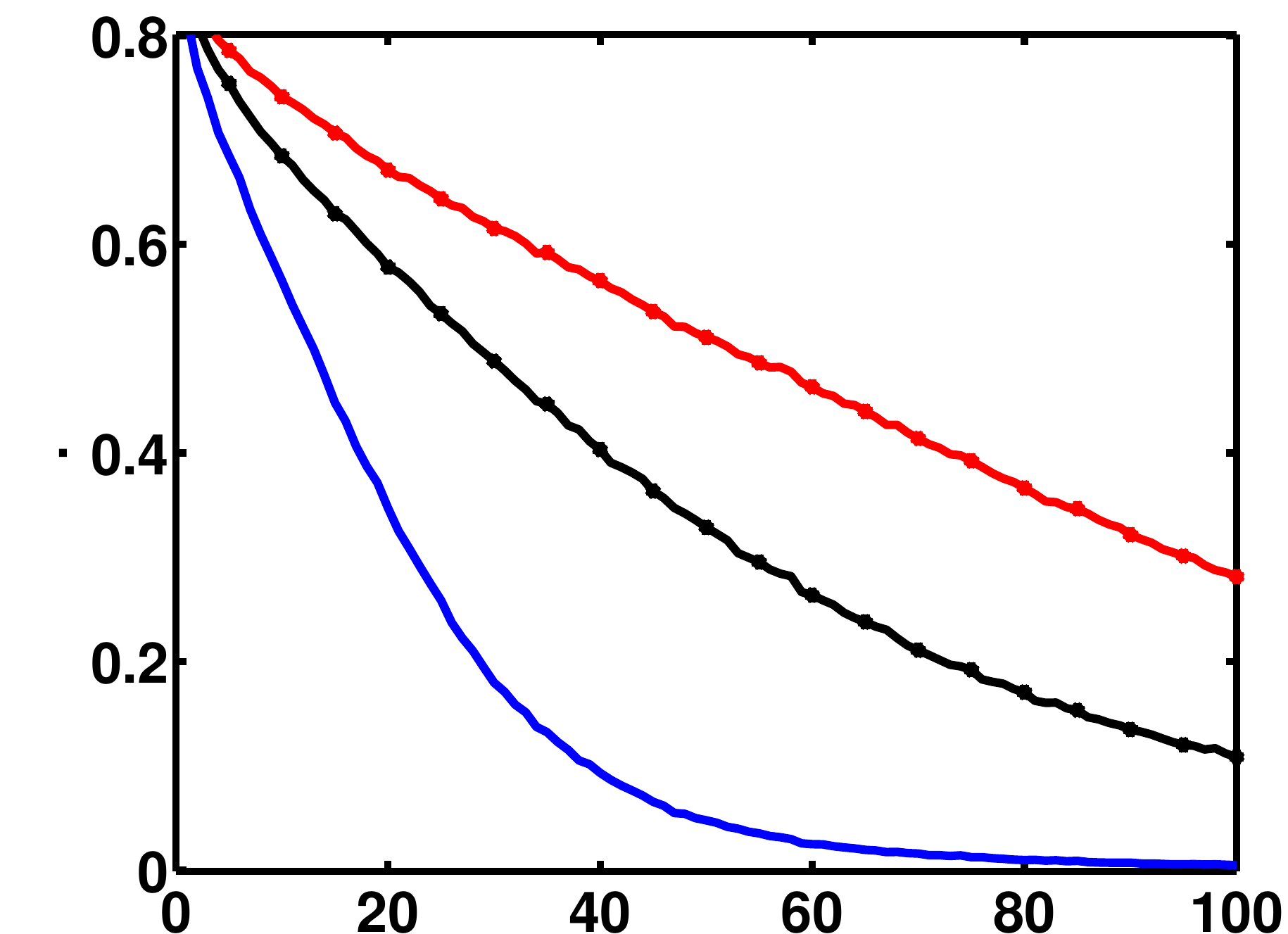} \\
   \includegraphics[width=1.9in]{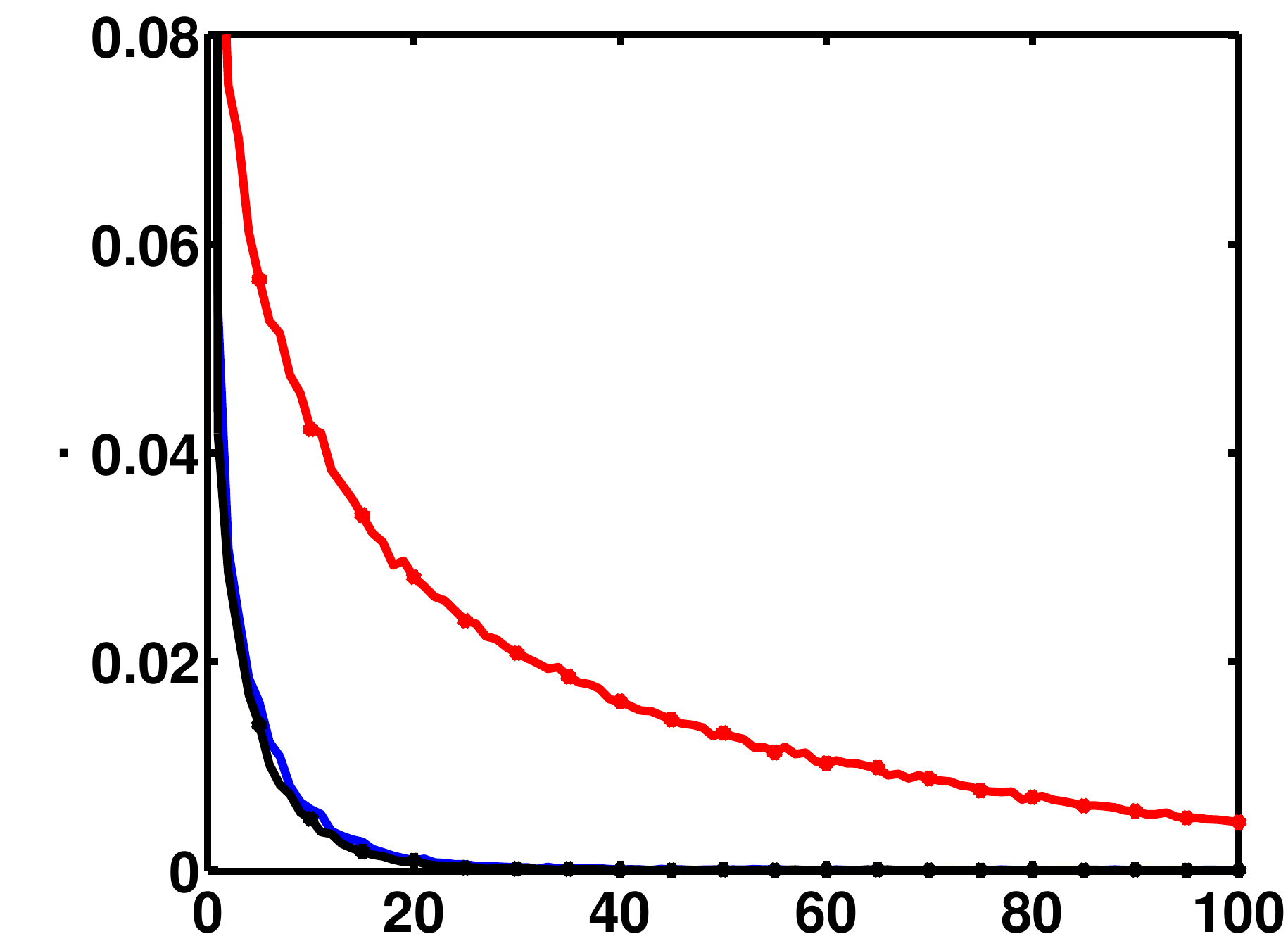} \\
   \includegraphics[width=\picwidth]{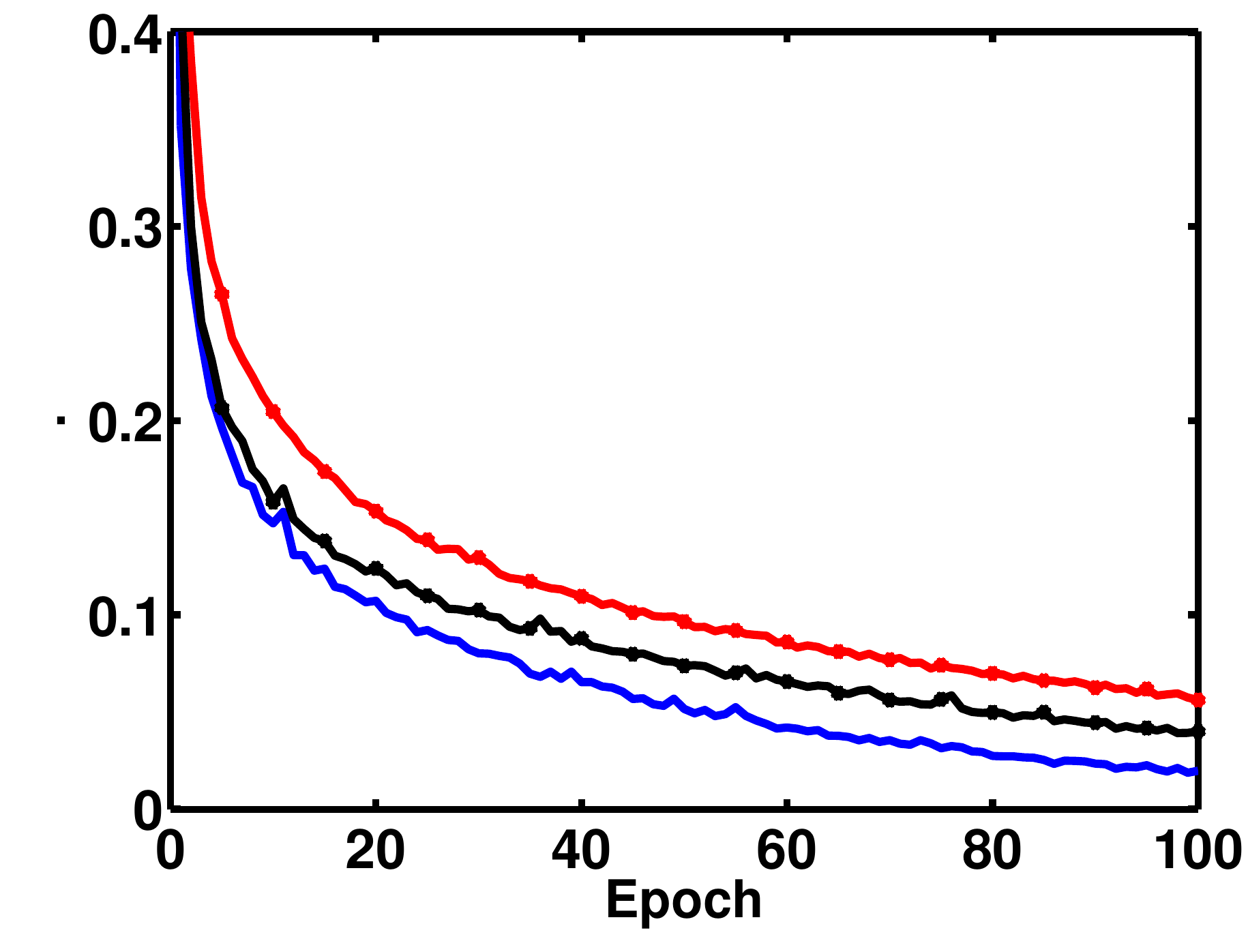}
  \end{tabular}
 }\hspace{-0.3in}
 \subfloat{
  \begin{tabular}{r}
   \includegraphics[width=\picwidth]{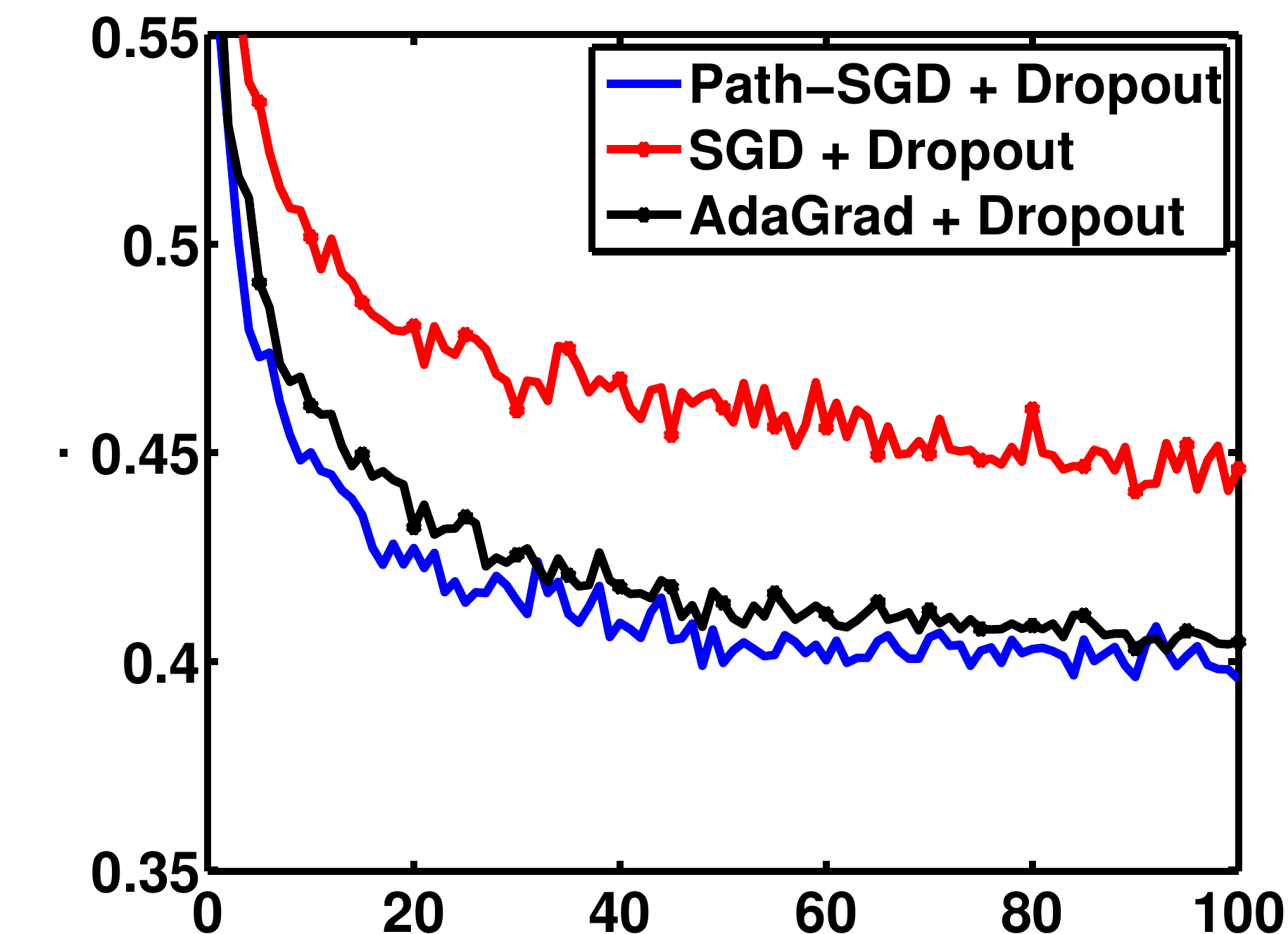} \\
   \includegraphics[width=\picwidth]{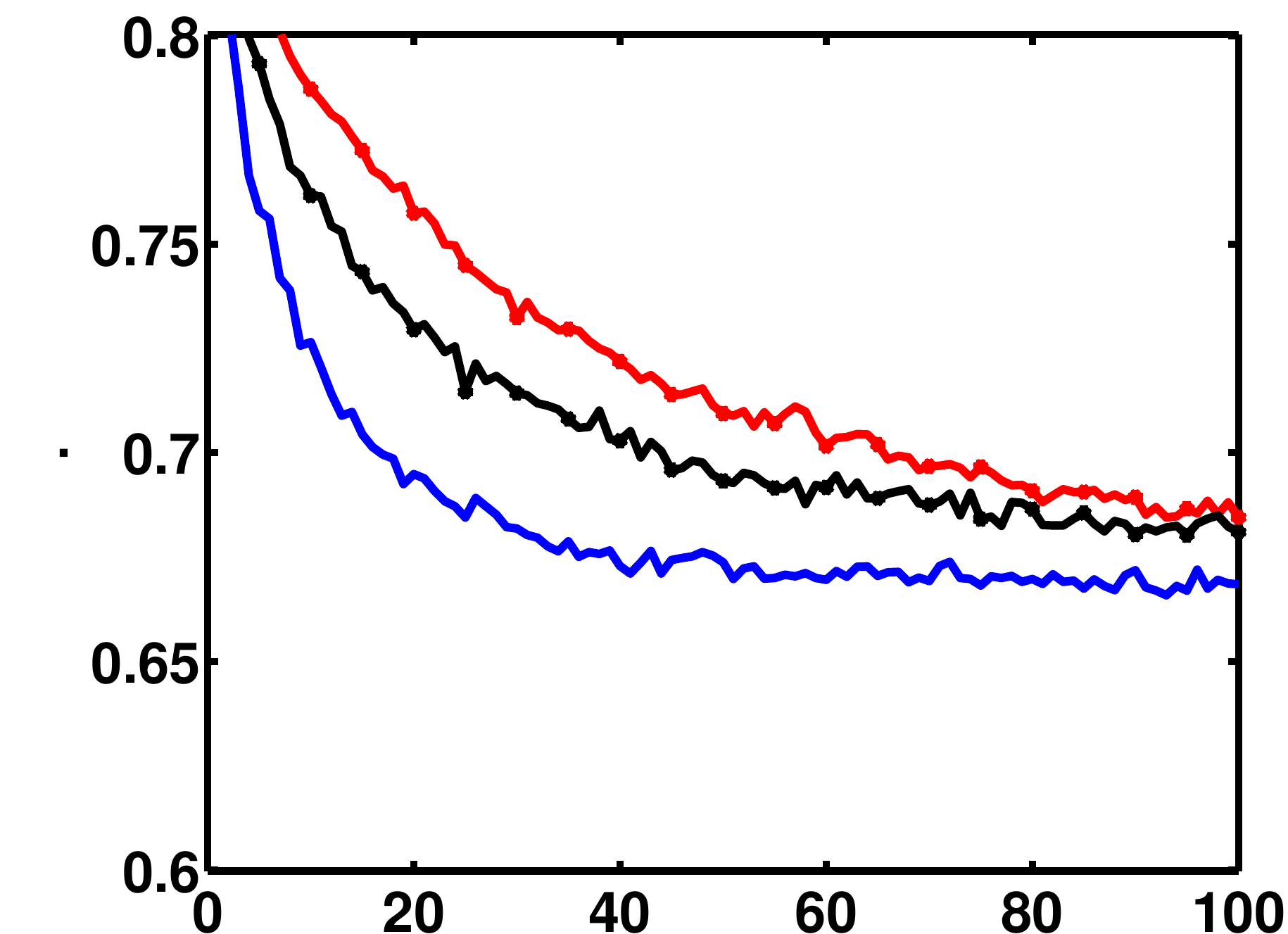} \\
   \includegraphics[width=1.9in]{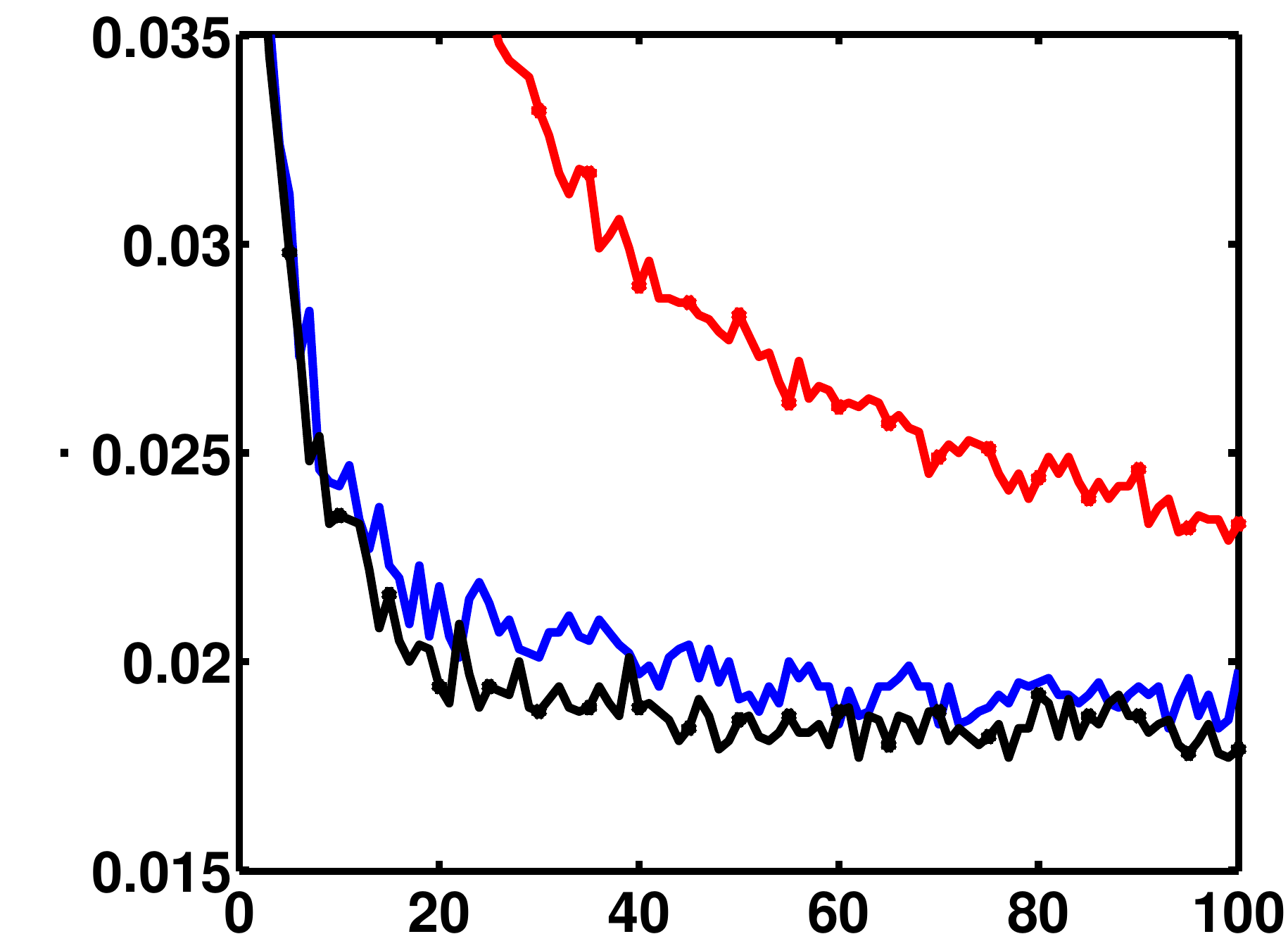} \\
   \includegraphics[width=\picwidth]{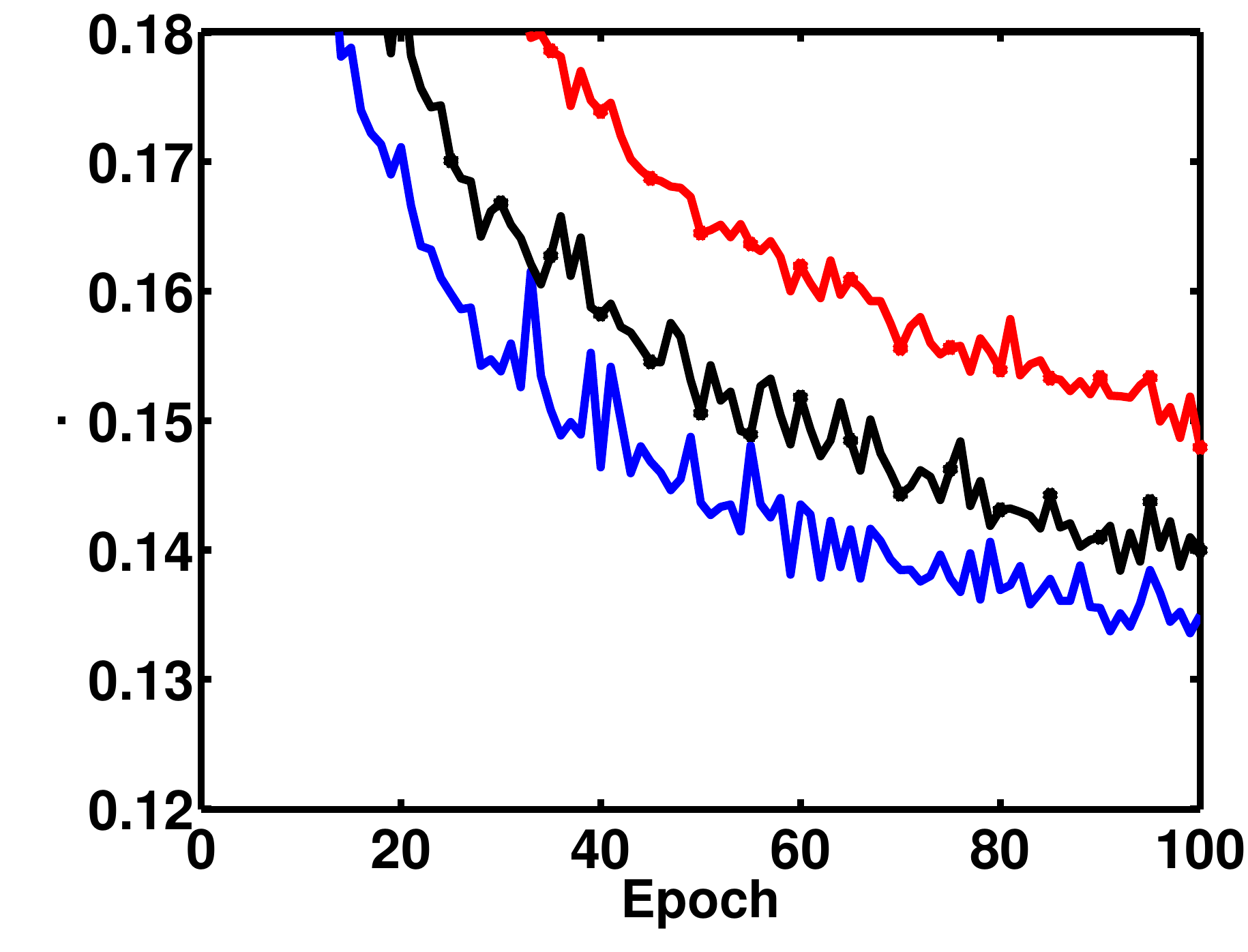}
  \end{tabular}
 }
 
 \begin{picture}(0,0)(0,0)
\rotatebox{90}{\put(342, 0){CIFAR-10}\put(240, 0){CIFAR-100}\put(147, 0){MNIST}\put(50, 0){SVHN}}
\end{picture}
  \begin{picture}(0,0)(0,0)
{\put(30, 420){\small Cross-Entropy Training Loss}\put(187, 420){\small 0/1 Training Error}\put(328, 420){ \small 0/1 Test Error}}
\end{picture}
\vspace{-0.1in}
 \caption{\small Learning curves using different optimization methods
 for 4 datasets with dropout. Left panel displays the cross-entropy objective function;     
middle and right panels show the corresponding values of the training and test errors. Best viewed in color.}
 \label{fig:dropout}
\vspace{-0.1in}
\end{figure}
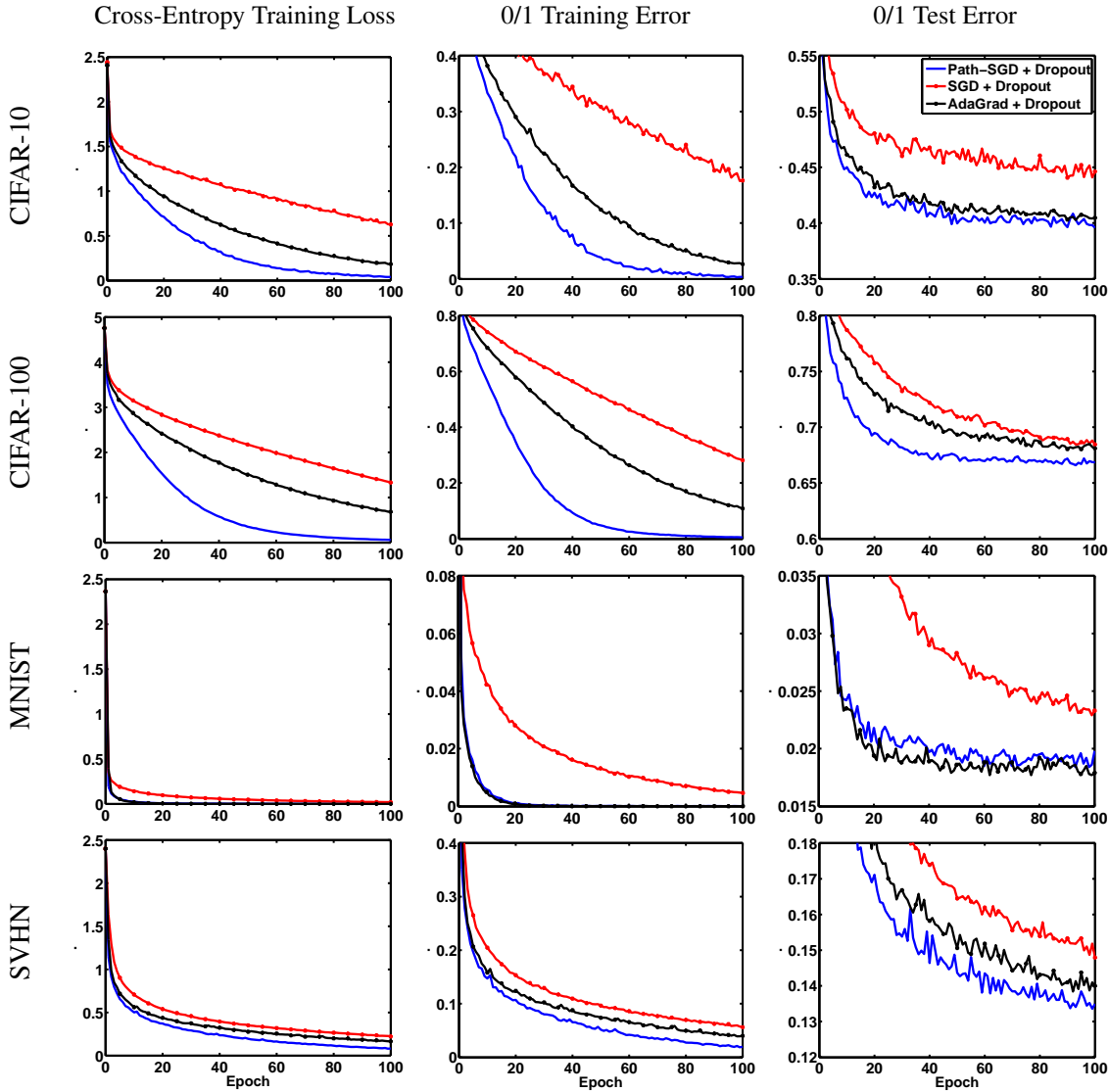

The optimization results without dropout are shown in
Figure~\ref{fig:nodrop}. For each of the four datasets, the plots for
objective function (cross-entropy), the training error and the test
error are shown from left to right where in each plot the values are
reported on different epochs during the optimization. Although we
proved that \RSGD updates are the same for balanced and unbalanced
initializations, to verify that despite numerical issues they are
indeed identical, we trained \RSGD with both balanced and unbalanced initializations. 
Since the curves were exactly the same we only show a single curve. 

We can see that as expected, the unbalanced initialization
considerably hurts the performance of SGD and AdaGrad (in many cases
their training and test errors are not even in the range of the plot
to be displayed), while \RSGD performs essentially the same. Another
interesting observation is that even in the balanced settings, not
only does \RSGD often get to the same value of objective function, training and test error faster, but also the final generalization error for \RSGD is sometimes considerably lower than SGD and AdaGrad (except CIFAR-100 where the generalization error for SGD is slightly better compared to \RSGD). The plots for test errors could also imply that implicit regularization due to steepest descent with respect to path-regularizer leads to a solution that generalizes better. This view is similar to observations in \cite{neyshabur15b} on the role of implicit regularization in deep learning.

The results for training with dropout are shown in
Figure~\ref{fig:dropout}, where here we suppressed the (very poor)
results on unbalanced initializations.  We observe that except for
MNIST, \RSGD convergences 
much faster than SGD or AdaGrad. It also generalizes better to
the test set, which again shows the effectiveness of path-normalized
updates.

The results suggest that \RSGD outperforms SGD and AdaGrad in two
different ways. First, it can achieve the same accuracy much faster
and second, the implicit regularization by \RSGD leads to a local
minima that can generalize better even when the training error is
zero. This can be better analyzed
by looking at the plots for more number of epochs which we have
provided in Figures~\ref{fig:more1} and \ref{fig:more2}. We should also point that
Path-SGD can be easily combined with AdaGrad to take advantage of the
adaptive stepsize or used together with a momentum term. This could
potentially perform even better compare to \RSGD.

\section{Discussion}

We revisited the choice of the Euclidean geometry on the weights of
RELU networks, suggested an alternative optimization method
approximately corresponding to a different geometry, and showed that
using such an alternative geometry can be beneficial.  In this work we show
proof-of-concept success, and we expect Path-SGD to be beneficial also
in large-scale training for very deep convolutional networks.
Combining Path-SGD with AdaGrad, with momentum or with other
optimization heuristics might further enhance results.

Although we do believe Path-SGD is a very good optimization method,
and is an easy plug-in for SGD, we hope this work will also inspire
others to consider other geometries, other regularizers and
perhaps better, update rules.
A particular property of Path-SGD is its rescaling invariance, which
we argue is appropriate for RELU networks.  But Path-SGD is certainly
not the only rescaling invariant update possible, and other invariant
geometries might be even better.

Finally, we choose to use steepest descent because of its simplicity
of implementation.  A better choice might be mirror descent with
respect to an appropriate potential function, but such a construction
seems particularly challenging considering the non-convexity of neural
networks.

\subsubsection*{Acknowledgments}
Research was partially funded by NSF award IIS-1302662 and Intel ICRI-CI. We thank Hao Tang for insightful discussions.

\bibliographystyle{plain}
\bibliography{ref}

\newpage

\vspace{-2in}
\begin{figure}
\hspace{0.1in}
 \subfloat{
  \begin{tabular}{r}
   \includegraphics[width=\picwidth]{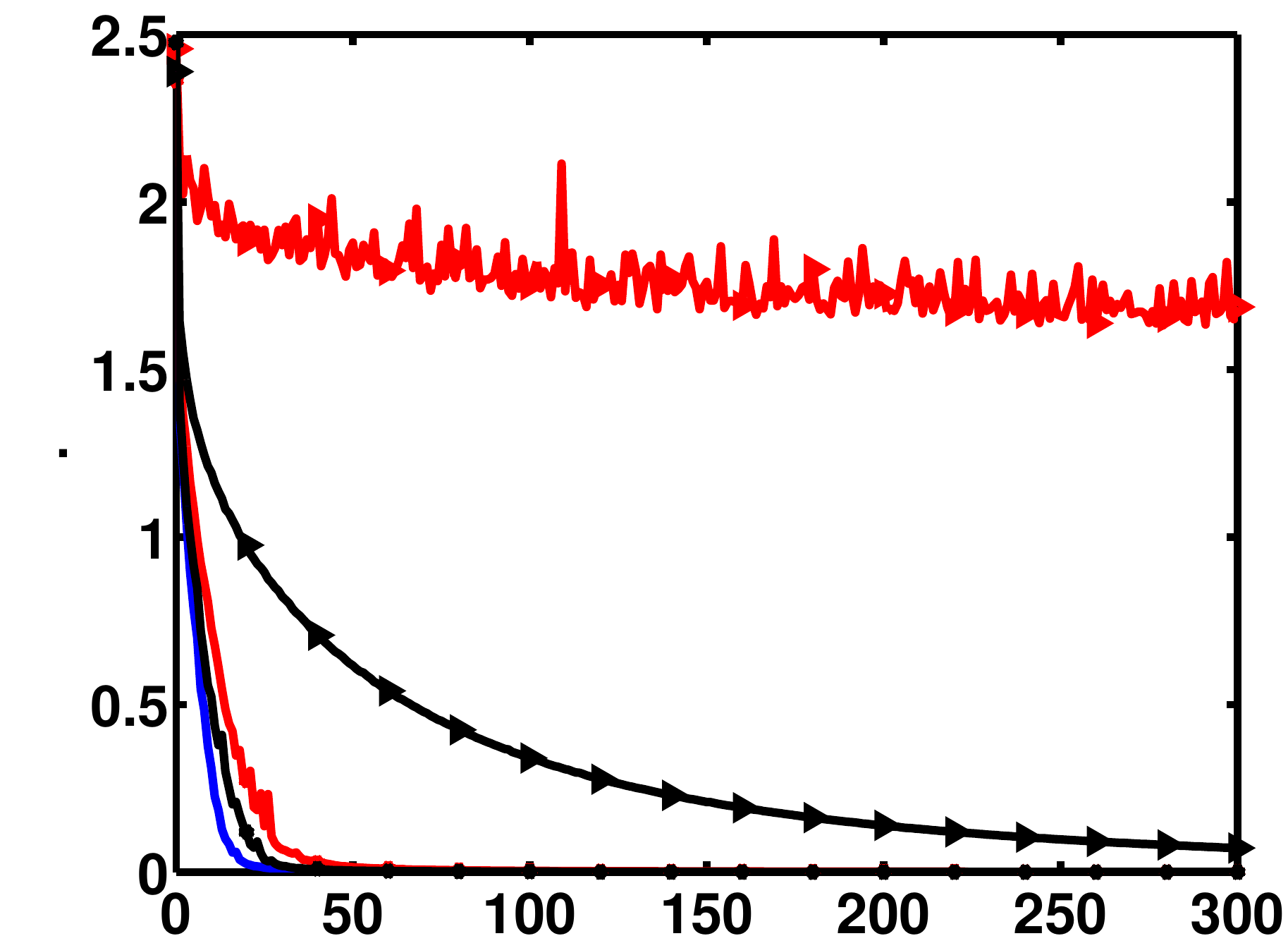} \\
   \includegraphics[width=1.76in]{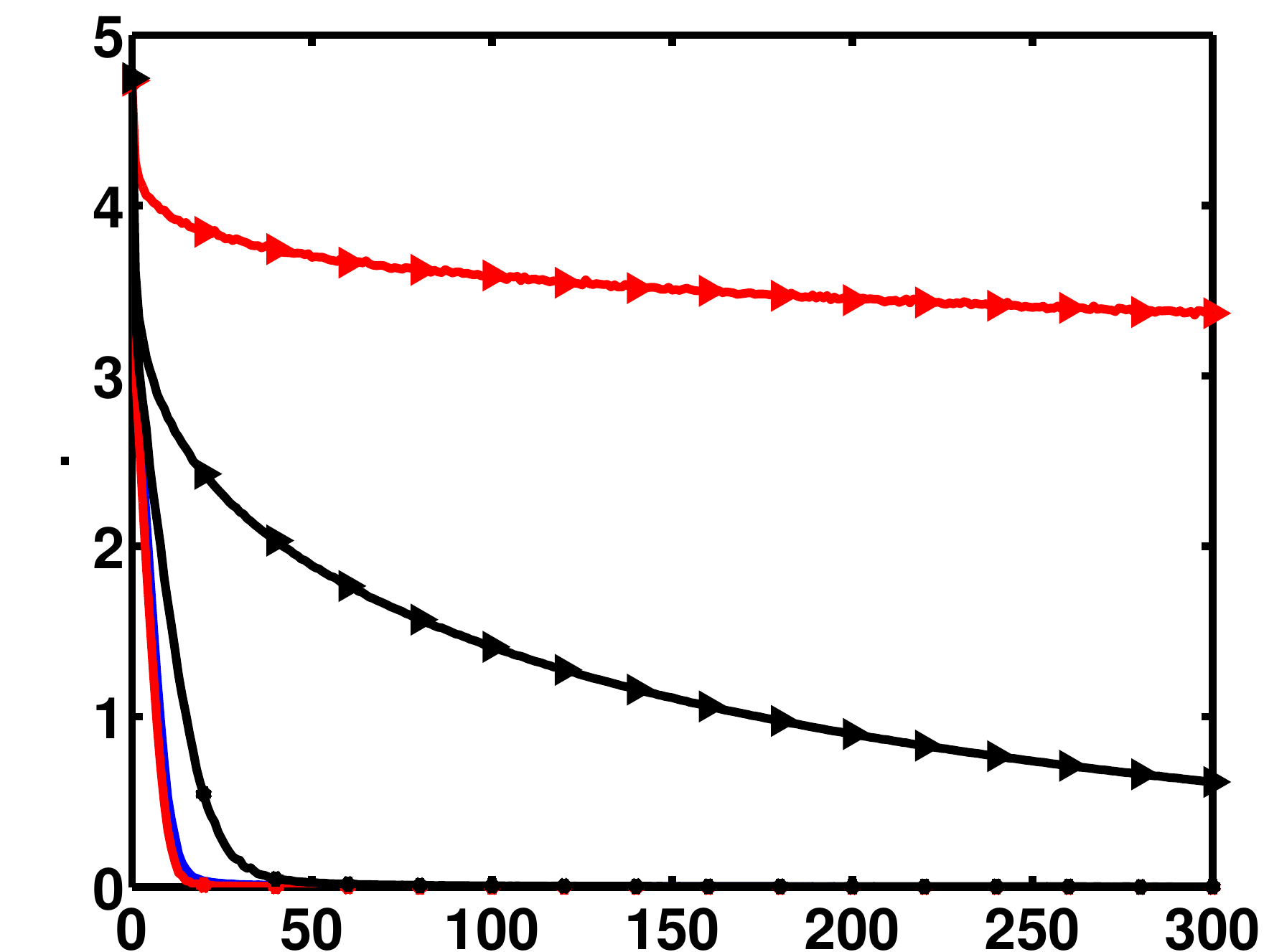} \\
   \includegraphics[width=\picwidth]{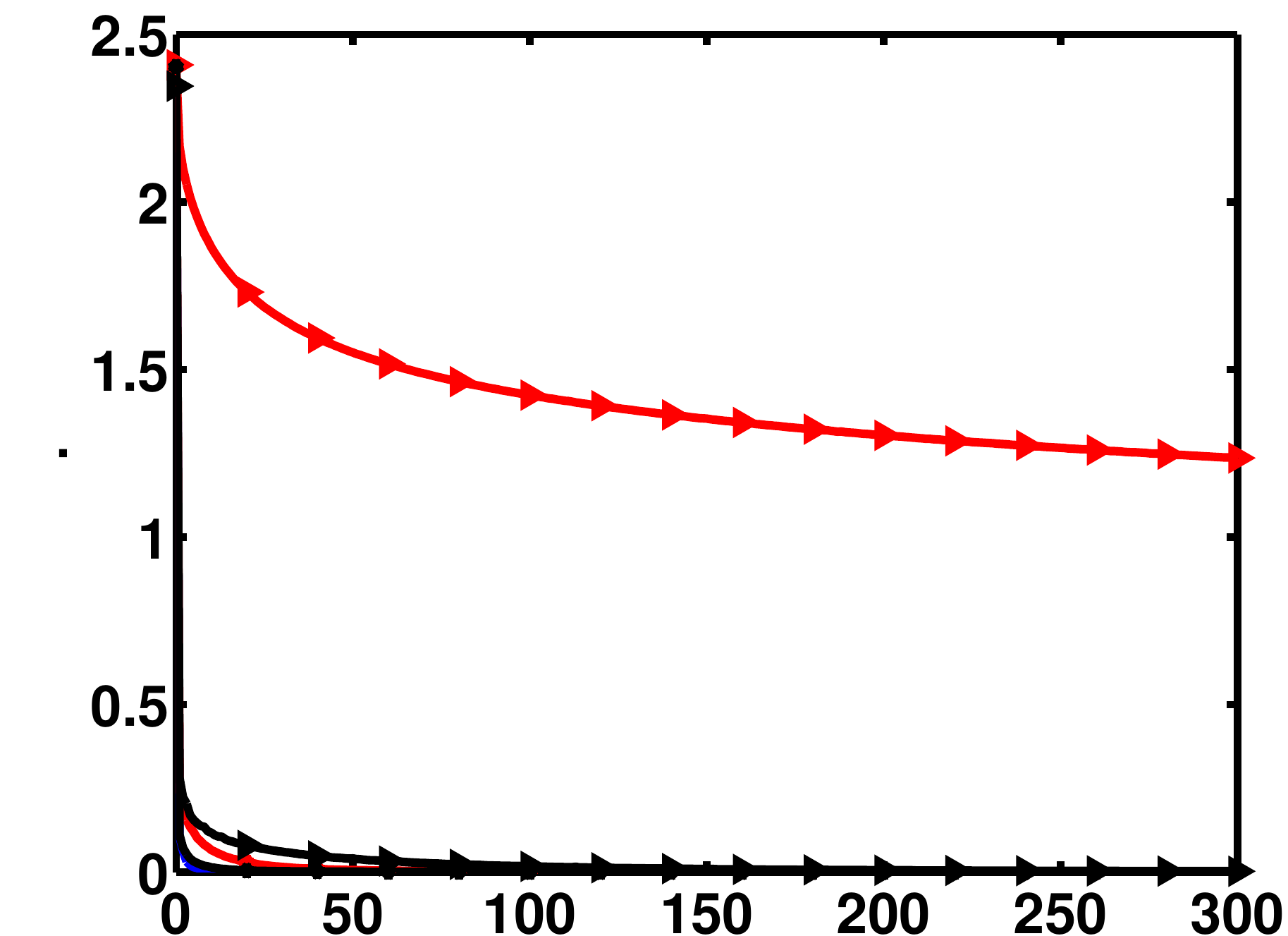} \\
   \includegraphics[width=\picwidth]{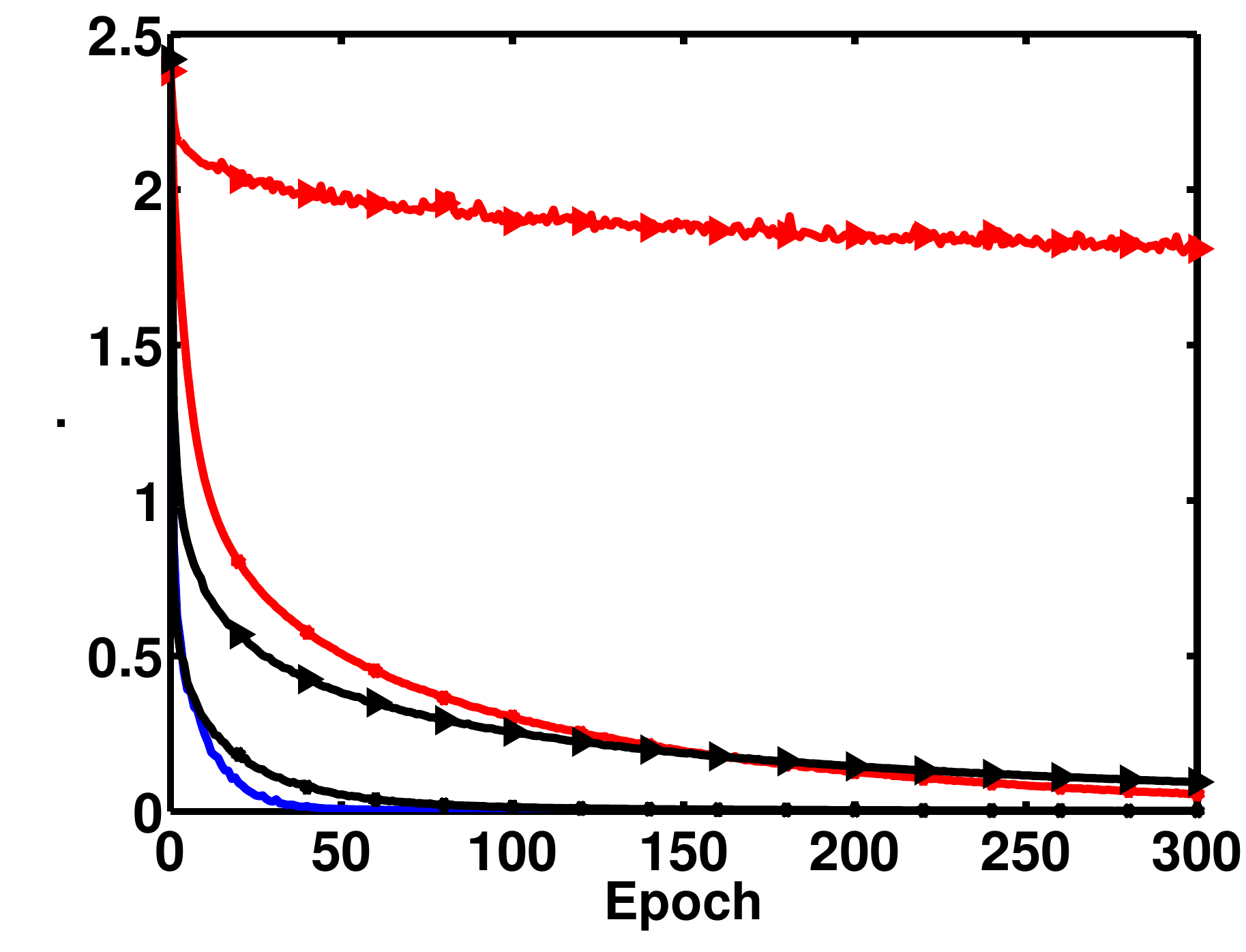}
  \end{tabular}
 }
 \hspace{-0.3in}
 \subfloat{
  \begin{tabular}{r}
   \includegraphics[width=\picwidth]{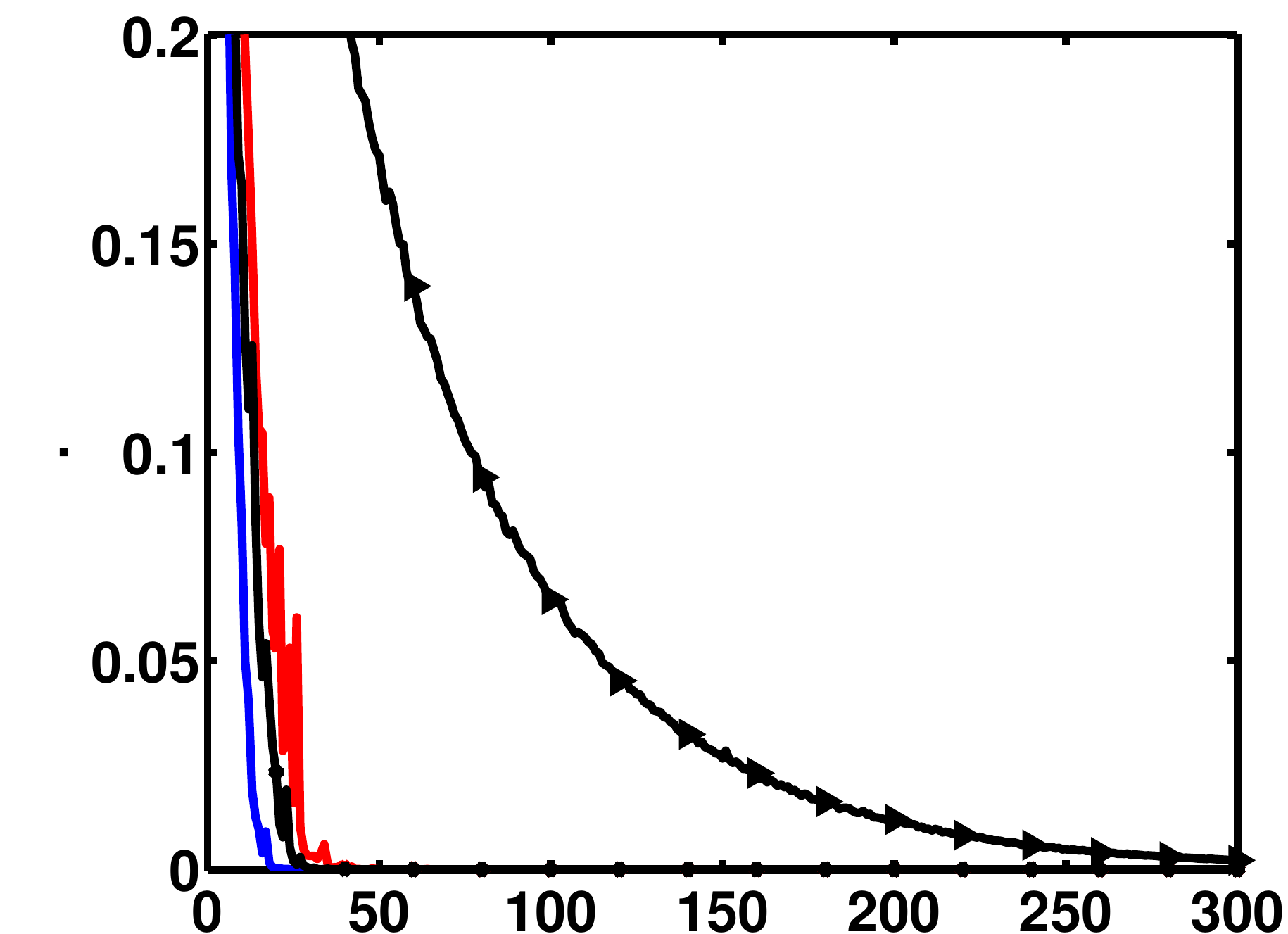} \\
   \includegraphics[width=\picwidth]{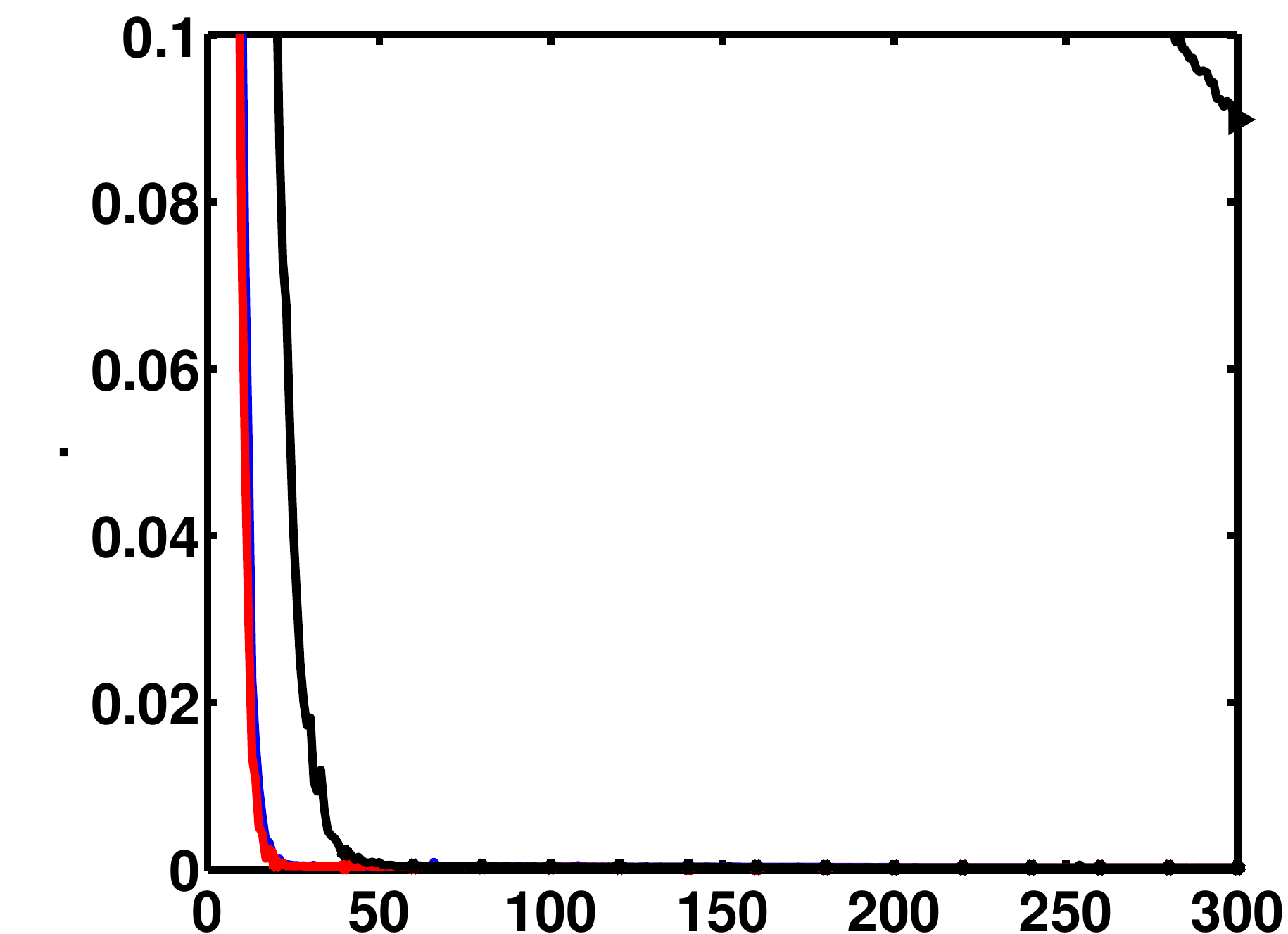} \\
   \includegraphics[width=1.9in]{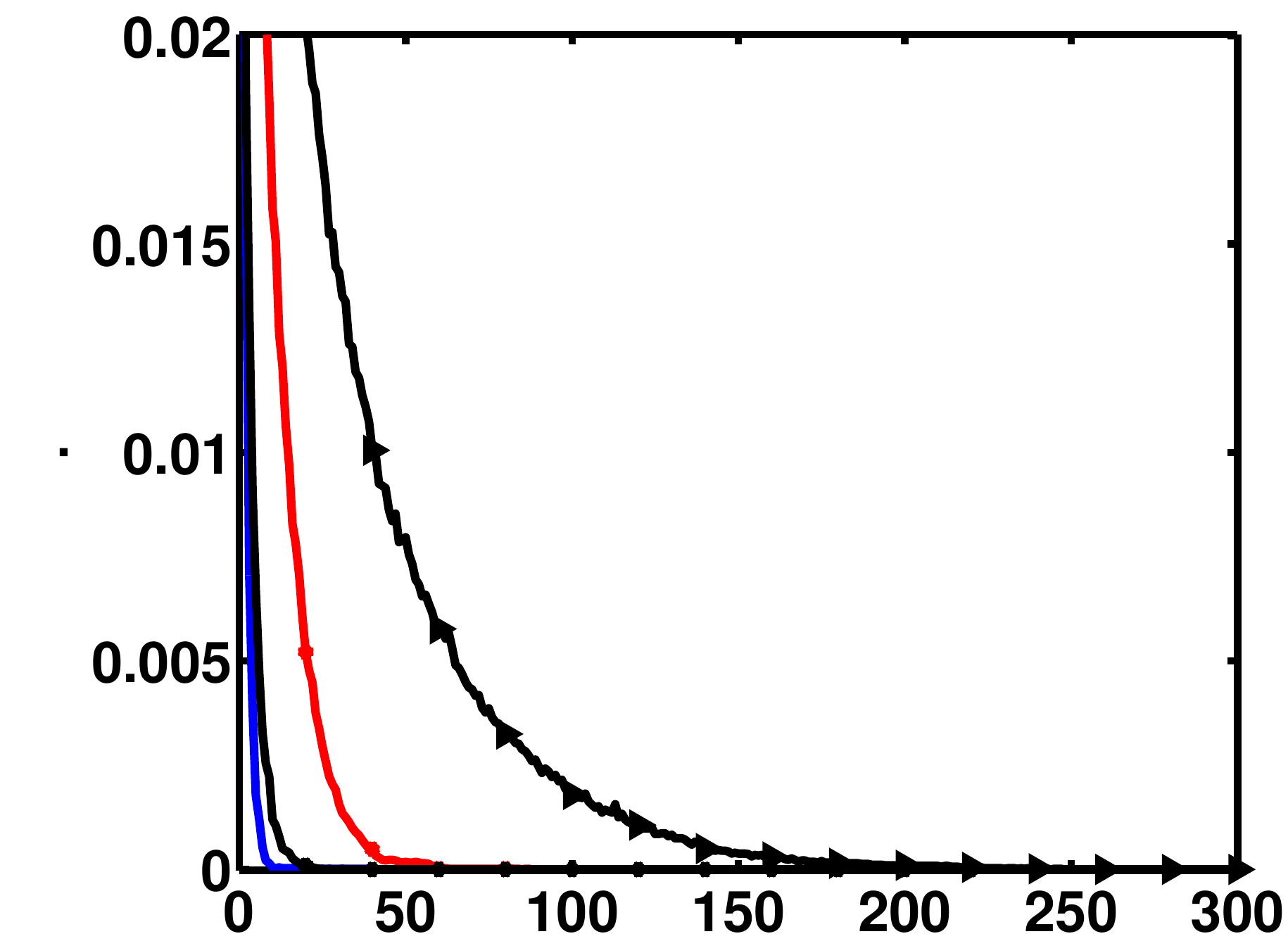} \\
   \includegraphics[width=1.85in]{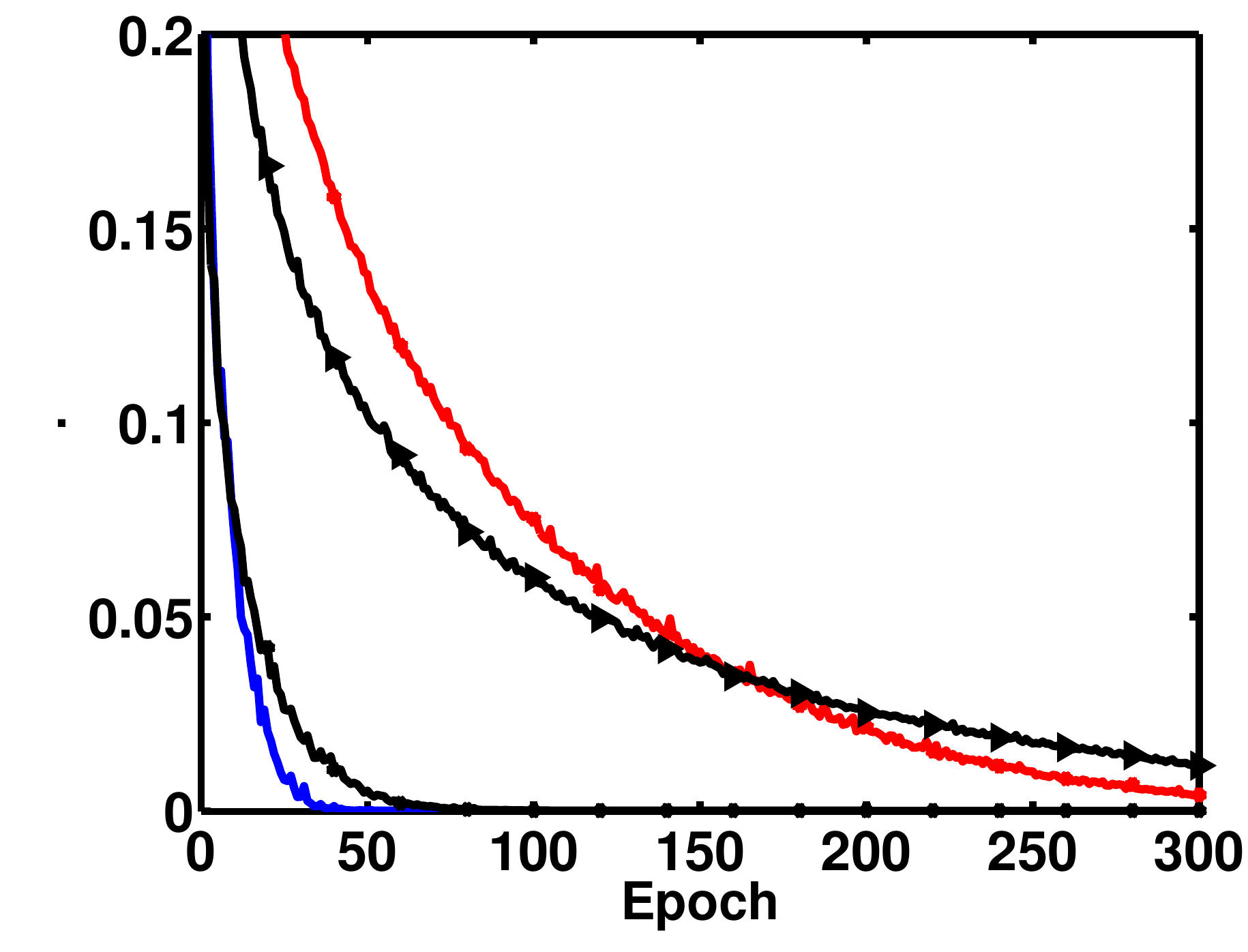}
  \end{tabular}
 }
 \hspace{-0.3in}
 \subfloat{
  \begin{tabular}{r}
   \includegraphics[width=\picwidth]{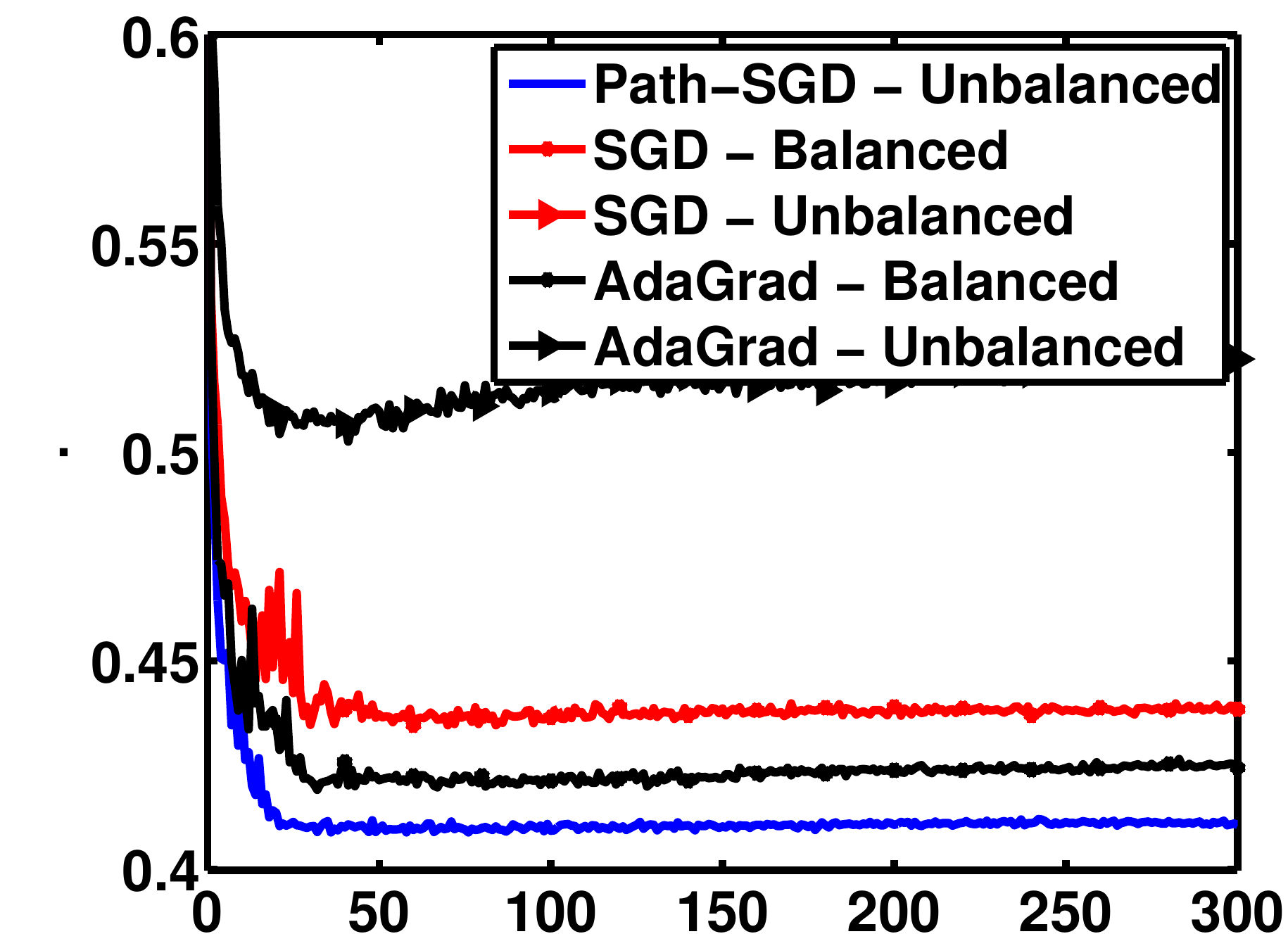} \\
   \includegraphics[width=\picwidth]{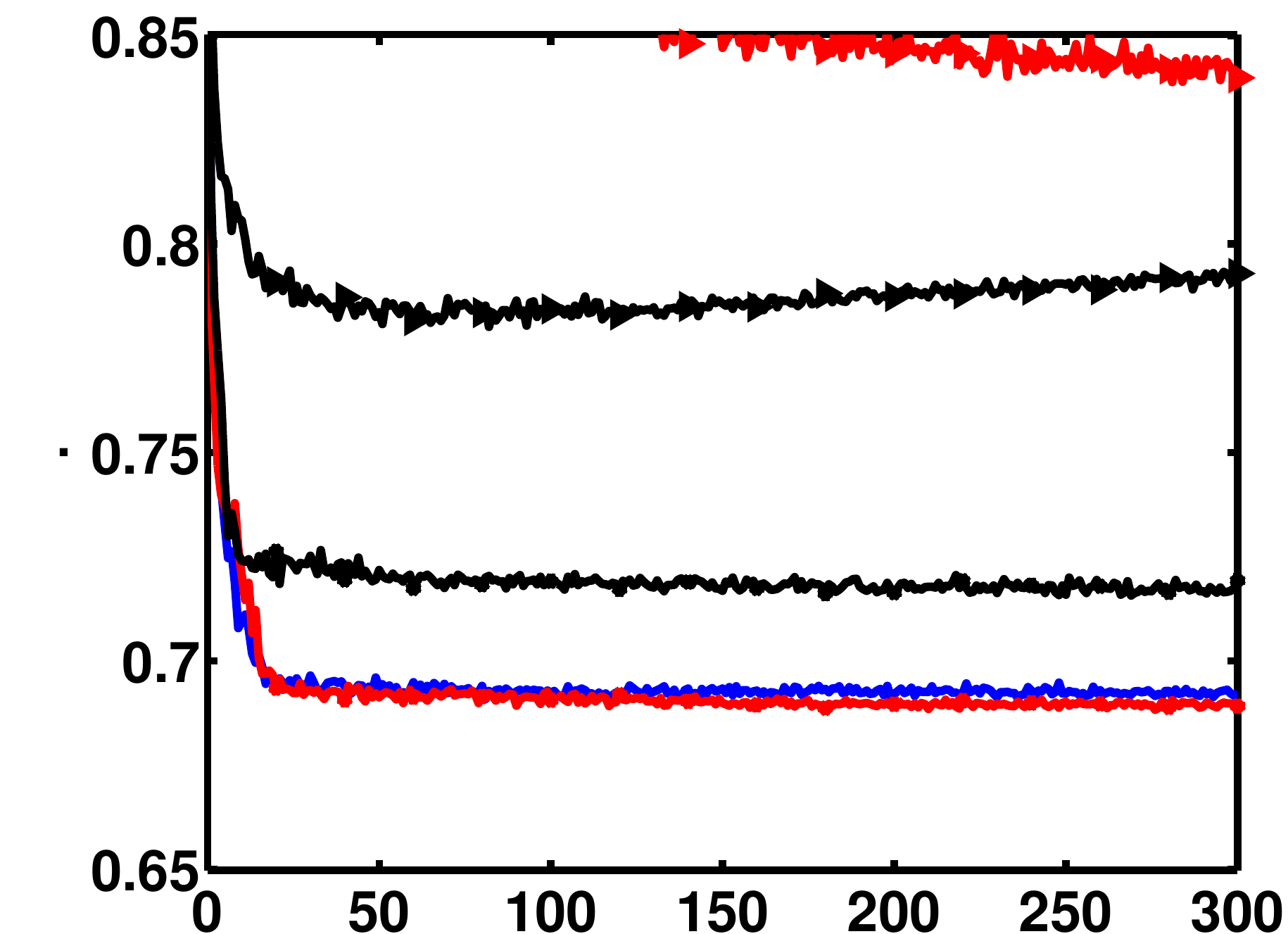} \\
   \includegraphics[width=1.9in]{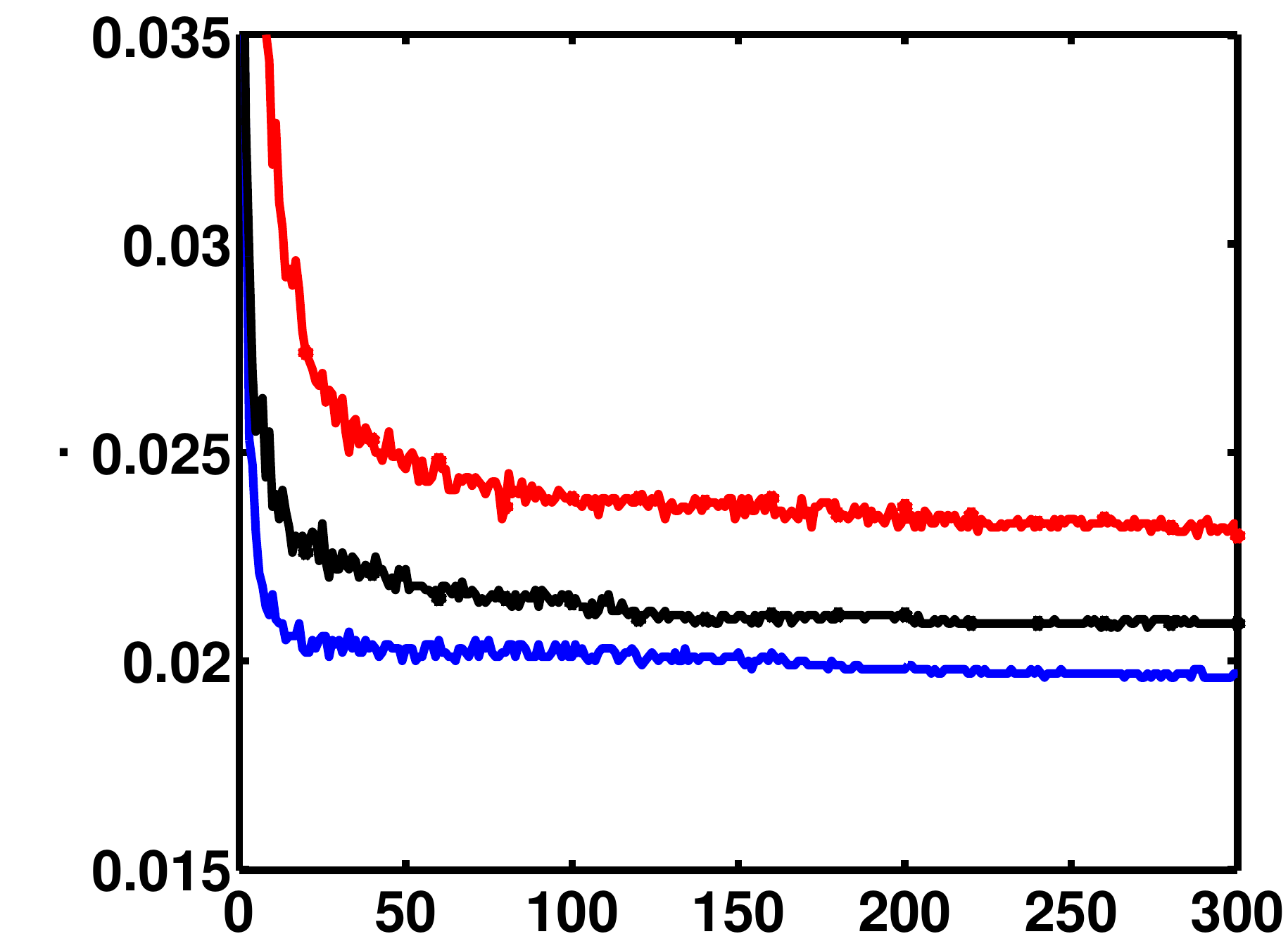} \\
   \includegraphics[width=1.85in]{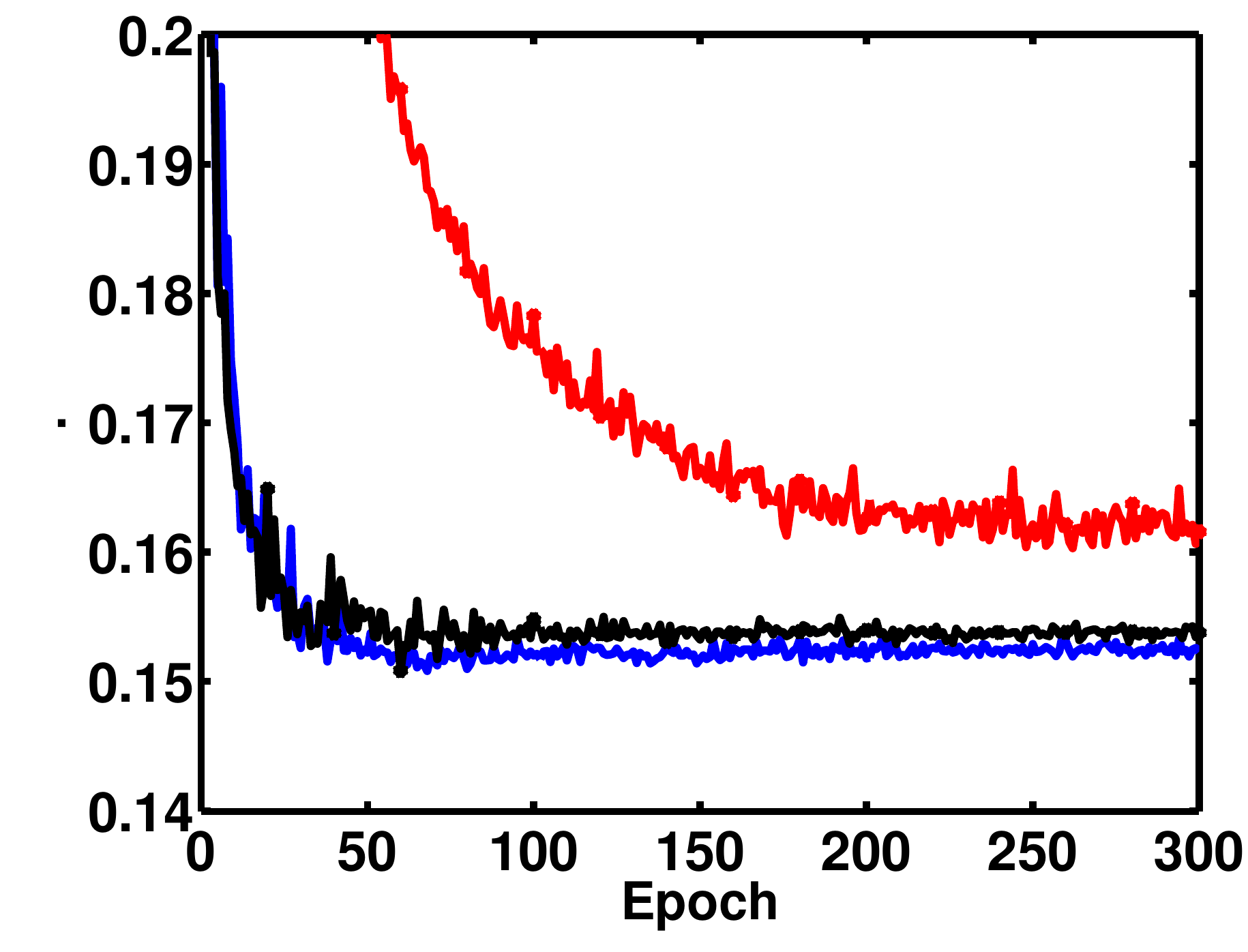}
  \end{tabular}
 }

 \begin{picture}(0,0)(0,0)
\rotatebox{90}{\put(342, 0){CIFAR-10}\put(240, 0){CIFAR-100}\put(147, 0){MNIST}\put(50, 0){SVHN}}
\end{picture}
  \begin{picture}(0,0)(0,0)
{\put(30, 420){\small Cross-Entropy Training Loss}\put(187, 420){\small 0/1 Training Error}\put(328, 420){ \small 0/1 Test Error}}
\end{picture}
 \caption{\small Learning curves for more number of epochs using different optimization methods 
 for 4 datasets without dropout. Left panel displays the cross-entropy objective function; 
middle and right panels show the corresponding values of the training and test errors, where the values are reported on
different epochs during the course of optimization. Best viewed in color.}
 \label{fig:more1}
  \vspace{1.1in}
\end{figure}

\begin{figure}
\hspace{0.1in}
 \subfloat{
  \begin{tabular}{r}
   \includegraphics[width=\picwidth]{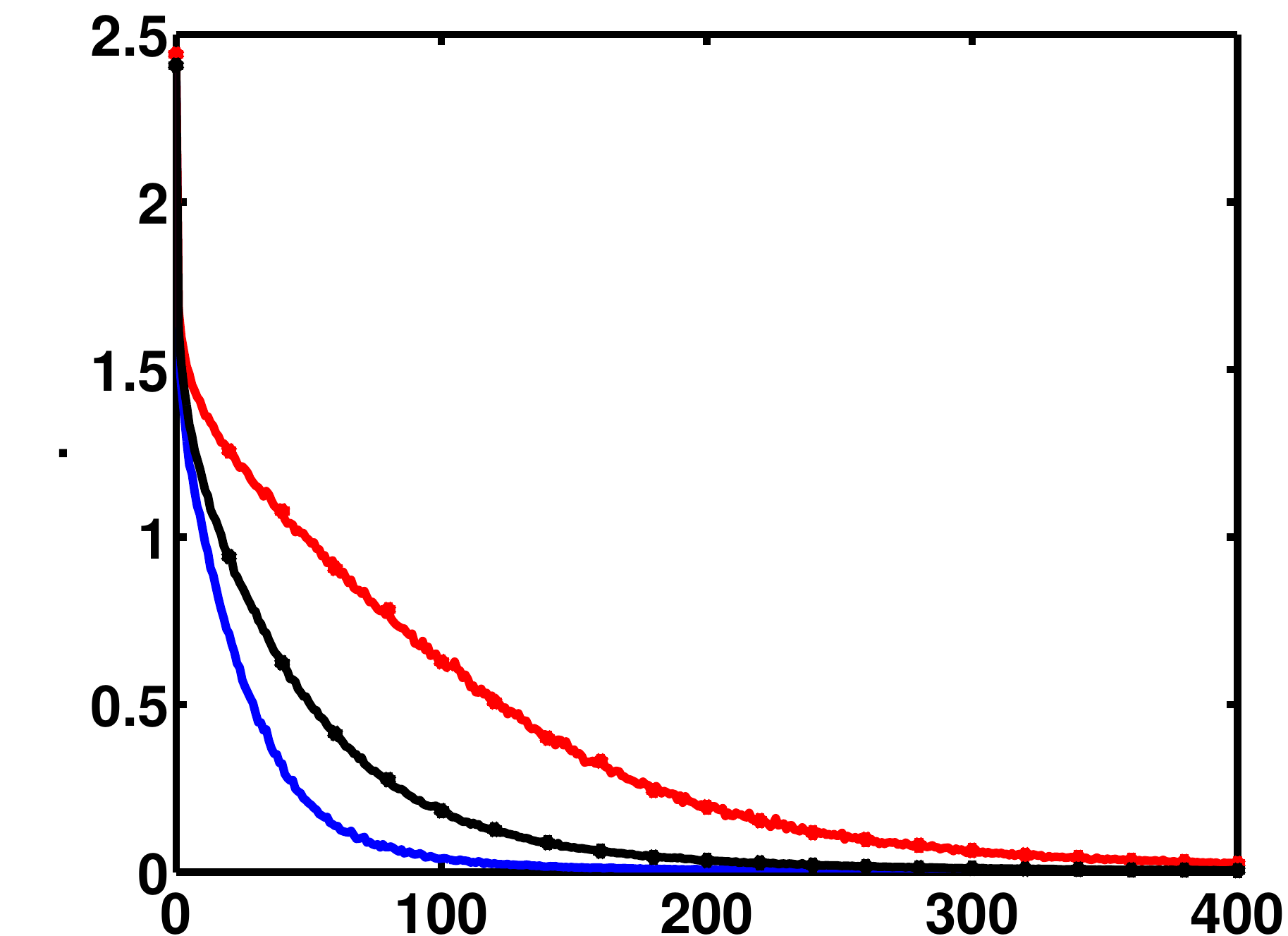} \\
   \includegraphics[width=1.79in]{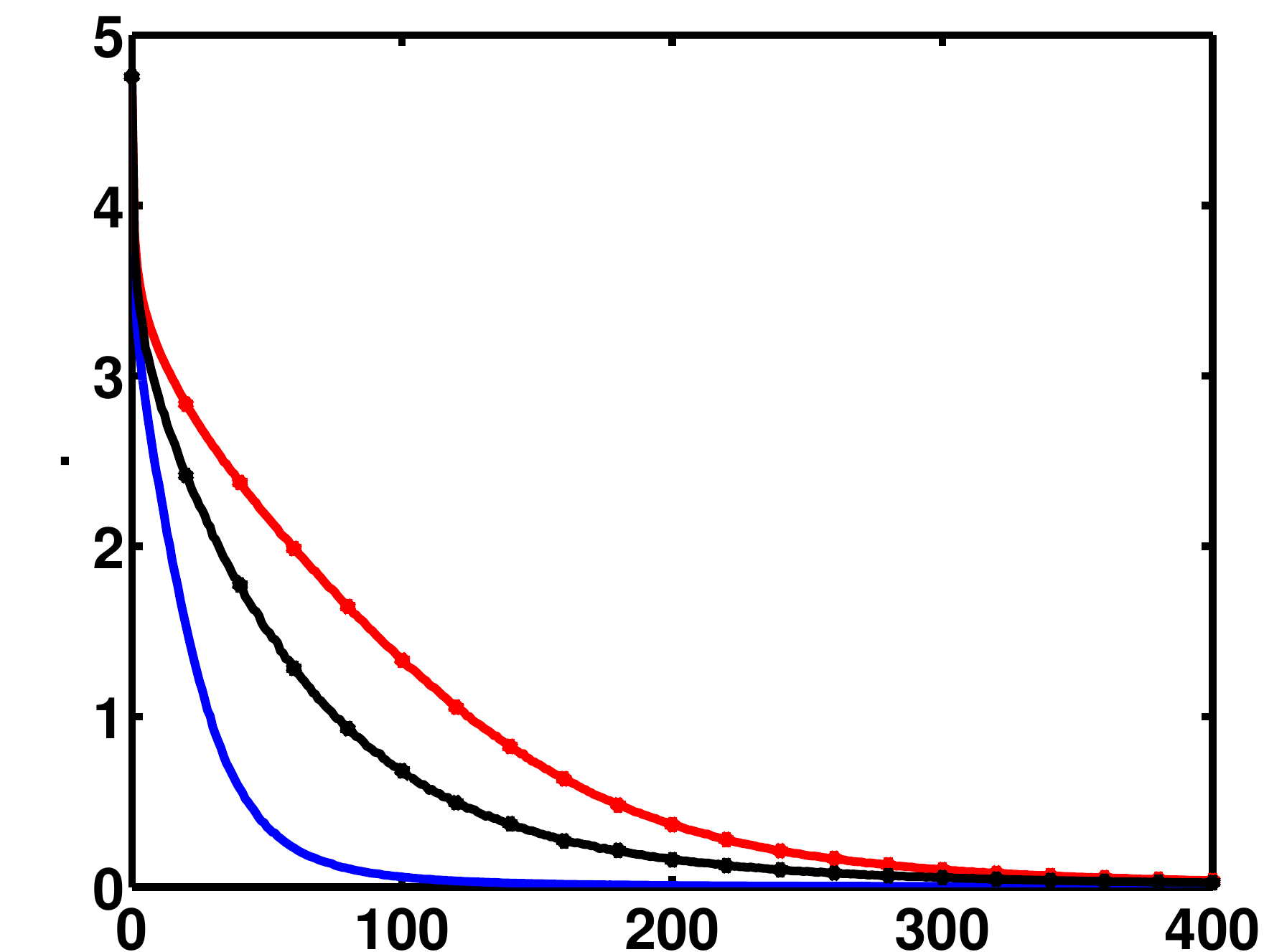} \\
   \includegraphics[width=1.85in]{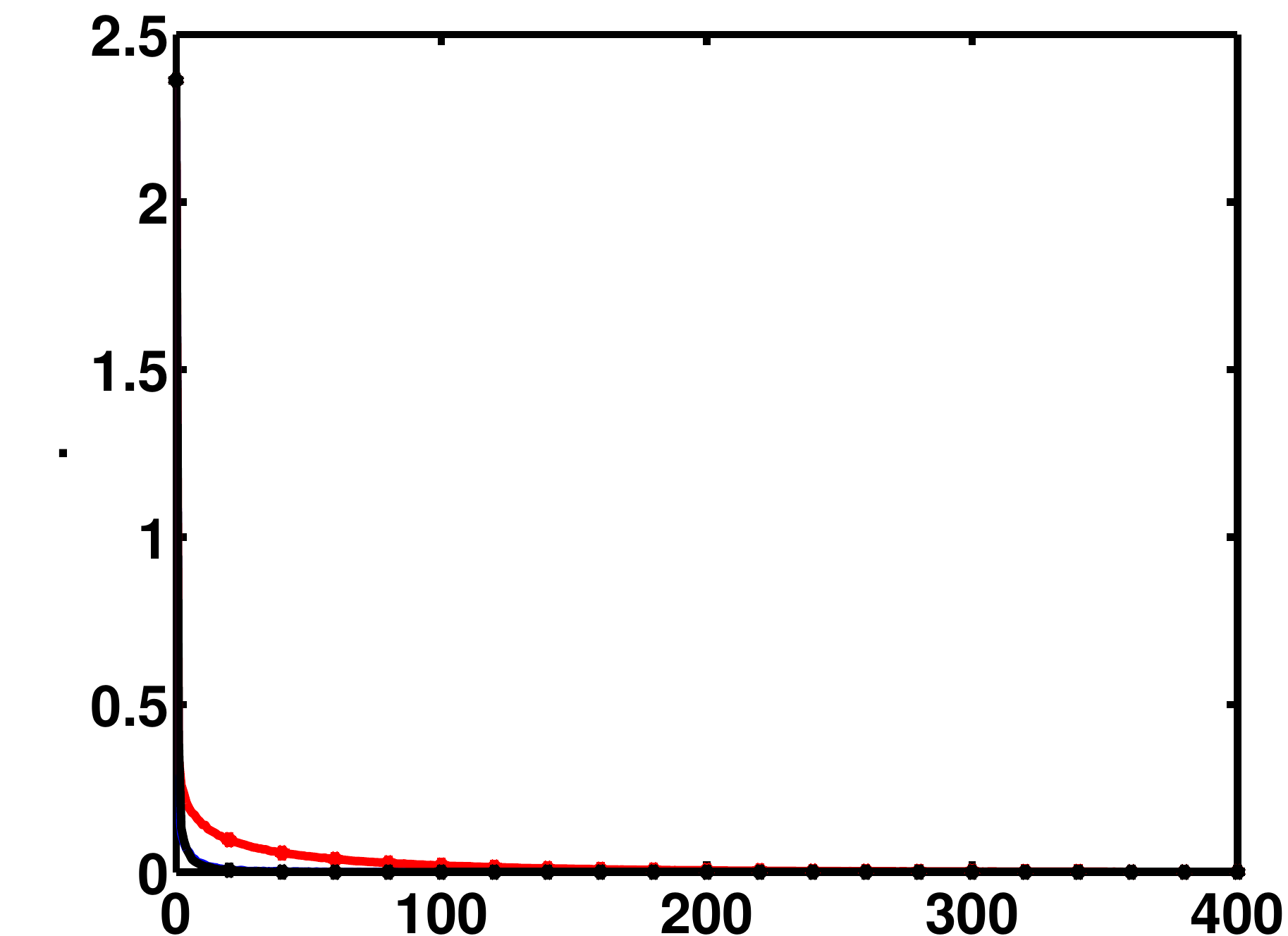} \\
   \includegraphics[width=1.85in]{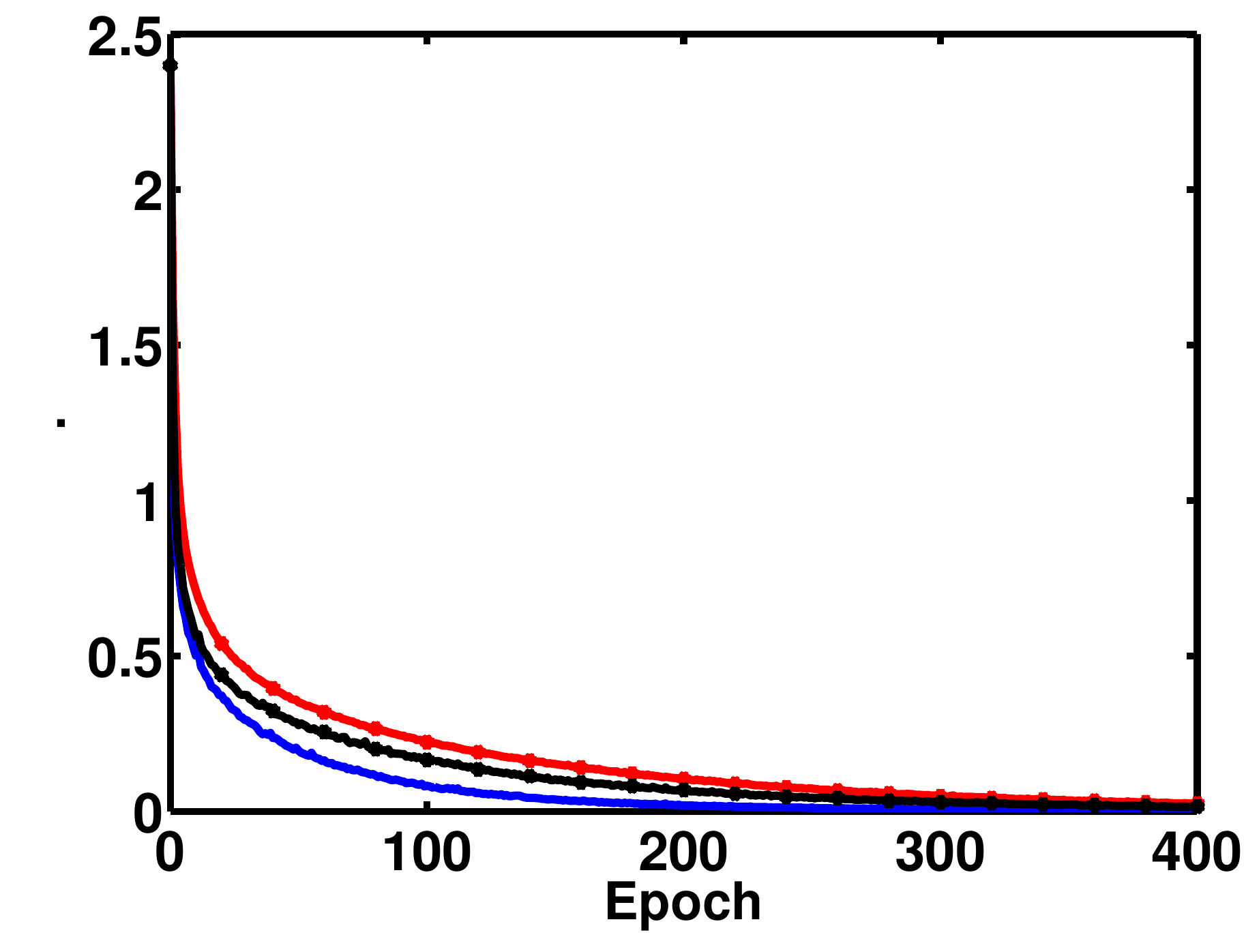}
  \end{tabular}
 }\hspace{-0.3in}
 \subfloat{
  \begin{tabular}{r}
   \includegraphics[width=\picwidth]{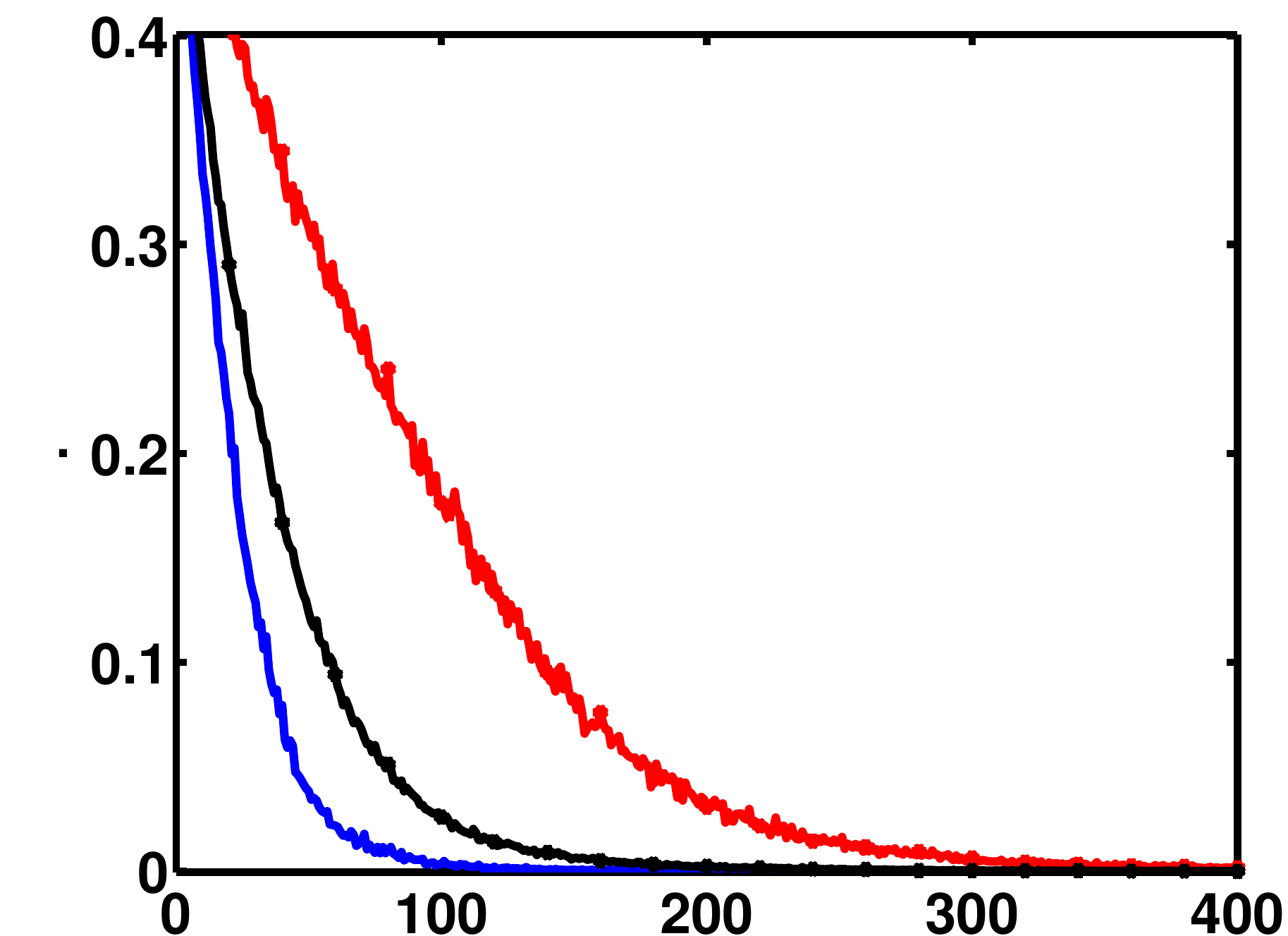} \\
   \includegraphics[width=\picwidth]{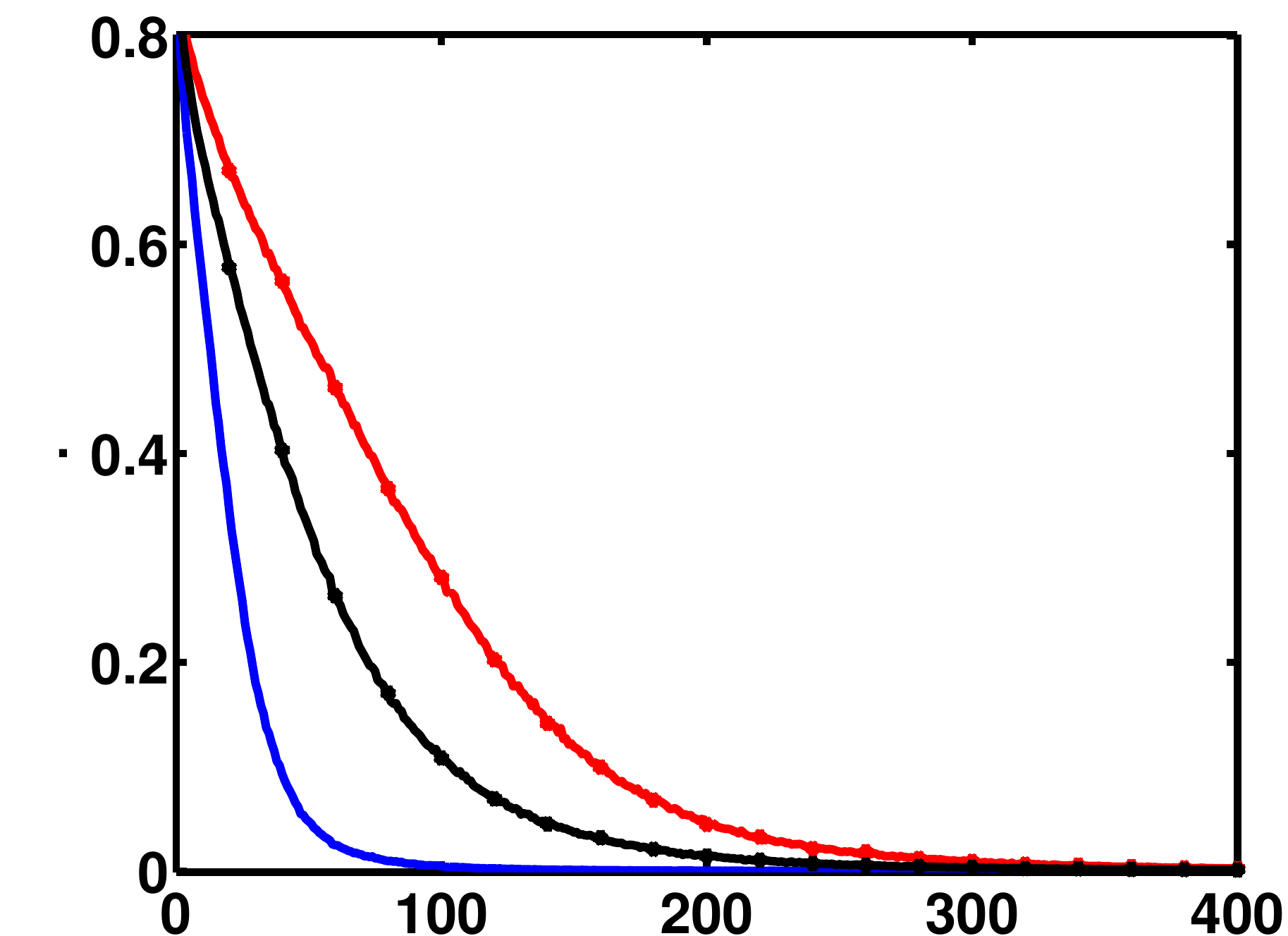} \\
   \includegraphics[width=1.9in]{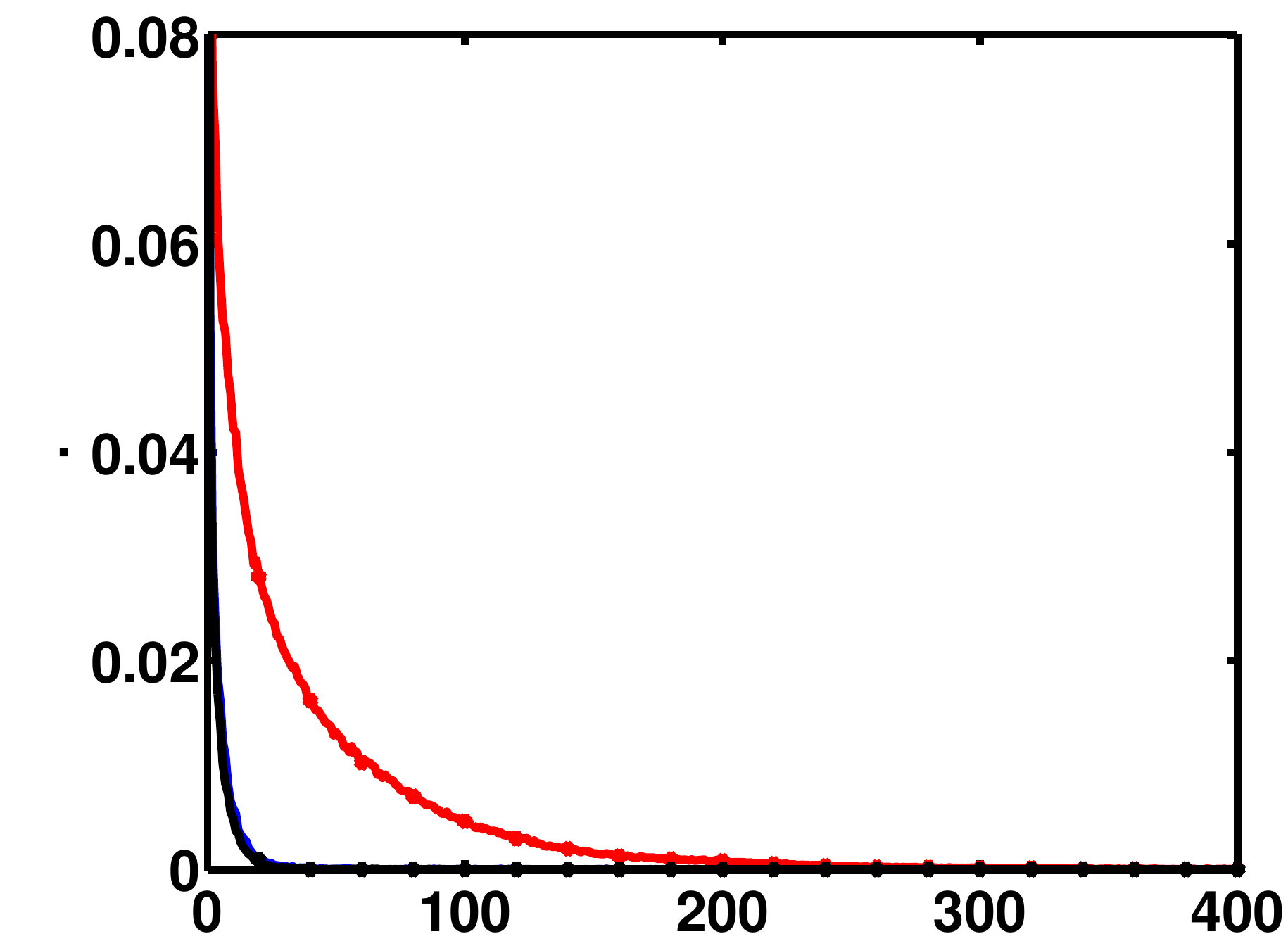} \\
   \includegraphics[width=\picwidth]{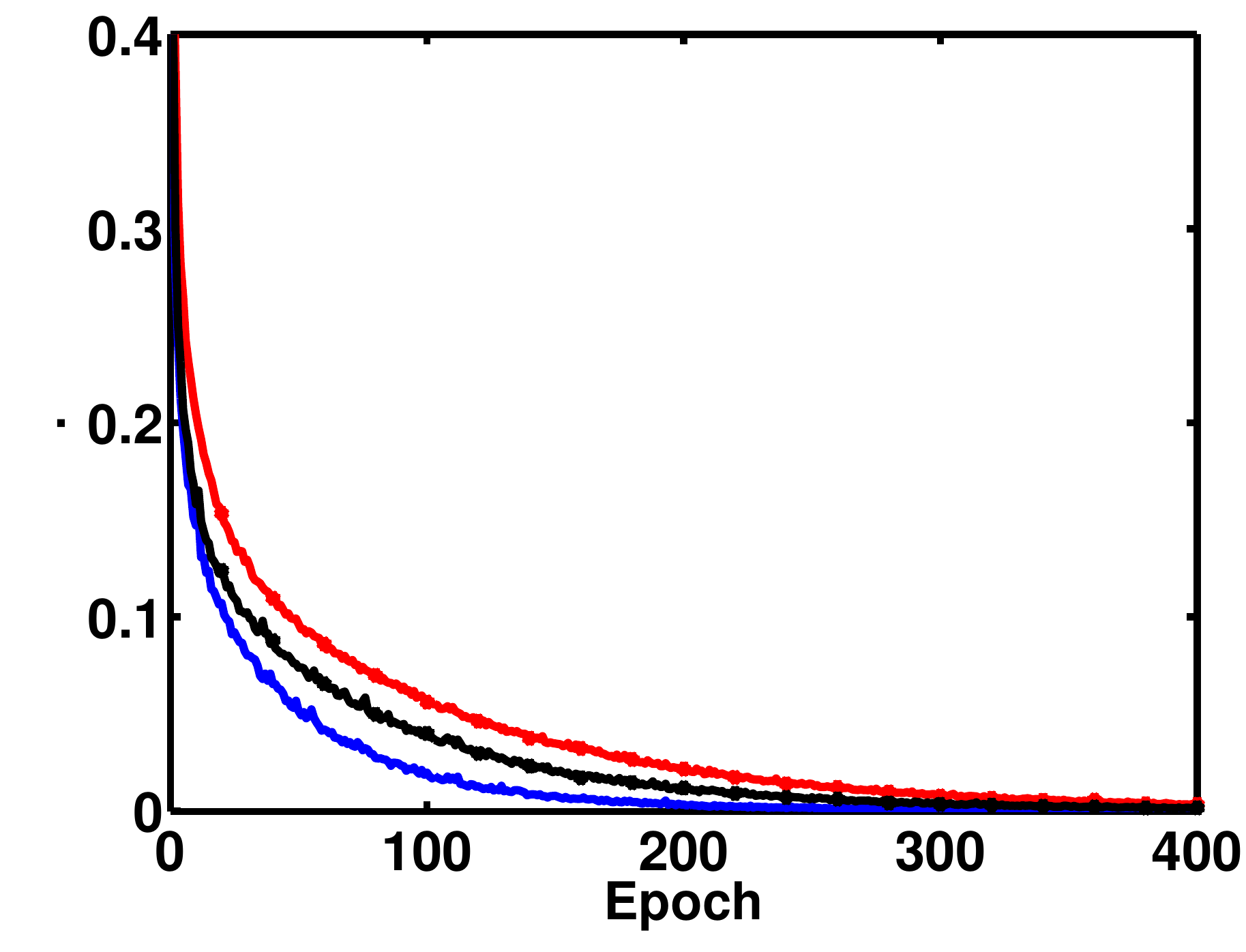}
  \end{tabular}
 }\hspace{-0.3in}
 \subfloat{
  \begin{tabular}{r}
   \includegraphics[width=\picwidth]{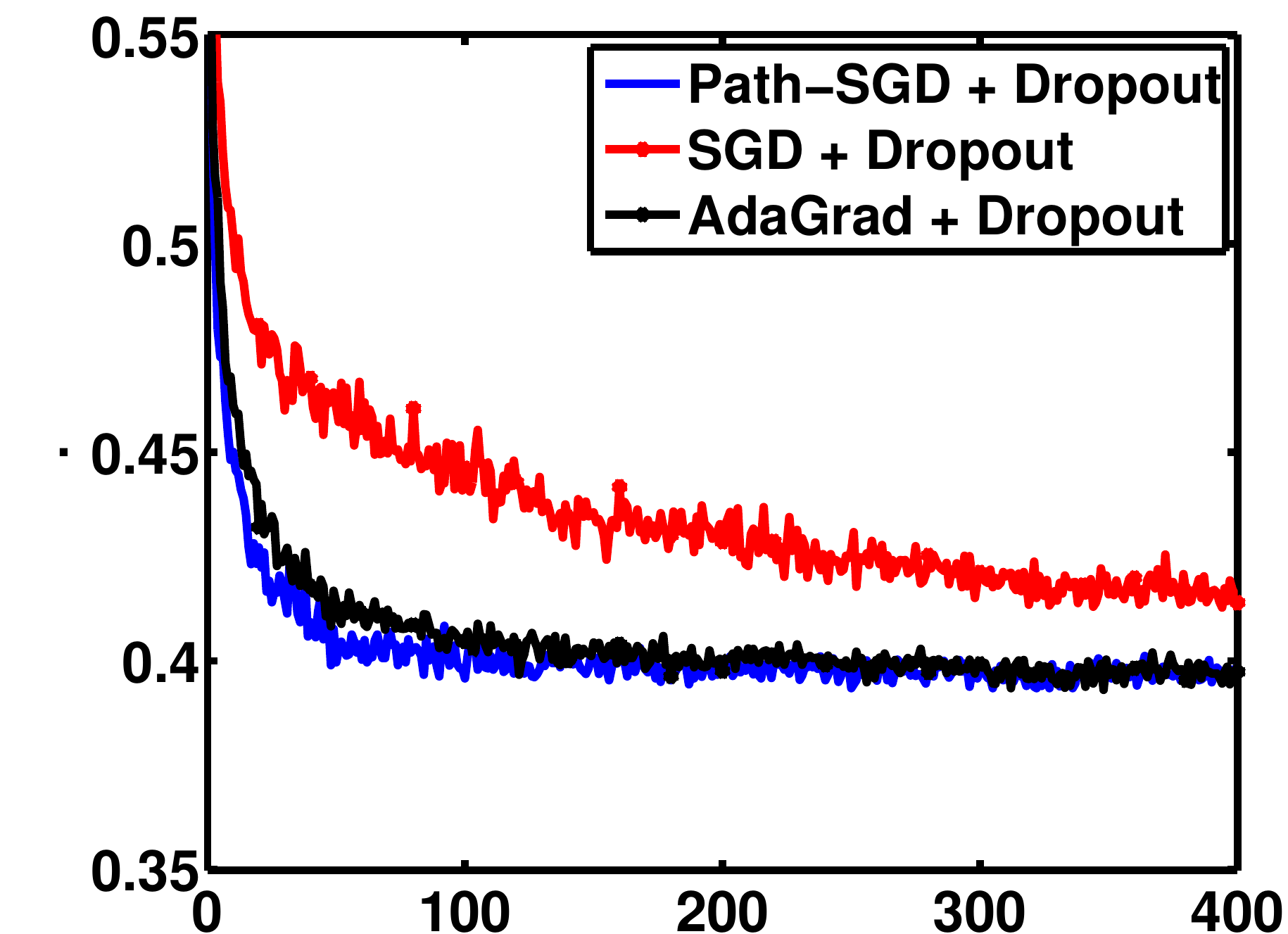} \\
   \includegraphics[width=\picwidth]{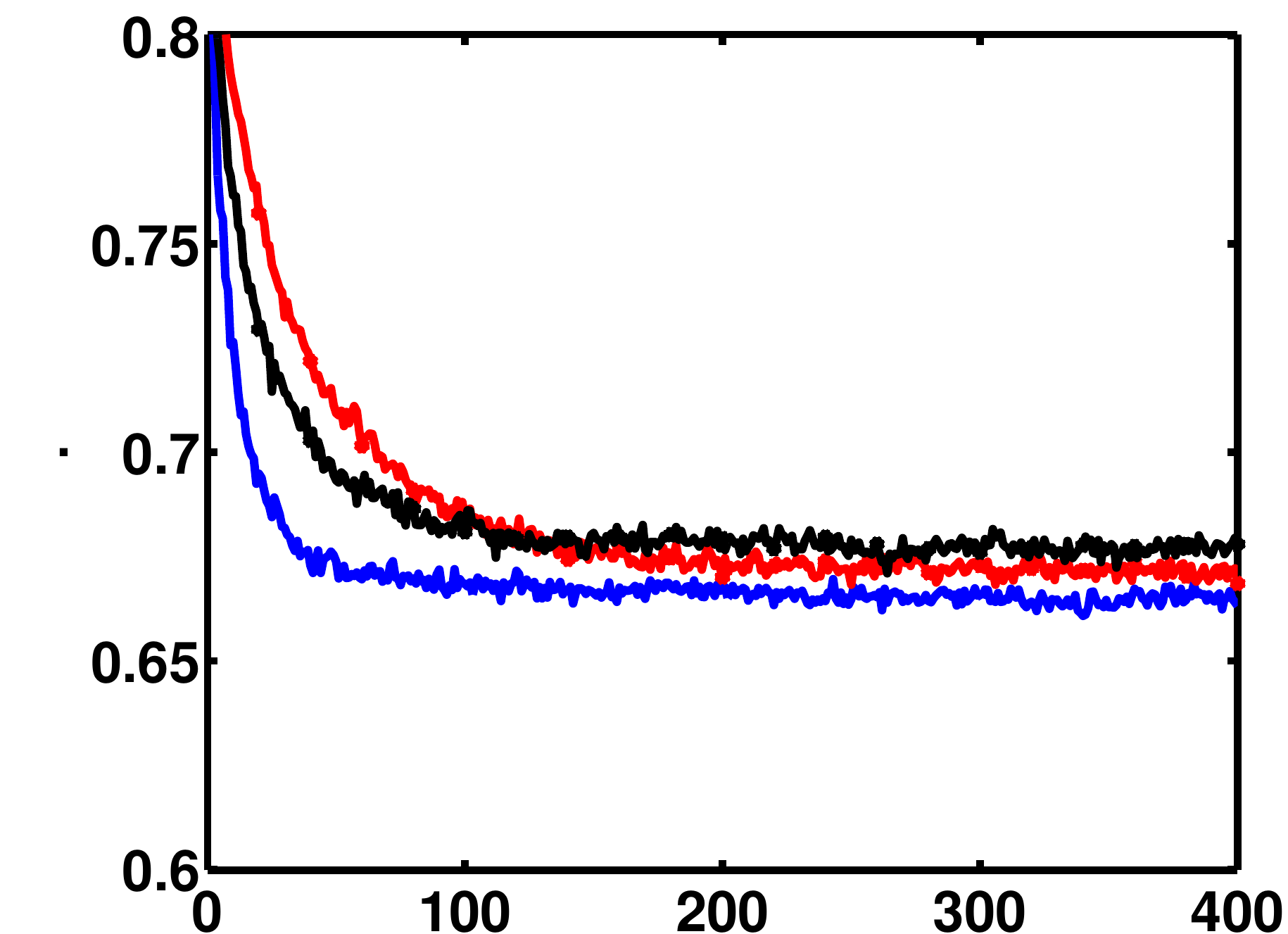} \\
   \includegraphics[width=1.9in]{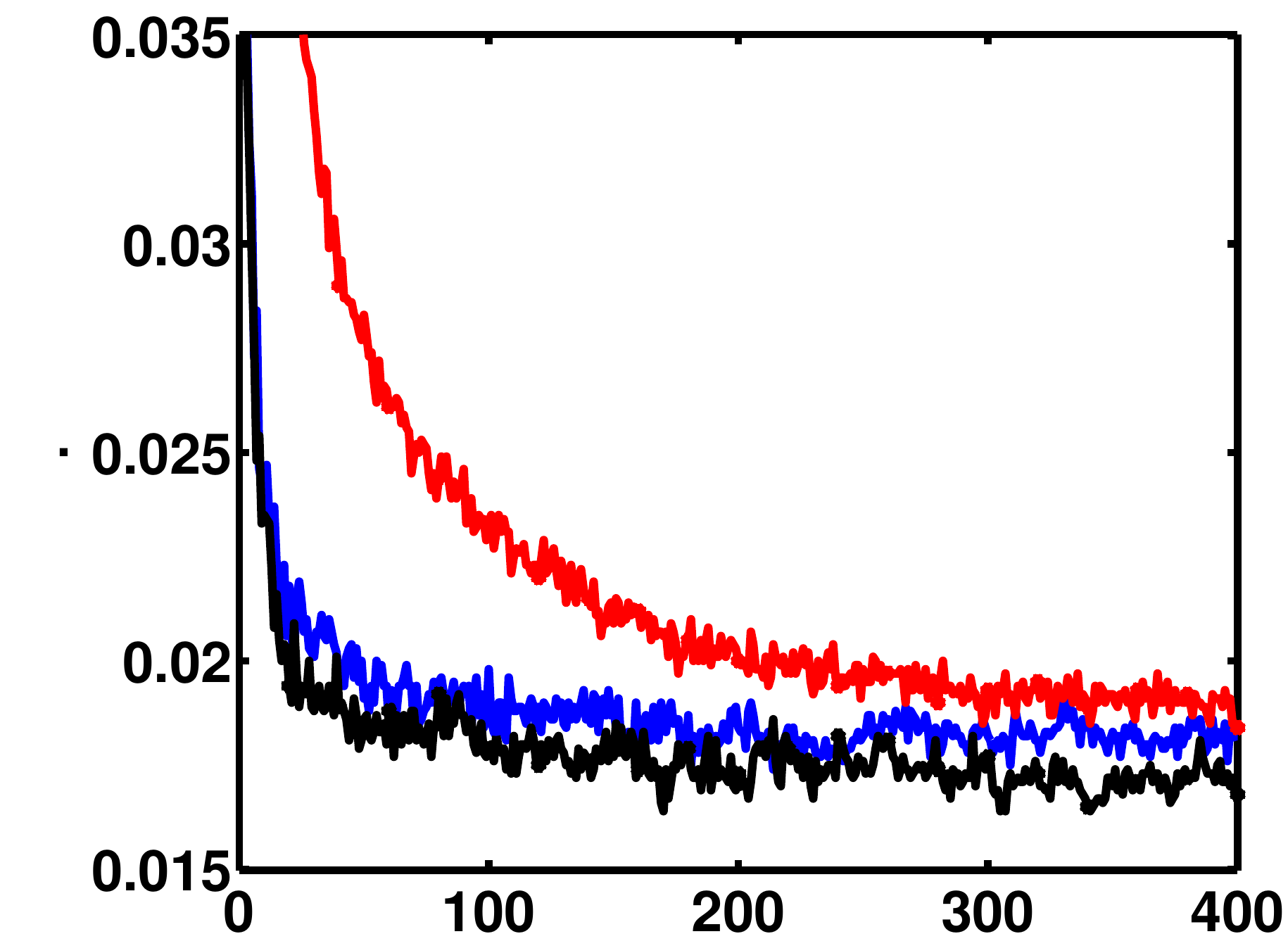} \\
   \includegraphics[width=\picwidth]{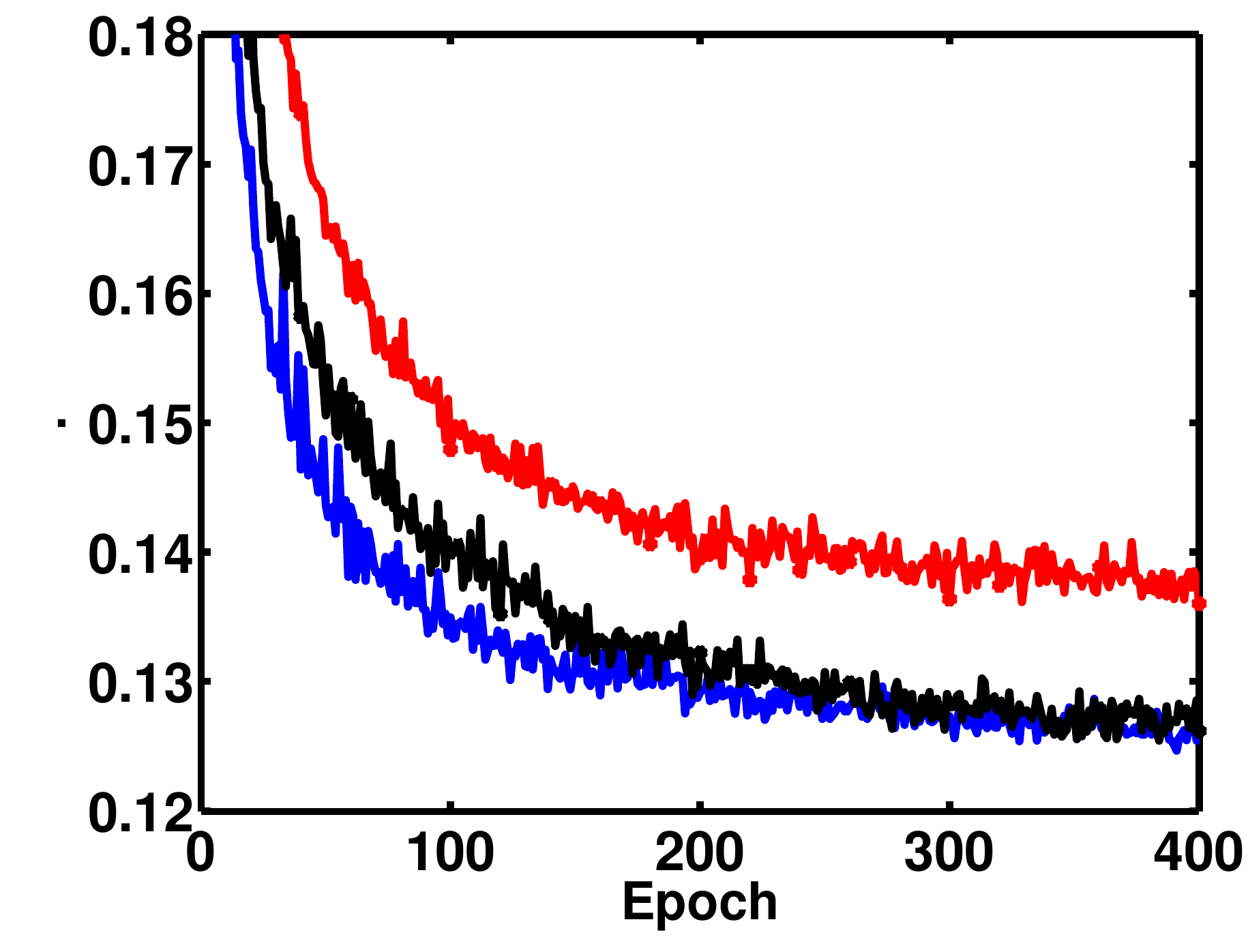}
  \end{tabular}
 }
 
 \begin{picture}(0,0)(0,0)
\rotatebox{90}{\put(342, 0){CIFAR-10}\put(240, 0){CIFAR-100}\put(147, 0){MNIST}\put(50, 0){SVHN}}
\end{picture}
  \begin{picture}(0,0)(0,0)
{\put(30, 420){\small Cross-Entropy Training Loss}\put(187, 420){\small 0/1 Training Error}\put(328, 420){ \small 0/1 Test Error}}
\end{picture}
\vspace{-0.1in}
 \caption{\small Learning curves for more number of epochs using different optimization methods
 for 4 datasets with dropout. Left panel displays the cross-entropy objective function;     
middle and right panels show the corresponding values of the training and test errors. Best viewed in color.}
 \label{fig:more2}
\vspace{-0.1in}
\end{figure}

\end{document}

%% file: header.tex
\usepackage{times,wrapfig,bm,color,enumitem}
\usepackage{amsmath,amsfonts,amssymb,amsthm}
\usepackage{algorithm,algpseudocode}

\newcommand{\inner}[2]{\left\langle #1, #2 \right\rangle}

\newcommand{\norm}[1]{\left\lVert{#1}\right\rVert}
\newcommand{\abs}[1]{\left\lvert{#1}\right\rvert}

\newcommand{\R}{\mathbb{R}}

\newcommand{\calS}{\mathcal{S}}

\newtheorem{thm}{Theorem}[section]
\newtheorem{lem}{Lemma}[section]

\mathchardef\hyphen="2D

%% file: main_arxiv.bbl
\begin{thebibliography}{10}

\bibitem{adagrad}
John Duchi, Elad Hazan, and Yoram Singer.
\newblock Adaptive subgradient methods for online learning and stochastic
  optimization.
\newblock {\em The Journal of Machine Learning Research}, 12:2121 -- 2159,
  2011.

\bibitem{difficulty}
Xavier Glorot and Yoshua Bengio.
\newblock Understanding the difficulty of training deep feedforward neural
  networks.
\newblock In {\em AISTATS}, 2010.

\bibitem{goodfellow13}
Ian~J. Goodfellow, David Warde{-}Farley, Mehdi Mirza, Aaron~C. Courville, and
  Yoshua Bengio.
\newblock Maxout networks.
\newblock In {\em Proceedings of the 30th International Conference on Machine
  Learning, {ICML}}, pages 1319--1327, 2013.

\bibitem{he2015delving}
Kaiming He, Xiangyu Zhang, Shaoqing Ren, and Jian Sun.
\newblock Delving deep into rectifiers: Surpassing human-level performance on
  imagenet classification.
\newblock {\em arXiv preprint arXiv:1502.01852}, 2015.

\bibitem{batch_norm}
Sergey Ioffe and Christian Szegedy.
\newblock Batch normalization: Accelerating deep network training by reducing
  internal covariate shift.
\newblock In {\em arXiv}, 2015.

\bibitem{Adam}
D.~P. Kingma and J.~Ba.
\newblock Adam: {A} method for stochastic optimization.
\newblock {\em CoRR}, abs/1412.6980, 2014.

\bibitem{krizhevsky2009learning}
Alex Krizhevsky and Geoffrey Hinton.
\newblock Learning multiple layers of features from tiny images.
\newblock {\em Computer Science Department, University of Toronto, Tech. Rep},
  1(4):7, 2009.

\bibitem{lecun1998gradient}
Yann LeCun, L{\'e}on Bottou, Yoshua Bengio, and Patrick Haffner.
\newblock Gradient-based learning applied to document recognition.
\newblock {\em Proceedings of the IEEE}, 86(11):2278--2324, 1998.

\bibitem{Kronecker}
James Martens and Roger Grosse.
\newblock Optimizing neural networks with kronecker-factored approximate
  curvature.
\newblock In {\em ICML}, 2015.

\bibitem{netzer2011reading}
Yuval Netzer, Tao Wang, Adam Coates, Alessandro Bissacco, Bo~Wu, and Andrew~Y
  Ng.
\newblock Reading digits in natural images with unsupervised feature learning.
\newblock In {\em NIPS workshop on deep learning and unsupervised feature
  learning}, 2011.

\bibitem{neyshabur15b}
Behnam Neyshabur, Ryota Tomioka, and Nathan Srebro.
\newblock In search of the real inductive bias: On the role of implicit
  regularization in deep learning.
\newblock {\em International Conference on Learning Representations (ICLR)
  workshop track}, 2015.

\bibitem{neyshabur15}
Behnam Neyshabur, Ryota Tomioka, and Nathan Srebro.
\newblock Norm-based capacity control in neural networks.
\newblock {\em COLT}, 2015.

\bibitem{srebro05}
Nathan Srebro and Adi Shraibman.
\newblock Rank, trace-norm and max-norm.
\newblock In {\em Learning Theory}, pages 545--560. Springer, 2005.

\bibitem{srebro11}
Nathan Srebro, Karthik Sridharan, and Ambuj Tewari.
\newblock On the universality of online mirror descent.
\newblock In {\em Advances in neural information processing systems}, pages
  2645--2653, 2011.

\bibitem{srivastava14}
Nitish Srivastava, Geoffrey Hinton, Alex Krizhevsky, Ilya Sutskever, and Ruslan
  Salakhutdinov.
\newblock Dropout: A simple way to prevent neural networks from overfitting.
\newblock {\em The Journal of Machine Learning Research}, 15(1):1929--1958,
  2014.

\bibitem{icml2013}
I.~Sutskever, J.~Martens, George Dahl, and Geoffery Hinton.
\newblock On the importance of momentum and initialization in deep learning.
\newblock In {\em ICML}, 2013.

\end{thebibliography}
